\documentclass[acmsmall]{acmart}
\AtBeginDocument{%
  }

\newcommand{\hide}[1]{}
\usepackage{amsmath}
\usepackage{amsfonts}
\usepackage{verbatim}
\usepackage{bbm}
\usepackage{stmaryrd}
\usepackage{xcolor}
\usepackage{graphicx}
\usepackage{wrapfig}
\usepackage{subcaption}
\usepackage{amsthm}

\usepackage{amssymb}

\newtheorem{definition}{Definition}

\usepackage[capitalise]{cleveref}

\usepackage{listings}
\definecolor{codebg}{HTML}{FFFFFF}       % dark-ish background
\definecolor{codeframe}{HTML}{263238}
\definecolor{keywordcol}{HTML}{82AAFF}
\definecolor{functioncol}{HTML}{C3E88D}
\definecolor{stringcol}{HTML}{C792EA}
\definecolor{commentcol}{HTML}{546E7A}
\definecolor{operatorcol}{HTML}{FFCB6B}

\lstdefinelanguage{HaskellCustom}{
  morekeywords={
    module,where,import,as,qualified,let,in,if,then,else,case,of,do,
    data,type,newtype,deriving,class,instance,default,foreign,return
  },
  sensitive=true,
  morecomment=[l]{--},
  morecomment=[s]{\{-}{-\}},
  morestring=[b]",
  morestring=[b]'
}

\Crefname{theorem}{Thm.}{Thm.}
\Crefname{corollary}{Cor.}{Corollary}
\Crefname{proposition}{Prop.}{Propositions}
\Crefname{claim}{Claim}{Claims}
\Crefname{definition}{Def.}{Definitions}
\Crefname{fact}{Fact}{Facts}
\Crefname{conjecture}{Conj.}{Conjectures}
\Crefname{example}{Ex.}{Ex.}
\Crefname{remark}{Rem.}{Remarks}
\Crefname{convention}{Convention}{Conventions}
\Crefname{lemma}{Lem.}{Lemmas}
\Crefname{assumption}{Ass.~}{Ass.~}
\Crefname{section}{Sec.}{Sec.}
\Crefname{appendix}{App.}{App.}
\Crefname{figure}{Fig.}{Fig.}
\Crefname{algorithm}{Alg.}{Alg.}
\Crefname{lstlisting}{Listing}{Listings}
\Crefname{namedproof}{proof}{proofs}
\crefname{condsi}{Condition}{Conditions}
\Crefname{condsi}{Condition}{Conditions}

\newtheorem{theorem}{Theorem}
\newtheorem{proposition}{Proposition}
\newtheorem{lemma}{Lemma}

\usepackage{enumitem}

\newlist{conds}{enumerate}{1}
\setlist[conds]{label=(A\arabic*), ref=A\arabic*}

\DeclareMathOperator{\Supp}{\textsf{Supp}}
\newcommand*\diff{\mathrm{d}}

\setcopyright{cc}
\setcctype{by}
\acmDOI{10.1145/3828696}
\acmYear{2026}
\acmJournal{PACMPL}
\acmVolume{10}
\acmNumber{ICFP}
\acmArticle{298}
\acmMonth{8}
\acmSubmissionID{icfp26main-p82-p}
\received{2026-02-19}
\received[accepted]{2026-05-13}

\begin{document}

%%
%% The "title" command has an optional parameter,
%% allowing the author to define a "short title" to be used in page headers.
\title{LazyHMC: Hamiltonian Monte Carlo Simulation for Lazy, Infinite Dimensional Probabilistic Programs}

%%
%% The "author" command and its associated commands are used to define
%% the authors and their affiliations.
%% Of note is the shared affiliation of the first two authors, and the
%% "authornote" and "authornotemark" commands
%% used to denote shared contribution to the research.
\author{Maria-Nicoleta Cr\u{a}ciun}
\orcid{0009-0006-0108-1692}
\affiliation{%
  \institution{University of Oxford}
  %\department{Department of Computer Science}
  \city{Oxford}
  \country{UK}
}
\email{maria.craciun@cs.ox.ac.uk}

\author{C.-H. Luke Ong}
\orcid{0000-0001-7509-680X}
\affiliation{%
  \institution{Nanyang Technological University}
  %\department{College of Computing and Data Science}
  \city{Singapore}
  \country{Singapore}
}
\email{luke.ong@ntu.edu.sg}

\author{Tom Schrijvers}
\orcid{0000-0001-8771-5559}
\affiliation{%
  \institution{KU Leuven}
  %\department{Department of Computer Science}
  \city{Leuven}
  \country{Belgium}
}
\email{tom.schrijvers@kuleuven.be}

\author{Sam Staton}
\orcid{0000-0002-7149-3805}
\affiliation{%
  \institution{University of Oxford}
  %\department{Department of Computer Science}
  \city{Oxford}
  \country{UK}
}
\email{sam.staton@cs.ox.ac.uk}

%%
%% By default, the full list of authors will be used in the page
%% headers. Often, this list is too long, and will overlap
%% other information printed in the page headers. This command allows
%% the author to define a more concise list
%% of authors' names for this purpose.
%\renewcommand{\shortauthors}{Trovato et al.}

\begin{abstract}
Hamiltonian Monte Carlo (HMC) is a successful generic inference method in probabilistic programming, but in its ordinary formulation it needs gradients and finite-dimensional parameter spaces.
In Haskell, lazy evaluation lets probabilistic programs express stochastic processes and other non-parametric Bayesian models over implicit infinite-dimensional spaces.
This paper develops new formulations of gradient-based HMC for this infinite-dimensional setting, via lazy evaluation.
For automatic differentiation, we provide an analysis based on a new notion of ``piecewise analytic under cylindrical analytic partition'' (PACAP), to show that even if a program is infinite-dimensional and defined lazily, the gradient of the likelihood function is finitely supported.
For the Monte Carlo method itself, we develop several HMC variants and a No-U-Turn Sampler that operate over the infinite-dimensional parameter space but are still productive because of lazy evaluation.
Experiments cover Gaussian mixture clustering, random walks, and piecewise-constant regression with Poisson-process changepoints.
\end{abstract}


% %%
% %% Keywords. The author(s) should pick words that accurately describe
% %% the work being presented. Separate the keywords with commas.

%%
%% This command processes the author and affiliation and title
%% information and builds the first part of the formatted document.
\maketitle

%\tableofcontents

\section{Introduction}
Probabilistic programming (e.g.~\cite{DBLP:journals/corr/abs-1809-10756,Barthe_Katoen_Silva_2020}) is a method for Bayesian statistical modelling by writing programs.
Recall that Bayes' law specifies how to calculate the posterior probability, in terms of a prior probability and the likelihood of observations.
Probabilistic programming is often treated as a declarative programming method: the statistician declares a prior and likelihood by writing a high level program, and a generic inference method (such as a Monte Carlo method) is then used to provide samples from the posterior.

Infinite or unbounded dimensions in a parameter space is often called `non-parametric' statistics (for a broad overview, see e.g.~\cite{orbanz2010bayesian}). For example, in a clustering problem we would like to group data points into clusters; in the non-parametric setting we want to explore the number of clusters, and we do not want to fix in advance how many clusters there are. Similarly in many models based on a stochastic process, such as a random walk, the process on the face of it goes on forever, or with infinite resolution. Non-parametric statistics is often well suited to programming, because the programming notations can more clearly express the complex dynamics than informal statistical notations.

\paragraph{The problem: generic inference over non-parametric models.}
Although probabilistic programming provides an elegant declarative framework for specifying statistical models, generic inference over non-parametric models is notoriously difficult. Hamiltonian Monte Carlo (HMC) simulation is arguably the most successful generic inference method, and is widely considered responsible for the uptake of Bayesian statistics in practice. However, HMC fundamentally relies on gradients and finite-dimensional parameter spaces, and so does not directly apply to non-parametric models. We can identify three existing approaches to this challenge:
\begin{itemize}
\item \emph{Truncation by hand:} Approximate the infinite dimensional structure by a finite dimensional one. For example, an infinite-dimensional Gaussian process is replaced by a high dimensional multivariate Gaussian. We can then employ well understood generic inference methods for finite dimensions, including HMC. This approach is most widely used in practice. However, the approach is not compositional: the truncation bounds must be carefully adjusted across the whole model.
\item \emph{Dynamic dimensions:} As a program runs, keep track of which dimensions are actually needed at any point, adding new dimensions as necessary. This avoids approximation. But the model is less purely declarative, as it explicitly tracks the dimension count and changes to dimensions. Nonetheless, this is a common approach in the probabilistic programming community~\cite{NPHMCMakZO21,NPiMCMCMakZO22,RoyMansinghkaGoodmanTenenbaum2008,DBLP:journals/jmlr/WingateSG11,DBLP:conf/aistats/WoodMM14,ZhouDCC}.
\item \emph{Lazy use of infinite dimensions:} The statistical model and inference take place in the context of an infinite dimensional parameter space. Of course, no generic inference method can explicitly track infinitely many dimensions. However, lazily, at the last minute, we notice that only a finite subset of dimensions are actually needed for the calculation, and we use that structure to provide samples from the posterior.

  This lazy approach has been demonstrated~\cite{DashKPS23LazyPPL} for a simple Metropolis-Hastings simulation, and~\cite{DBLP:journals/pacmpl/BowersLTSM25} for discrete programs. However, these methods do not use gradient information, and so they do not benefit from the efficiency that makes HMC so effective in practice.
\end{itemize}
In summary, the existing approaches either sacrifice compositionality (truncation), declarativeness (dynamic dimensions), or efficiency (gradient-free lazy methods). What is missing is a way to bring the power of gradient-based HMC to the lazy, infinite-dimensional setting.

\paragraph{Our solution: Lazy Hamiltonian Monte Carlo.}
In this paper we demonstrate that HMC can be applied directly to the infinite dimensional parameter spaces where non-parametric statistical models naturally live, by evaluating lazily. The main technical contribution is to structure the Hamiltonian so that the acceptance ratio, ostensibly an infinite product over all dimensions, collapses to a finite product: the gradient has finite support, and the unvisited dimensions cancel.
Our lazy HMC method operates over programs in the LazyPPL library (\S\ref{sec:lazyppl-overview}), which has an implicit infinite dimensional parameter space and allows explicit types of infinite dimensional stochastic processes.

\paragraph{Lazy Hamiltonian Monte Carlo in a Nutshell}
At a high level, HMC simulation provides a Markov chain whose stationary distribution is the posterior distribution of the statistical model. It can be regarded as the combination of the following three ideas:
\begin{enumerate}
\item Gradient descent: the log-likelihood function $\mathbb{R}^n\to \mathbb{R}$ or unnormalized density is a function of the parameters; here the dimension of the parameter space is $n$. We can optimize this by gradient descent.
\item Momentum: gradient descent alone provides a sequence of samples but this does not converge to the posterior and may not explore all modes. To deal with this, we regard a particle following a gradient trajectory and endow it with momentum, moreover, this momentum changes randomly over time. At each step, this Hamiltonian dynamics is typically simulated by a leapfrog integrator.
\item Metropolis correction: The discretization error from the leapfrog integrator is corrected for by randomly either accepting or rejecting each proposed sample, according to the Metropolis-Hastings acceptance ratio.
\end{enumerate}
Extending each of these ideas to the infinite-dimensional lazy setting requires new contributions, which we summarize as follows:
\begin{enumerate}
\item \textbf{Infinite-dimensional automatic differentiation} (\S\ref{sec:ADtheorem}, \S\ref{sec:Haskell-Nagata}). In the lazy situation, the log-likelihood function now has an infinite dimensional domain, e.g.~$\mathbb{R}^{\mathbb N}\to \mathbb{R}$. Although we can lazily maintain an infinite dimensional stream of parameters, we cannot hope to make a gradient step in an infinite dimensional direction. We introduce a smoothness condition, PACAP, guaranteeing that the gradient is non-zero in only a finite subset of the dimensions, and we show that it holds for a core calculus of programs. We find these gradients using automatic differentiation, adapting a recent idea based on sparse-map Nagata numbers to this infinite-dimensional setting.
\item \textbf{Lazy No-U-Turn Sampler} (\S\ref{sec:LazyNUTS}). HMC performance is sensitive to the number of leapfrog steps, which is typically mitigated by switching to the No-U-Turn Sampler (NUTS). We develop lazy versions of this too.
\item \textbf{Lazy Metropolis-Hastings correction} (\S\ref{sec:LazyHMC}). The Metropolis-Hastings ratio naturally involves an infinite number of terms, one for each dimension. We carefully choose the momentum and structure the Hamiltonian so that all but a finite number of these cancel. Although this requires some care, there are various options, and we propose three different ways of doing this. In \S\ref{sec:stepReg} we give a detailed worked example, piecewise-constant regression, illustrating how the leapfrog integrator and dimension changes interact.
\item \textbf{Experimental evaluation} (\S\ref{sec:experiments}). We implement lazy HMC and lazy NUTS as an extension to the LazyPPL library and evaluate them on non-parametric models from the literature, including geometric distributions, random walks, and clustering. There is no standard benchmark suite for non-parametric statistics, but our experiments provide empirical evidence that the methods work as intended.
\end{enumerate}
\section{Overview of Lazy Probabilistic Programming}
\label{sec:lazyppl-overview}
In this section we give an overview of lazy probabilistic programming, illustrating through examples how laziness enables non-parametric statistical models. We then identify where the existing gradient-free inference methods for this setting fall short, and preview how lazy HMC addresses these limitations.

We use the syntax of the LazyPPL library~\cite{DashKPS23LazyPPL} for Haskell. We do not assume very
much familiarity with Haskell, but the key idea is that within a \lstinline|do| block for the probability monad \lstinline|Prob|, we have, in effect, a domain specific language for probabilistic programming. Haskell's lists are not necessarily finite, and so random lists amount to random processes. Thus the key starting point of non-parametric statistics, that the dimension of the parameter space is unknown or unbounded, is mapped to the programming concept of laziness.

We consider three kinds of example: sequences and geometric distributions (\S\ref{sec:geometric}); random walks (\S\ref{sec:randomwalk}); and non-parametric clustering (\S\ref{sec:clustering}).

\subsection{Examples: IID Sequences and Geometric Distributions}\label{sec:geometric}
We assume a built in primitive \lstinline|uniform :: Prob RealNum|
that produces a random real number uniformly between 0 and~$1$.
We can then write a recursive program that produces an infinite stream of uniform random samples.
These will be independent and identically distributed (IID). 
\begin{lstlisting}
iiduniform :: Prob [RealNum]
iiduniform = do { x <- uniform ; xs <- iiduniform ; return (x : xs) }
\end{lstlisting}
(For now, the reader can assume \lstinline|RealNum=Double|, but in Section~\ref{sec:Haskell-Nagata} we will generalize this for automatic differentiation.)
In Haskell, lists can be infinite, and this recursive program describes a random infinite list.
(Haskell uses lazy evaluation: values are computed only when needed, so only the elements actually accessed will be generated.)
This amounts to what would be written in statistics as
``$
  x_n\sim U(0,1)\ \mathrm{iid}\ (\text{for all } n\in\mathbb{N})
$''.
More generally, we can define a function
\begin{lstlisting}
iid p :: Prob a -> Prob [a]
iid p = do { x <- p ; xs <- iid p ; return (x : xs) }
\end{lstlisting}
that transforms any probability distribution into an IID sequence, so that \lstinline|iid uniform = iiduniform|. 

This `IID sequence' is idiomatic in statistics. For a first example, for any $p\in(0,1)$, we can consider the Bernoulli distribution
\begin{lstlisting}
bernoulli :: RealNum -> Prob Bool
bernoulli p = do { x <- uniform ; return (x < p) }
\end{lstlisting}
which gives true or false with probability $p$. A sequence of IID Bernoulli trials can then be given by \lstinline|iid (bernoulli p)|.

The \emph{geometric distribution} with mean $\frac 1p$ is typically defined as the time of the first success in the random sequence. As a program: 
\begin{lstlisting}[caption={Geometric distribution}, label={lst:geometric}]
geometric :: RealNum -> Prob Int
geometric p = do { xs <- iid (bernoulli p) ; let (Just n) = findIndex id xs ;
                   return (1+n) }
\end{lstlisting}

As an aside we recall that, as is well known, the geometric distribution can be written without an explicit infinite sequence, as follows:
\begin{lstlisting}
geometricStrict p = do x <- bernoulli p
                       if x then return 1
                       else do { n <- geometricStrict p ; return (1 + n) }
\end{lstlisting}
The transformation from \lstinline|geometric| to \lstinline|geometricStrict| in effect manually propagates the evaluation of the lazy evaluation of the infinite stream \lstinline|xs|. The number of random choices remains unknown and unbounded.  

\subsection{Examples: Random Walks}\label{sec:randomwalk}
We next consider random walks, which are random sequences of positions.
We will work with random walks where the steps are uniformly distributed, for which we first consider a more general uniform distribution: 
\begin{lstlisting}
uniformRange :: RealNum -> RealNum -> Prob RealNum
uniformRange a b = do { x <- uniform ; return (a + (b-a)*x) }
\end{lstlisting}
In statistics, this would be written as $U(a,b)$. We can then define a random walk by first sampling all the steps IID, and then combining them, together with a starting position in $[0,3]$. 
\begin{lstlisting}
walk :: Prob [RealNum]
walk = do { start <- uniformRange 0 3 ; steps <- iid (uniformRange (-1) 1) ;
            return (scanl (+) start steps) }
\end{lstlisting}
Here, recall that \lstinline|scanl (+)| produces the cumulative sums, i.e., the list of intermediate results. 
An equivalent way of describing this distribution on sequences adds the steps at each point, as is well known:
\begin{lstlisting}
walkFrom :: RealNum -> Prob [RealNum]
walkFrom start = do { next <- uniformRange (start-1) (start+1) ;
                      rest <- walkFrom next ; return (next : rest) }
\end{lstlisting}
Here we will consider a puzzle that also involves tracking the distance travelled (odometer):
\begin{lstlisting}
walkOdo :: Prob ([RealNum],[RealNum])
walkOdo = do { start <- uniformRange 0 3 ; steps <- iid (uniformRange (-1) 1) ;
               return (scanl (+) start steps , scanl (+) 0 (map abs steps))}
\end{lstlisting}
Notice that this is a random pair of sequences, representing position and distance travelled respectively. But the pair itself is not independent, which is clear from the type.

\paragraph{Observations and Measures}
We use the distribution \lstinline|walkOdo| to phrase the puzzle from~\cite{NPHMCMakZO21}.
The walker will rest when either they reach home (position $\leq 0$) or they have travelled further than \lstinline|distLim|. When they first rest, their distance from home is roughly 1.1. What is the posterior distribution on starting positions?
%\newpage
\begin{lstlisting}[caption=Random walk model, label={lst:walk}]
walkModel :: RealNum -> Meas RealNum
walkModel distLim = do
 (xs,ds) <- sample walkOdo
 let (Just(_,finalDistance)) = find (\(x,d) -> x <= 0 || d >= distLim) (zip xs ds) in
 scoreLog (normalLogPdf 1.1 0.1 finalDistance)
 return (head xs)
\end{lstlisting}
In this example, we have included an observation by switching from the probability monad, \lstinline|Prob|, to the measures monad, \lstinline|Meas|. There is a coercion \lstinline|sample :: Prob a -> Meas a|, but the measures monad also allows \lstinline|scoreLog :: RealNum -> Meas ()|, which scores (or weights) by the likelihood of the observation. In this case, the rough distance of 1.1 from home is modelled by a normal distribution with standard deviation 0.1. 
Although \lstinline|walkOdo| is infinite, \lstinline|find| forces only a finite prefix: the odometer accumulates i.i.d. positive increments, so almost surely exceeds \lstinline|distLim| after finitely many steps, where evaluation stops.

As explained in~\cite{DashKPS23LazyPPL}, the distinction between \lstinline|Prob| and \lstinline|Meas| is helpful because recursion in the \lstinline|Prob| monad can be lazy, allowing us to build the random infinite sequences that are idiomatic in statistics, lazily ignoring any random elements that are not needed, but recursion in the \lstinline|Meas| monad is not lazy, because we cannot ignore any observations. 
This is not really a limitation: an observation must terminate to contribute a score, so observing an infinite computation is impossible in any PPL, while the laziness LazyPPL adds in the \lstinline|Prob| layer is extra expressivity.

\subsection{Example: Clustering with a Gaussian Mixture}\label{sec:clustering}
Our final illustration is a clustering model, based on an example from~\cite[\S3]{ZhouDCC}.
The number~$k$ of clusters is unknown, but the prior is that it is Poisson distributed. %; the positions of the clusters are spread out according to the number.
\begin{lstlisting}
clusters :: Prob [RealNum]
clusters = do m <- poisson 9
              let k=m+1 in
              mapM (\n -> uniformRange (20*(n-1)/k) (20*n/k)) [1..k] 
\end{lstlisting}
Here \lstinline|mapM| is a Haskell library routine that recurses over the given list using the monadic computation.
(We elide coercions between integers and real numbers.) Note that the number of points returned is not known nor bounded.

We next define a basic clustering model by treating the points \lstinline|xs <- clusters| as the means of Gaussian distributions. This uses a Gaussian mixture density:
$
\textstyle p(y|\mbox{\lstinline{xs}}) = \sum_{i=0}^{k-1} \frac 1 kf(y,\mbox{\lstinline|xs!!i|})
$,
where \lstinline|k=(length xs)| and $f(y,x)=\frac{1}{\sqrt{2\pi*1.5^2}}\exp(-\frac {(x-y)^2}{2*1.5^2})$ is the normal density, for fixed standard deviation $1.5$.

A standard clustering inference model scores a dataset according to this density function, providing the inferred centres of the clusters: 
\begin{lstlisting}
clustering :: [RealNum] -> Meas [RealNum]
clustering dataset = do xs <- sample clusters
                        mapM (\y ->  scoreLog (log (p(y|xs)))) dataset
                        return xs
\end{lstlisting}
Above, the program phrase \lstinline/(log (p(y|xs)))/ is pseudocode, the actual code takes some care over addition in the log domain, as is standard. 

This clustering model is a very simple model based on the leading example in \cite{ZhouDCC}, to illustrate the idea. More general models would also infer the variance, and infer different ratios in different clusters. This example is already non-parametric, since the number of clusters is inferred. More advanced non-parametric clustering models include Dirichlet process models, which can still be expressed within this language. 

\subsection{Recap of the Probability and Measures Monads, and the Goal of Monte Carlo}
\newcommand{\myspace}{\mbox{\lstinline|a|}}
\newcommand{\dd}{\mathrm{d}}
Probability theory typically starts from some underlying probability space or seed space $\Omega$, which is equipped with a probability measure $p$. A random variable in a space $\myspace$ is a measurable function
$X:\Omega\to\myspace$. A random variable induces a probability measure on $\myspace$ itself, by pushing forward, $X^*p$; this is the law of $X$. Often we are interested in achieving samples from the distribution $X^*p$, but the underlying space is still important and useful in Markov Chain Monte Carlo (MCMC) simulation.

%For example, if $(\myspace)$ comprises infinite lists of real numbers, a random variable in $(\myspace)$ amounts to a sequence of random variables $\Omega\to \mathbb R$. 

\subsubsection{Normalization and MCMC}
Recall that a probability measure on $\Omega$ is normalized, 
that is to say, $p(\Omega)=1$.
An unnormalized measure~$q$ has $q(\Omega)\neq 1$. Provided $q(\Omega)\not\in\{0,\infty\}$, we can form a normalized, probability measure $\frac{q(-)}{q(\Omega)}$. The problem is that the normalization constant
$q(\Omega)$ is very difficult or impossible to calculate exactly in general. Nonetheless, sampling from the normalized form of $q$ is very important, for example many Bayesian inference problems are of this form. 

The starting point for MCMC simulation is to have an \emph{unnormalized} measure $q$ on $\Omega$ (and hence an unnormalized pushforward measure on $\myspace$), and to provide a Markov chain that gives samples from the corresponding normalized probability distribution but without first calculating $q(\Omega)$.

\subsubsection{Score Functions and Unnormalized Densities}
A typical way to express an unnormalized distribution is as a function
\begin{equation}
 \label{eqn:unnorm-density}
 l: \Omega\to [0,\infty]
\end{equation}
on the underlying probability space. If $\int l\,\dd p=1$, then this is the density for a probability measure, and if not, it is merely a density for an unnormalized measure. 

Many situations in Bayesian statistics refer to the likelihood  $f(d|x)$ of a datapoint~$d$ as a function of its parameters~$x$. For fixed parameters, the function $f(-|x)$ is a normalized density, but in Bayesian inversion the datapoint is fixed and the parameters vary; the function $f(d|-)$ is an unnormalized density. If $p$ (or $X^*p$) is the prior belief, and $l: \Omega\to [0,\infty]$ determines the likelihood of the data, then the Bayesian posterior is proportional to the unnormalized measure $q=p_l$  induced by the density $l$, where
$
\textstyle \int k(\omega)\,q(\dd \omega)\ =\ \textstyle \int k(\omega)\,l(\omega)\ p(\dd\omega)
$
for all measurable $k:\Omega\to[0,\infty]$.
\subsubsection{LazyPPL Implementation}
The idea of LazyPPL is to fix a probability space $\Omega$ with a measure-preserving isomorphism $\Omega\cong\Omega\times\Omega$. Thus the source of randomness can always be split in two.
The probability monad is then implemented as \lstinline|Prob a = (|$\Omega$\lstinline| -> a)|,
and the monadic sequencing works by splitting~$\Omega$, and the correct recursive behaviour for infinite lists follows. 
A convenient example is $\Omega={\mathbb{R}}^{\mathbb{N}^*}$, infinite rose trees. Thus the sample space is naturally infinite dimensional, but this is not a problem in practice because it is explored lazily. 

The measure monad includes observations via scoring and unnormalized densities, such as the \lstinline{scoreLog} of \lstinline|walkModel| and \lstinline|clustering|, or \eqref{eqn:unnorm-density}.
This monad can be implemented as the writer monad transformer, \lstinline|Meas a|$=\Omega\to ([0,\infty],\myspace)$, providing both the
result in \lstinline|a| and the unnormalized density.

The LazyPPL library, then, provides a compositional language for building measures in \lstinline|Prob a| and \lstinline|Meas a|. It comes with basic Metropolis-Hastings MCMC simulators that convert type \lstinline|Meas a| to a stream of samples.

\subsection{Towards Gradient-Based Inference in the Lazy Setting}
\label{sec:overview-preview}
The lazy approach to probabilistic programming allows models with infinite-dimensional parameter spaces to be expressed compositionally, as illustrated above. The LazyPPL library includes a lazy Metropolis-Hastings sampler (lazyLMH) that operates over these infinite-dimensional spaces without gradient information. As discussed in the introduction, gradient-based methods like HMC can be applied to non-parametric models via truncation or dynamic dimension tracking, but these approaches sacrifice compositionality or declarativeness respectively. The aim of this paper is to show that gradient-based HMC can work directly in the lazy infinite-dimensional setting, preserving the compositional benefits.

This is not straightforward. Even defining what a `gradient' means over an infinite-dimensional space requires care, and the Metropolis-Hastings correction, which involves a product over all dimensions, appears to involve infinitely many terms (in fact it collapses to a finite product, as we explain next).

Consider again the random walk model (\cref{lst:walk}). The infinite sequence of steps produces an infinite-dimensional state space. As we demonstrate in \cref{sec:experiments} (\cref{fig:walk}), our lazy HMC methods successfully sample from the posterior of this model, despite operating over a truly infinite-dimensional space, and gradient information allows for more efficient exploration than the gradient-free lazyLMH method.

The key insight enabling this is that, although the parameter space is infinite dimensional, the gradient of any program with a PACAP likelihood (\cref{def:PACAP}) is non-zero in only finitely many dimensions. This means we can perform HMC-style leapfrog steps over only the relevant dimensions, while the infinitely many remaining dimensions are handled lazily. Moreover, on unvisited coordinates the leapfrog involution acts as a rotation preserving the standard-normal density, so the corresponding factors in the Metropolis-Hastings acceptance ratio cancel, leaving a finite product.

\paragraph{Structure of the rest of the paper.}
\Cref{sec:AD} develops the smoothness theory and automatic differentiation needed for gradients over infinite-dimensional seed spaces.
\Cref{sec:LazyHMC-big-section} presents the lazy HMC algorithm, with three variants for handling the infinite-dimensional Metropolis-Hastings correction.
\Cref{sec:Framework-B} generalizes this to the No-U-Turn Sampler (lazy NUTS), which adaptively sets the trajectory length.
Finally, \cref{sec:experiments} evaluates the methods on the models introduced in this section.

\section{Smoothness and Automatic Differentiation with Infinite Dimensional Seed Spaces}
\label{sec:AD}
\newcommand{\RR}{\mathbb{R}}
\newcommand{\NN}{\mathbb{N}}
\looseness=-1
Having motivated the need for gradient-based inference in the lazy setting (\S\ref{sec:lazyppl-overview}), we now address the first technical challenge: computing gradients over infinite-dimensional seed spaces.
We introduce a smoothness condition, PACAP, and prove compositionally that it holds for a core calculus covering primitive recursion and corecursion; for programs involving unbounded search, such as the geometric distribution of \S\ref{sec:lazyppl-overview}, PACAP can also be established by an explicit partition of the seed space.

For Hamiltonian Monte Carlo (\cref{sec:LazyHMC-big-section}) we need the gradient of the unnormalized density $l: \Omega\to [0,\infty]$ from~\eqref{eqn:unnorm-density}, where the seed space $\Omega$ is infinite dimensional. This leads to two tasks:
\begin{itemize}\item We need to establish a useful notion of gradient for~$l$. Although we want to consider infinite dimensional systems, in finite time we can only ever inspect finitely many dimensions; although we do not expect to bound the number of dimensions, we can still show that `locally' a function only depends on finitely many dimensions (`cylindrical'), and so locally (in a sense we make precise below) factors through a projection, $\Omega\to \RR^n\to\RR$.

  This local approach is inspired by earlier analysis of the definable functions $f\colon \NN^\NN\to \NN$. These are known to be locally dependent on finitely many dimensions, which amounts to topological continuity in that setting (e.g.~\cite{LongleyNormann2015}); topological continuity is too strong in the real-valued setting. %: for all $x\in \NN^\NN$ there is $n$ such that for all $y\in\NN^\NN$ with
%  $(x_1,\dots,x_n)=(y_1,\dots,y_n)$ we have that $f(x)=f(y)$. 
\item If locally the function only uses finitely many dimensions, and yet the number of dimensions used overall is not bounded, there must be discontinuity or non-smooth points in $\Omega$ where the number of relevant dimensions changes. These might play an important role. For example, even in the geometric distribution (\S\ref{sec:geometric}), the sample from the geometric distribution \emph{is} the number of dimensions used, and so it is very important that it is allowed to change. These non-smooth points could be problematic for gradient methods, but we show that they form a measure zero set: there is no chance of actually reaching a non-smooth point, in that the gradient exists almost everywhere.

This analysis is subtle: as~\cite{DBLP:conf/nips/0001YRY20} have shown, `almost surely smooth functions' are not closed under composition. We follow that work by focusing on analytic functions and analytic partitions, but now extended to the infinite dimensional setting. 
\end{itemize}
A final important high level point is that in the course of a finite computation, we will call $l$ with various arguments, although these might not be known in advance. Locally, each of these function calls will be dependent on finitely many dimensions; there is no harm in over-approximating this to say that from the perspective of all the function calls in one run of the computation, the function~$l$ itself only depends on finitely many dimensions.  
  
\subsection{Smoothness Guarantees}
\label{sec:ADtheorem}

\subsubsection{Functions That Are PACAP: Piecewise Analytic Under Cylindrical Analytic Partition}
We introduce a notion of piecewise smoothness for functions on the seed space, so that we can consider derivatives.
Because plain notions of piecewise-smooth are not compositional \cite{DBLP:conf/nips/0001YRY20}, we adopt a more refined notion that is stable under composition, following \cite{DBLP:conf/lics/HuotLMS23,DBLP:conf/nips/0001YRY20}. Moreover, we guarantee that locally the functions are cylindrical, in the sense that they are only dependent on finitely many dimensions. This makes implementation easier, and is computationally natural over an infinite dimensional structure -- an output only depends on a finite part of its input.

Let $\mathcal{A}$ be a countable set of `addresses', indexing the coordinates of the seed space. We refer to a coordinate of a particular seed as a `site'.
For example, we could let $\mathcal{A} = \mathbb{N}^*$, the set of finite lists of natural numbers, so that addresses correspond to nodes of a rose tree (a tree with countably infinite branching and infinite depth).
We think of $\mathbb{X} = \mathbb{R}^{\mathcal{A}}$ as a space of seeds, for example with the product normal distribution $\mathcal{N}(0,1)^{\otimes \mathcal{A}}$. So if $\mathcal{A} = \mathbb{N}^*$, then $\mathbb{X}$ comprises rose trees where each node contains a real number.
Recall that a function is \emph{analytic} if its Taylor series converges to the function in some neighbourhood of every point; analytic functions are smooth.

\begin{definition}\label{def:PACAP}
  A subset $U$ of $\mathbb{X}$ is an \emph{analytic cylinder} if there
exist
\begin{itemize}
\item a finite set $B\subseteq \mathcal{A}$ of addresses, and open $V\subseteq \RR^B$;
\item finite sequence of analytic functions
  $g_1,\dots, g_n\colon V\to \RR$, such that
\end{itemize}
\[U= \{x\in \mathbb{X}~|~ x|_B\in V\ \&\  g_1(x|_B)\leq 0 \ \&\ \dots\ \&\ g_n(x|_B)\leq 0\}\text.\]
We call $B$ a \emph{support} of $U$.
Note that strict inequalities $g(x|_B)< 0$ can be absorbed into the open set $V$ (replace $V$ by $V\cap g^{-1}((-\infty,0))$), so the non-strict form $g_i\leq 0$ suffices to express both strict and non-strict constraints (cf.~\cite[Lemma~B.3]{DBLP:conf/lics/HuotLMS23}).

A subset of $\mathbb{X}$ is \emph{c-analytic} if it is a countable disjoint union of analytic cylinders.

A function $f\colon U\to \RR$ with c-analytic domain is defined to be \emph{PACAP} if there
exists
\begin{itemize}
\item a countable partition $U=\biguplus_{i=1}^\infty U_i$ into analytic cylinders, with given supports $B_i\subseteq \mathcal{A}$;
\item for each $i$, an analytic function $f_i\colon V\to \RR$ where $V\subseteq \RR^{B_i}$ is open such that
  $U_i\subseteq V\times \RR^{\mathcal{A}\setminus B_i}$
\end{itemize}
and $f(x)=f_i(x|_{B_i})$ when $x\in U_i$.
\end{definition}
\begin{example}\label{ex:PAP}
  \looseness=-1
  Let $\mathcal{A}=\NN$.
  A function $c : \RR^n\to \RR$ is \emph{PAP} (piecewise analytic on an analytic partition)~\cite{DBLP:conf/nips/0001YRY20,DBLP:conf/lics/HuotLMS23} if $\RR^n$ can be partitioned into finitely many analytic sets $P_j = \{y \in V_j \mid h_1(y) \leq 0,\dots\}$ (with $V_j \subseteq \RR^n$ open, $h_i : V_j \to \RR$ analytic), together with analytic functions $c_j : V_j \to \RR$ such that $c = c_j$ on~$P_j$.
  Every PAP function $\RR^n\to \RR$ induces a PACAP map ${\mathbb{X}\to \RR^n\to \RR}$ by composition.

  The function $f:\mathbb{X}\to \RR$ given by
  $
    \textstyle f(x_1,\dots,x_n,\dots)=
    \sum_{i=1}^\infty 2^{-i}\arctan(x_i)
    $
  is not PACAP, because its result depends on every input dimension; such functions lie outside our scope, since the likelihood is then not finitely computable. \end{example}

\subsubsection{A Core Calculus for Establishing PACAP}
\label{sec:core-calculus}
% Macros for typing rules
\newcommand{\judg}[3]{#1 \vdash #2 : #3}
\newcommand{\Rule}[2]{\displaystyle\frac{#1}{#2}}%\ (\mathrm{#3})}

% Macros for terms and types
\newcommand{\tmvar}{x}
\newcommand{\tmabs}[3]{\lambda #1{:}#2.#3}
\newcommand{\tmapp}[2]{#1\,#2}
\newcommand{\tmpair}[2]{(#1,#2)}
\newcommand{\tmproj}[2]{\pi_{#1}\,#2}
\newcommand{\tmunit}{()}
\newcommand{\tminj}[2]{\mathtt{inj}_{#1}\,#2}
\newcommand{\tmcase}[4]{\mathtt{case}\ #1\ \mathtt{of}\ (\mathtt{inj}_i\,#2\Rightarrow #3)_{i\in #4}}

% Macros for type constructors 
\newcommand{\tyarrow}[2]{#1\to #2}
\newcommand{\typroduct}[2]{#1\times#2}
\newcommand{\tyunit}{1}
\newcommand{\tysum}[2]{\textstyle \sum_{#1\in #2}}
\newcommand{\tyreal}{\mathtt{real}}
\newcommand{\sem}[1]{\llbracket #1 \rrbracket}
\newcommand{\rel}[2]{\mathcal{P}^{#1}_{#2}}
\newcommand{\id}{\mathrm{id}}
\newcommand{\tmtrue}{\mathsf{tt}}
\newcommand{\tmfalse}{\mathsf{ff}}
To establish PACAP-ness compositionally, we introduce a simple core lambda calculus with a real numbers type, function types, finite product types, and countable sum types.
\[
\begin{aligned}
e &::= \tmvar \mid \tmabs{x}{\tau}{e} \mid \tmapp{e}{e} \mid \tmpair{e}{e} \mid \tmproj{j}{e} \mid \tmunit \mid \tminj{i}{e} \ (i\in I)\mid \tmcase{e}{x}{e_i}{I}\\
\tau,\upsilon &::= \tyarrow{\tau}{\upsilon} \mid \typroduct{\tau}{\upsilon} \mid \tyunit \mid \tysum{i}{I}\tau_i\mid\tyreal
\end{aligned}
\]
The typing rules are standard:
\[\small
\begin{array}{c}
\Rule{\Gamma(\tmvar)=\tau}{\judg{\Gamma}{\tmvar}{\tau}}
\qquad
\Rule{\judg{\Gamma,x:\tau_1}{e}{\tau_2}}{\judg{\Gamma}{\tmabs{x}{\tau_1}{e}}{\tyarrow{\tau_1}{\tau_2}}}
\qquad
\Rule{\judg{\Gamma}{e_1}{\tyarrow{\tau_1}{\tau_2}}\quad \judg{\Gamma}{e_2}{\tau_1}}{\judg{\Gamma}{\tmapp{e_1}{e_2}}{\tau_2}}
\qquad
\Rule{}{\judg{\Gamma}{\tmunit}{\tyunit}}
\qquad
\Rule{\judg{\Gamma}{e}{\typroduct{\tau_1}{\tau_2}}}{\judg{\Gamma}{\tmproj{j}{e}}{\tau_j}}
\\[12pt]
\Rule{\judg{\Gamma}{e_1}{\tau_1}\quad \judg{\Gamma}{e_2}{\tau_2}}{\judg{\Gamma}{\tmpair{e_1}{e_2}}{\typroduct{\tau_1}{\tau_2}}}
\qquad
\Rule{\judg{\Gamma}{e}{\tau_i}}{\judg{\Gamma}{\tminj{i}{e}}{\tysum{j}{I}{\tau_j}}}
\qquad
  \Rule{\judg{\Gamma}{e}{\tysum{j}{I}{\tau_j}} \quad \{\judg{\Gamma,x:\tau_i}{e_i}{\tau}\}_{i\in I}}
     {\judg{\Gamma}{\tmcase{e}{x}{e_i}{I}}{\tau}}
\end{array}
\]
We also include standard functions as constants, such as $\exp$, $\sin$, $\log$, ${+}$, ${\times}$, typed as appropriate, ${\judg\Gamma{c}{\tyarrow{\typroduct{\tyreal}{\cdots}}{\tyreal}}}$.
We assume each constant denotes a total PAP function (Example~\ref{ex:PAP}); this covers all analytic functions as well as functions like $\log$ and $\mathtt{abs}$ (extended to total functions, e.g.\ $\log(x) = 0$ for $x \leq 0$).

We interpret types as sets, in a standard way:
\[\sem\tyreal=\RR\quad
  \sem{\tyunit}=1\quad
  \sem{\typroduct{\tau_1}{\tau_2}}=\sem{\tau_1}\times \sem{\tau_2}\quad
  \sem{\tysum{j}J{\tau_j}}=\biguplus_{j\in J}\sem{\tau_j}
  \quad
  \sem{\tyarrow{\tau_1}{\tau_2}}=  (\sem{\tau_1}\to\sem {\tau_2})\text.\]
We then interpret typing contexts $\Gamma=(x_1{:}\tau_1,\dots,x_n{:}\tau_n)$ as sets of valuations, i.e.\ products $\sem\Gamma=\sem{\tau_1}\times\cdots\times\sem{\tau_n}$, and every typed term $\judg\Gamma e \tau$ induces a function $\sem e:\sem\Gamma\to\sem \tau$ in a standard way (e.g.~\cite{10.5555/151145,reynolds}). 

 The countable sum types include every countable set  $I$ as a type, via $\tysum{i}{I}\tyunit$.
 In particular we have a type of booleans $\{\tmtrue,\tmfalse\}$, a type of natural numbers $\NN$, and a type of lists of natural numbers~$\NN^*$.
 We have a type of infinite streams, $(\NN\to \tyreal)$, and a type of lazy rose trees, $(\NN^*\to \tyreal)$. 

 Although countable sum types cannot be directly coded on a computer, countable case analysis gives a simple calculus that subsumes primitive recursion (over~$\NN$) and corecursion (into streams $\NN\to\tyreal$), and many of our statistical models fit inside this calculus; in practice, Haskell's general recursion is used to describe the resulting infinite terms finitely.
 Programs involving unbounded search, such as \lstinline|findIndex| and \lstinline|find|, go beyond this calculus, but their PACAP property can be verified directly.
 See \cref{appendix:recursion-encoding} for details.
 
\begin{theorem}\label{thm:ad}
  If $f:\sem{(\mathcal A\to \tyreal)\to\tyreal}$ is definable then it is PACAP.
\end{theorem}
(Here $\mathcal A$ is a type since it is countable, via $\tysum{a}{\mathcal A}\tyunit$, and $\sem{\mathcal A\to\tyreal}=\RR^{\mathcal A}$.)
\begin{proof}
  We prove this by infinitary logical relations; the proof is reminiscent of \cite{DBLP:conf/esop/BartheCLG20,DBLP:conf/fossacs/HuotSV20}. For each type $\tau$ and each c-analytic set $U$ we define
  a relation $\rel U \tau\subseteq \sem \tau ^U$, by induction on structure of types:
  \begin{align*}
    &\rel U \tyreal =\{f:U\to \RR~|~ f\text{ is PACAP}\}\\
    &\rel U \tyunit = \{(): U\to 1\} \quad\text{(all)}\\
    &\rel U{\typroduct {\tau_1}{\tau_2}} = \{(f_1,f_2)~|~f_1\in \rel U {\tau_1} \ \& \ f_2\in \rel U {\tau_2}\}\\
    &\textstyle \rel U{\tysum i I {\tau_i}} = \{\biguplus_{i\in I} f_i~|~U=\biguplus_{i\in I} U_i\ \&\ \forall i\in I.\ \text{$U_i$ c-analytic} \ \&\  f_i\in \rel {U_i} {\tau_i}\}\\
    &\rel U{\tyarrow {\tau_1}{\tau_2}}=
      \{f:U\to \sem{\tau_1}\to \sem{\tau_2}~|~\forall\text{ c-analytic }U'\subseteq U,\ \forall g\in \rel {U'} {\tau_1}.\ (\lambda u.\,f(u)(g(u)))\in \rel {U'} {\tau_2}\}
  \end{align*}
  We prove the following `fundamental lemma', by induction on the structure of typing derivations:
  \begin{equation}\label{eqn:logrel-fundlemma}\begin{aligned}
    &\text{If $\judg {x_1:\tau_1,\dots, x_n:\tau_n} e \tau$ and given c-analytic $U$ and $f_1\in \rel U {\tau_1}$, \dots, $f_n\in \rel U {\tau_n}$}\\&\text{then we have}\quad
    \lambda u.\ \sem{e}(f_1(u),\dots, f_n(u))\in \rel U \tau\text.
   \end{aligned}
 \end{equation}
 The full proof, which uses auxiliary lemmas on support enlargement, intersection, restriction, gluing, and composition, is given in \cref{appendix:fundamental-lemma}.
 In the function-type clause the subsets~$U'$ range over c-analytic subsets of~$U$; no restriction of~$f$ is involved, since for $u\in U'\subseteq U$ the value $f(u)\in\sem{\tau_1}\to\sem{\tau_2}$ is just~$f$ evaluated pointwise. As $g(u)\in\sem{\tau_1}$ for $g\in\rel{U'}{\tau_1}$, the map $\lambda u.\,f(u)(g(u))\colon U'\to\sem{\tau_2}$ is well-typed, and the clause asks that it lie in $\rel{U'}{\tau_2}$. This quantification over subsets is a Kripke refinement~\cite{DBLP:conf/esop/BartheCLG20,DBLP:conf/fossacs/HuotSV20}, needed for the restriction and lambda cases.
 
  The main result then follows from this lemma~\eqref{eqn:logrel-fundlemma}. For if $\judg{}e{(\mathcal A\to \tyreal)\to\tyreal}$
  then by the special case of the lemma with empty context, the constant function
  \begin{equation}\text{$\lambda u.\ \sem{e}\quad$ is in $\quad\rel {\RR^{\mathcal A}} {(\mathcal A\to \tyreal)\to\tyreal}$}\text.\label{eqn:logrelA}\end{equation}
  By expanding the definitions, the identity function $\id_{\RR^{\mathcal A}}$ is in $\rel {\RR^{\mathcal A}}{\mathcal A\to \tyreal}$.
  Therefore, using \eqref{eqn:logrelA} and expanding the definition of $\rel {\RR^{\mathcal A}} {(\mathcal A\to \tyreal)\to\tyreal}$, putting $f_1$ as the identity function, we have ${\lambda u.\sem e(u)\in \rel {\RR^{\mathcal A}}\RR}$, so $\sem e$ is PACAP. 
      \end{proof}

      \subsection{Representing the Probability and Measures Monads in the Core Lambda Calculus}
      \label{sec:monads-in-lambda-calculus}
      \newcommand{\ProbMonad}{\mathit{Prob}}      \newcommand{\MeasMonad}{\mathit{Meas}}
      \newcommand{\mreturn}{\mathit{return}}
      \newcommand{\mbind}{\mathit{bind}}
We can define the probability monad in our calculus,
by
\begin{align*}&\ProbMonad(\tau) = ((\NN^*\to \tyreal)\to \tau)\qquad
  \mreturn_\ProbMonad = \lambda x .\lambda t.\ x : \tau \to \ProbMonad(\tau)\\
&\mbind_\ProbMonad = \lambda p.\lambda f.\lambda t.f(p\,(\lambda j.t (\imath_1\,j)))(\lambda j.t\, ( \imath_2\,j))
: \ProbMonad(\tau)\to (\tau \to \ProbMonad(\upsilon))\to \ProbMonad(\upsilon)\end{align*}
where $\imath_1,\imath_2: \NN^*\to\NN^*$ are two maps with disjoint range.
In this way, sequencing splits the random seed space.

As in LazyPPL~\cite{DashKPS23LazyPPL}, this monad does not immediately satisfy the identity and associativity laws for monads in the usual theory of equality for the calculus here. Rather, it satisfies the associativity laws in a suitable semantics where we identify random variables by their law (e.g.~\cite{DashKPS23LazyPPL,DBLP:conf/lics/HeunenKSY17,DBLP:conf/lics/HuotLMS23}). Since we do not need to fully formalize a denotational semantics of the language for what follows, we omit the details here. 

We can also define the measure monad, which accumulates the unnormalized density, using the standard writer monad transformer:
\begin{align*}&\MeasMonad(\tau) = (\ProbMonad (\tyreal \times\tau))\qquad
  \mreturn_\MeasMonad = \lambda x .\mreturn_\ProbMonad(1,x) : \tau \to \MeasMonad(\tau)\\
              &\mbind_\MeasMonad = \lambda p.\lambda f.\mbind_{\ProbMonad}\,p\,(\lambda q.\,
                \mbind_\ProbMonad\,(f(\tmproj 2 q))\,(\lambda r.\,\mreturn_\ProbMonad((\tmproj 1 q)*(\tmproj 1 r),\tmproj 2 r)))
\\&\qquad\qquad\qquad
: \MeasMonad(\tau)\to (\tau \to \MeasMonad(\upsilon))\to \MeasMonad(\upsilon)\end{align*}

The point is that PACAP is a sufficient condition for lazy HMC.
The clustering, step regression, and polynomial regression models are all definable in the core calculus, and hence PACAP by \cref{thm:ad}; the geometric and random walk models involve unbounded search, but their PACAP property can be verified directly (\cref{appendix:recursion-encoding}).
Once we know that the likelihood is PACAP, it follows that at almost every seed the likelihood is locally analytic and depends on only finitely many coordinates, so its gradient is well-defined and finitely supported (the boundaries between analytic pieces form a measure-zero set). This is exactly the property exploited by the HMC methods of \cref{sec:LazyHMC-big-section}.

\subsection{Implementation of Automatic Differentiation in Haskell}
\label{sec:Haskell-Nagata}
\newcommand{\Address}{\mathcal{A}}
\newcommand{\Nag}{\mathcal{R}}
We automatically propagate the derivatives of functions by automatic differentiation, using the sparse map version of the Nagata numbers from~\cite{vandenBerg2023ForwardReverse}. This works even for infinite dimensional spaces, as we now explain. 

Recall that a naive approach to dual number forward-mode AD (e.g.~\cite{DBLP:conf/fossacs/HuotSV20,DBLP:journals/pacmpl/ShaikhhaFVJ19}) would replace a function ${f:\RR^n\to\RR}$ by its dual number form,
$\bar f:\RR^n\times \RR^n\to \RR^2$ and the partial derivative $\frac {\partial f(x_1\dots x_n)}{\dd x_i}$ is given by 
$\pi_2\circ \bar f(-,e_i):\RR^n\to \RR$ where $e_i$ is the one-hot vector. We can then find the gradient at a point by forming a vector of all the partial derivatives. 

In the infinite dimensional setting, we consider functions of the form
$f:\RR^\Address\to \RR$ (for $\Address=\NN^*$ or $\NN$, the set of all dimensions). From Theorem~\ref{thm:ad}, it is plausible that we can take a step in the direction of the gradient at any given point: the
function $f$ will locally only depend on finitely many dimensions, and so the gradient at any point will only be non-zero at finitely many dimensions. 
However, a naive forward-mode AD (see e.g.~\cite{DBLP:conf/fossacs/HuotSV20,DBLP:journals/pacmpl/ShaikhhaFVJ19}) would work in terms of the dual number form $f':\RR^\Address\times \RR^\Address\to \RR^2$, from which we can again find the gradients by passing in one-hot vectors, but it is not obvious how to actually find out which of the infinitely-many dimensions have non-zero gradient, and so it is not clear how to take a gradient step.

To circumvent this problem, we follow~\cite{vandenBerg2023ForwardReverse} and switch from the ordinary dual numbers view to one in which a number is represented by a pair $(x,m)$ where $x\in \RR$ and $m:\Address\to \RR$ is a map. We write $\Nag=\RR\times \RR^\Address$ for these `Nagata numbers'. Now a function $f:\RR^\Address\to \RR$ is transformed by automatic differentiation into a function
$ f':\Nag^\Address\to \Nag$. For any $\vec{x}\in \RR^\Address$, we have 
$f'(\lambda i.(x_i,e_i))\in\Nag$, for which the first component is $f(\vec x)$, and the second component is a
map $\Address\to \RR$, assigning the gradient at each address.

The source-to-source transformation from $f:\RR^\Address\to \RR$ to $\bar f:\Nag^\Address\to \Nag$ is totally automatic in Haskell, simply by defining $\Nag$ as an instance of the \lstinline|Floating| type class. 

The final trick is to note that, for all relevant $(x,m)\in\Nag$, the map $m$ is sparse, i.e.~zero except at finitely many points. For this reason we can use a sparse map datatype for this, and this sparse map datatype allows us to inspect its domain of definition. We can then discover the dimensions in $\Address$ that are relevant to the gradient of $f$ at $\vec x\in \RR^\Address$ by inspecting the domain of the sparse map arising from $\bar f(\lambda i.(x_i,e_i))$. Even though this question requires infinitely many $x_i$'s and $e_i$'s, Haskell's lazy evaluation will ensure that only the relevant dimensions are inspected.

This sparse map approach with Nagata numbers was proposed by~\cite{vandenBerg2023ForwardReverse} to connect to the efficiencies of reverse mode automatic differentiation in a purely functional setting: rather than computing one directional derivative at a time (as in standard forward-mode AD), the full gradient is computed in a single pass, which in finite dimensions is the key advantage of reverse mode. In the infinite dimensional setting, this is not merely an efficiency gain but \emph{essential}: standard forward-mode AD would require infinitely many passes (one per dimension) to recover the gradient, whereas the Nagata approach discovers the finitely many relevant dimensions automatically. We use this throughout the HMC methods of \cref{sec:LazyHMC-big-section}.

\section{Lazy HMC: Lazy Hamiltonian Monte Carlo Simulation}
\label{sec:LazyHMC-big-section}
With automatic differentiation over infinite-dimensional structures in hand, via the PACAP analysis and laziness of the previous section, we turn to the Monte Carlo method, providing a correctness framework for HMC-based inference on rose-tree states arising from lazy probabilistic programs.

\paragraph{The key insight.} The main difficulty in applying HMC to lazy probabilistic programs is that the state space is infinite dimensional. 
A trace contains an unbounded collection of random seeds, and so a naive HMC update would appear to require both an infinite dimensional gradient and an infinite product in the acceptance ratio used for correction.

The key observation is that neither computation is actually infinite for the programs considered in this paper. 
By the PACAP analysis of \cref{sec:AD}, at almost every trace the likelihood is locally analytic and depends only on finitely many \emph{visited sites}. 
These visited sites are the coordinates whose values are actually read when evaluating the likelihood function on the given trace. 
Moreover, laziness ensures that only those sites needed along the leapfrog trajectory (defined in \cref{sec:HMC}) are ever sampled.
Together with a carefully modified HMC integrator, the infinite product in the acceptance ratio collapses to a \emph{finite} product over sites visited during the trajectory, and the unvisited dimensions cancel exactly.

Thus each lazy HMC proposal is computed by a finite procedure:
\begin{enumerate}
  \item sample an infinite momentum rose tree lazily, so that momentum values are generated only when demanded
  \item evaluate the likelihood and gradient only at the finitely many sites needed by the trajectory
  \item collect the finite set of sites visited during the trajectory
  \item keep the HMC update on the visited sites that was computed during the trajectory
  \item choose the update on all unvisited sites so their contributions cancel in the acceptance ratio.
\end{enumerate}

The purpose of this section is to make this argument precise.
We use the involutive MCMC framework as the correctness principle: if a proposal is generated by a measurable involution and the acceptance ratio uses the corresponding change-of-measure term, then the resulting Markov kernel preserves the target distribution.  
Our framework A (\cref{sec:Framework-A}) does not prescribe a unique proposal involution.  
It only states sufficient conditions under which the infinite-dimensional acceptance ratio reduces to a finite product.  
Our different lazy HMC variants instantiate the framework but  satisfy the conditions in different ways:
\begin{itemize}
  \item \textbf{lazyHMC1} (\cref{sec:LazyHMC}): uses a
    \emph{rotation}-based position update.
    Because the standard normal is rotationally symmetric, the momentum density at unvisited sites is unchanged by the update, so those dimensions cancel immediately from the acceptance ratio.
    Only the addresses with non-zero gradient need to be tracked so no full visited-site bookkeeping is required.
    This is the simplest, most purely functional variant.

  \item \textbf{lazyHMC2} (\cref{sec:LazyHMC-Hilbert}): uses a
    different Hamiltonian splitting inspired by~\cite{beskos2011HMConHS}.
    The position update is again rotation-based, so it shares lazyHMC1's cancellation.
    It can be preferable when the prior's geometry interacts poorly with lazyHMC1's step size parameterisation.

  \item \textbf{lazyHMC3} (\cref{sec:LazyHMC-with-original-integrator}):
    uses the \emph{standard} HMC leapfrog integrator, whose position update is a \emph{translation}.
    Translations do not preserve the standard normal, so unvisited sites no longer cancel automatically.
    To restore cancellation, the involution must apply a corrective measure-preserving map to unvisited sites, which requires knowing the full set of visited sites $v(\mathbf{q},\mathbf{p})$ (not just the gradient-support $u$).
    This requires an extra $O(L)$ pass to collect visited thunks (a larger constant, same asymptotics).

  \item \textbf{lazyNUTS} (\cref{sec:LazyNUTS}): adapts the
    No-U-Turn sampler to the lazy setting, eliminating the need to hand-tune the trajectory length $L$.
\end{itemize}

\cref{sec:experiments} shows that no single
variant dominates across all models.

The step regression model of \cref{sec:stepReg} is the running
illustration throughout this section.
There, the parameter space contains infinitely many changepoint positions and segment heights, yet any evaluation of the likelihood on a concrete dataset only touches the finitely many segments that overlap the observed $x$-range (\cref{fig:hmc-dim-change}).
A reader who finds the formal development of this section dense is encouraged
to read \cref{sec:stepReg} in parallel.

The section is structured as follows.
\Cref{sec:mcmc-correctness,sec:iMCMC,sec:HMC} recall background on Markov kernels and stationarity, the iMCMC framework, and HMC.
\Cref{sec:Framework-A} then presents Framework A as an instance of iMCMC on rose trees, and \cref{sec:LazyHMC,sec:LazyHMC-Hilbert,sec:LazyHMC-with-original-integrator} derive HMC-inspired methods whose correctness follows from it.
We also briefly discuss lazyNUTS (\cref{sec:LazyNUTS}).

\subsection{Markov Kernels and Stationarity}
\label{sec:mcmc-correctness}

Let $(\mathbb{X}, \Sigma_{\mathbb{X}}, \mu_{\mathbb{X}})$ be a $\sigma$-finite measure space and let $l:\mathbb{X} \rightarrow \mathbb{R}_{\geq 0}$ be the (unnormalized) density of the target distribution with respect to $\mu_{\mathbb{X}}$. 
The aim is to produce Markov chains $(\mathbf{q}^{(i)})_{i = 0, 1, \dots}$ with $\mathbf{q}^{(i)}\in \mathbb{X}$ which converge to the target distribution, i.e. the draws are eventually distributed according to the target. 
This is done by constructing a Markov kernel (also known as probability kernel) $\kappa: \mathbb{X} \times \Sigma_{\mathbb{X}} \rightarrow [0,1]$ which assigns to each current state a probability distribution over next states. One step of the Markov chain then corresponds to drawing the next  state $\mathbf{q}'$ from the kernel given the current state $\mathbf{q}$ by sampling from $\kappa(\mathbf{q}, \cdot)$. 

For the chain to converge, the kernel has to be stationary with respect to the target measure, i.e. it leaves the target measure invariant:
$\int_{\mathbb{X}}\kappa(\mathbf{q}, A)l(\mathbf{q}) \mu_{\mathbb{X}}(\diff \mathbf{q}) = \int_{A}l(\mathbf{q}) \mu_{\mathbb{X}}(\diff \mathbf{q})$.
Intuitively, if the current state is distributed according to the target, then after one step it still is.

Invariance alone does not guarantee convergence; that requires further assumptions such as irreducibility and aperiodicity, which are not typically established in full generality for probabilistic programming systems.
Our correctness results therefore focus on proving invariance for the HMC-related kernels we define, and we evaluate mixing empirically in \cref{sec:experiments}.
Convergence rate is a separate question: it depends on the target's geometry, with no single theorem covering all targets \cite{livingstone2019geometricHMC}. Since laziness changes only \emph{when} dimensions are evaluated, not the target geometry, standard HMC tuning guidance carries over.

\subsection{Background on Involutive Markov Chain Monte Carlo}
\label{sec:iMCMC}
Involutive MCMC (iMCMC) \citep{automatingIMCMC2020,  IMCMCNeklyudovWEV20,andrieu2020MHperspective,NPiMCMCMakZO22} is a general framework for constructing 
valid MCMC kernels with the use of an auxiliary random variable and an involutive function, which is a function whose inverse is itself.
In particular, we will focus on the presentation of iMCMC introduced by \cite{automatingIMCMC2020}.
Algorithms that can be generalized by iMCMC include MH, HMC, LMH, Reversible jump  (see \cite{IMCMCNeklyudovWEV20}). 

\looseness=-1
Let $(\mathbb{Y}, \Sigma_{\mathbb{Y}}, \mu_{\mathbb{Y}})$ be the $\sigma$-finite measure space of the auxiliary variable $\mathbf{p}$.
For each $\mathbf{q} \in \Supp(l)$, the auxiliary distribution has probability density $p_{\mathbf{q}}:\mathbb{Y} \rightarrow \mathbb{R}_{\geq 0}$ with respect to $\mu_{\mathbb{Y}}$.
Consider the state space $(\mathbb{S}, \Sigma_{\mathbb{S}}, \mu_{\mathbb{S}})$ with $\mathbb{S} = \{(\mathbf{q}, \mathbf{p}) \in \mathbb{X}\times \mathbb{Y} \mid l(\mathbf{q})p_\mathbf{q}(\mathbf{p})>0\}$, $\Sigma_{\mathbb{S}} = \{A \cap \mathbb{S} \mid A \in \Sigma_{\mathbb{X}}\otimes\Sigma_{\mathbb{Y}}\}$ and $\mu_{\mathbb{S}} = \mu_{\mathbb{X}} \times \mu_{\mathbb{Y}}$.
The target distribution on the joint space has (unnormalized) density 
$w(\mathbf{q}, \mathbf{p})= l(\mathbf{q})p_\mathbf{q}(\mathbf{p})$.

We also require an involution $I: \mathbb{S} \rightarrow \mathbb{S}$, i.e., $I^{-1} = I$, with the property that the Radon-Nikodym derivative $\diff (\mu_{\mathbb{S}} \circ I^{-1})/\diff \mu_{\mathbb{S}} : \mathbb{S} \rightarrow [0, \infty)$ exists.
The involution structure simplifies the construction and verification of a reversible Metropolis--Hastings correction, which guarantees the correct stationary distribution (see \cite{automatingIMCMC2020} for details).
One trivial example would be the identity involution, which always proposes the same value $\mathbf{q}$. Even though this kernel is stationary it will not produce a chain which converges to the correct distribution.   

\paragraph{iMCMC}
Each iteration starts with a value for $\mathbf{q}$ and draws a value for $\mathbf{p}$ from the auxiliary distribution $p_{\mathbf{q}}(\mathbf{p})$ of $\mathbf{p}$ given $\mathbf{q}$.
The involution $I$ is then applied on $(\mathbf{q}, \mathbf{p})$ in order to get the proposed state $(\mathbf{q}', \mathbf{p'})$.
For correctness, an accept/reject step is needed. 
The ratio compares how likely the proposed state is under the target distribution versus the current state, correcting for any asymmetry in the proposal distribution.
With some probability, the proposed $\mathbf{q}'$ is accepted and returned, otherwise it is rejected and $\mathbf{q}$ is returned instead. 

One iteration of iMCMC given $\mathbf{q}, p_{\mathbf{q}}, l, \mu_{\mathbb{X}}, \mu_{\mathbb{Y}}, I$ does:

\begin{enumerate}
    \item sample $\mathbf{p} \sim p_{\mathbf{q}}$
    \item get proposed state $(\mathbf{q}', \mathbf{p}') = I(\mathbf{q}, \mathbf{p})$
    \item with probability $\tilde{\alpha}(\mathbf{q}, \mathbf{p}) = \min\{1, \alpha(\mathbf{q}, \mathbf{p})\}$ accept and return $\mathbf{q}'$, otherwise return $\mathbf{q}$.
\end{enumerate}
The acceptance ratio is: 
\[\alpha(\mathbf{q}, \mathbf{p}) = \frac{w(\mathbf{q}', \mathbf{p}')}{w(\mathbf{q}, \mathbf{p})} \cdot \left( \frac{\diff (\mu_{\mathbb{S}} \circ I^{-1})}{\diff \mu_{\mathbb{S}}} \right)(\mathbf{q}, \mathbf{p}).\]
Let $\kappa: \mathbb{X}\times \Sigma_{\mathbb{X}} \rightarrow [0,1]$ (see \cref{app:iMCMCkernel}) be the probability kernel resulting from one iteration of iMCMC.
Informally, $\kappa(\mathbf{q}, A)$ gives the probability of starting with a position $\mathbf{q}$ and after one iMCMC iteration returning a position in $A$.

\begin{proposition}\label{prop:iMCMC-stationary-kernel}\emph{(iMCMC stationarity \citep{automatingIMCMC2020})} The kernel $\kappa$ is stationary with respect to the target distribution, i.e. for any $A \in \Sigma_{\mathbb{X}}$ we have $\int_{\mathbb{X}}\kappa(\mathbf{q}, A)l(\mathbf{q}) \mu_{\mathbb{X}}(\diff \mathbf{q}) = \int_{A}l(\mathbf{q}) \mu_{\mathbb{X}}(\diff \mathbf{q}).$
\end{proposition}

\subsection{Background on Hamiltonian Monte Carlo}
\label{sec:HMC}
HMC is a popular MCMC method that uses Hamiltonian dynamics on an augmented state space to propose the next sample. 
We give a brief description of HMC here, for more details see \cite{Neal2012HMC,leimkuhler2004simulatingHD,betancourt2017conceptual}.

Since we are looking at the standard HMC in this subsection our target (unnormalized) density $l$ is defined on the finite dimensional space $(\mathbb{R}^n, \Sigma_{\mathbb{R}^n}, \text{Leb}_n)$. Suppose $l$ is continuously differentiable.

\paragraph{Physical intuition.}
Imagine placing a frictionless particle on a surface whose height at position $\mathbf{q}$ equals the \emph{potential energy} $U(\mathbf{q}) = -\log l(\mathbf{q})$.
The particle naturally slides toward low-potential regions, i.e. toward regions of high $l$, but without momentum it stops at a local mode.
To explore the surface better, we periodically give the particle a random momentum kick and let it slide for a while before reading its new position.
The momentum kick helps HMC explore
the target distribution far more efficiently than random-walk
Metropolis.
The momentum is usually standard $n$-dimensional normal.
This gives the \emph{kinetic energy} $K:\mathbb{R}^n \rightarrow \mathbb{R}$ with $K(\mathbf{p}) = - \log \varphi_n(\mathbf{p}) = \sum_{i=1}^n \mathbf{p}_i^2/2$.
The Hamiltonian of the system is given by $H:\mathbb{R}^{2n} \rightarrow \mathbb{R}$ with $H(\mathbf{q}, \mathbf{p}) = U(\mathbf{q}) + K(\mathbf{p})$.
Hamiltonian equations describe how the position $\mathbf{q}$ and momentum $\mathbf{p}$ change over time: 
\begin{align*}
\frac{\diff \mathbf{q}}{\diff t} = \frac{\partial H}{\partial \mathbf{p}} = \mathbf{p}, \qquad 
\frac{\diff \mathbf{p}}{\diff t} = -\frac{\partial H}{\partial \mathbf{q}} = - \nabla U(\mathbf{q}).
\end{align*}
where $\partial K / \partial \mathbf{p} = \mathbf{p}$ since $K(\mathbf{p}) = \sum_i \mathbf{p}_i^2/2$.

HMC makes use of the Hamiltonian motion of the particle to propose the next position. 
This is done by sampling the momentum $\mathbf{p}$ from $\varphi_n$ and then simulating the trajectory of the particle with initial position $\mathbf{q}$.
After some time $t$, the particle will have momentum $\mathbf{p}'$ and position $\mathbf{q}'$, which will be the proposed position.

The canonical distribution corresponding to $H$ on the state space $(\mathbb{R}^{2n}, \Sigma_{\mathbb{R}^{2n}}, \text{Leb}_{2n})$ is given by $\zeta(\mathbf{q}, \mathbf{p})= \exp(-H(\mathbf{q}, \mathbf{p})) = l(\mathbf{q}) \varphi_n(\mathbf{p})$.
Hence, having draws $\{(\mathbf{q}^i, \mathbf{p}^i)\}_{i = 1,\dots}$ from $\zeta$ and discarding the draws for the momentum $\mathbf{p}^i$, gives us position draws $\{\mathbf{q}^i\}_{i = 1,\dots}$ from our target distribution $l$. 

\paragraph{The leapfrog integrator (kick–drift–kick).}
In almost all cases the Hamiltonian equations cannot be solved exactly, so we discretize time with the \emph{leapfrog} integrator.
One step of size $\epsilon$ alternates updating the momentum and the position:
\begin{enumerate}
  \item \textbf{Half-kick}: update momentum using the gradient of the potential for half a step, $\phi_{\epsilon/2}^P(\mathbf{q},\mathbf{p}) = (\mathbf{q},\,\mathbf{p} - \frac{\epsilon}{2}\nabla U(\mathbf{q}))$.
  \item \textbf{Drift}: update position using the new momentum, $\phi_{\epsilon}^Q(\mathbf{q},\mathbf{p}) = (\mathbf{q}+\epsilon\mathbf{p},\,\mathbf{p})$.
  \item \textbf{Half-kick}: update the momentum again for half a step, $\phi_{\epsilon/2}^P(\mathbf{q},\mathbf{p}) = (\mathbf{q},\,\mathbf{p} - \frac{\epsilon}{2}\nabla U(\mathbf{q}))$.
\end{enumerate}
One full leapfrog step is $\psi = \phi_{\epsilon/2}^P\circ\phi_{\epsilon}^Q\circ\phi_{\epsilon/2}^P$, and $L$ steps give $\Psi^{(L)} = F\circ\psi^L$, where $F(\mathbf{q},\mathbf{p}) = (\mathbf{q},-\mathbf{p})$ is the momentum flip needed for reversibility.
We omit step-size $\epsilon$ and step count $L$ from the notation when fixed.

The symmetric kick-drift-kick form of the leapfrog integrator is needed for reversibility: after negating the momentum, the same
update can be run backwards. 
This would not be true if a simpler form like kick-drift was used. 
Reversibility and volume preservation are key properties of Hamiltonian
dynamics. They are central to the proof that the HMC transition leaves the
target distribution invariant. The leapfrog integrator is used because it
preserves these properties:
\begin{proposition} \citep{Neal2012HMC} \label{prop:leapfrog-reversible-volume-preserving} The HMC leapfrog integrator is reversible, i.e. $\Psi^{(L)} = (\Psi^{(L)})^{-1}$ and volume preserving, i.e. $(\text{Leb}_{2n}\circ(\Psi^{(L)})^{-1})(A) = \text{Leb}_{2n}(A)$ for any measurable $A \in \Sigma_{\mathbb{R}^{2n}}$.
\end{proposition}

\paragraph{HMC as part of iMCMC} We follow the structure of the iMCMC framework from \cref{sec:iMCMC} with the target distribution on the space $(\mathbb{R}^n, \Sigma_{\mathbb{R}^n})$ given by the density $l$ with respect to the Lebesgue $\text{Leb}_n$ measure.
The momentum variable is the auxiliary variable from the iMCMC framework with the $n$-dimensional normal $\mathcal{N}_n$ as the auxiliary distribution, which has density $\varphi_n$ with respect to the Lebesgue measure.
The target density on the joint state space $(\mathbb{R}^{2n}, \Sigma_{\mathbb{R}^{2n}})$ is given by the density $\zeta$ with respect to $\text{Leb}_{2n}$.
Our involution now is the leapfrog integrator $\Psi^{(L)}$.

One iteration of HMC given position $\mathbf{q}, \varphi_n, l, \Psi^{(L)}$ does: 
\begin{enumerate}
  \item sample $\mathbf{p} \sim \varphi_n$
  \item get proposed position and momentum $(\mathbf{q}', \mathbf{p}') = \Psi^{(L)}(\mathbf{q}, \mathbf{p})$
  \item with probability $\min \{1, \alpha(\mathbf{q}, \mathbf{p})\}$ accept and return $\mathbf{q}'$, otherwise return $\mathbf{q}.$
\end{enumerate} 
Now we just have to make sure the correct acceptance ratio is used:
\begin{proposition}
  \label{prop:HMCacc_ratio}
  The probability kernel resulting from HMC is stationary with respect to the target density $l$, if the acceptance ratio is given by: 
  \[\alpha(\mathbf{q}, \mathbf{p}) = \frac{l(\mathbf{q}')\varphi_n(\mathbf{p}')}{l(\mathbf{q})\varphi_n(\mathbf{p})}.\]
\end{proposition}
\begin{proof}
  The proof follows by \cref{prop:iMCMC-stationary-kernel} and the fact that leapfrog preserves volume (contributing a Radon–Nikodym factor of 1). See \cref{proof:accHMC} for details.
\end{proof}

The acceptance ratio deduced here corresponds to the usual HMC acceptance ratio (see for example \cite{Neal2012HMC}).
If we could simulate the trajectory of the particle exactly we would always accept the proposed position. That is because Hamiltonian dynamics keeps $H$  invariant ($\frac{dH}{dt} = 0$) which would give $\alpha(\mathbf{q}, \mathbf{p}) = 1$.
So the acceptance ratio is needed to correct for the approximation error introduced by the numerical integrator. 

\paragraph{Example: Harmonic Oscillator.} Consider the case of the Harmonic Oscillator with Hamiltonian $H(\mathbf{q}, \mathbf{p}) = \frac{\mathbf{q}^2}{2} + \frac{\mathbf{p}^2}{2}$. 
The exact flow map is $\phi_{\theta}(\mathbf{q}, \mathbf{p}) = (\cos (\theta) \mathbf{q} + \sin (\theta) \mathbf{p}, -\sin (\theta) \mathbf{q} + \cos (\theta) \mathbf{p})$, so applying it will always preserve the Hamiltonian. 
We will make use of this in the following sections. 

\subsection{Framework A: Hamiltonian Monte Carlo on Rose Trees in General}
\label{sec:Framework-A}
\paragraph{Section overview.}
A \emph{rose tree} is an infinite map $\mathbf{q}:\mathcal{A}\to\mathbb{R}$ from a countable address set $\mathcal{A}$ (e.g. $\mathcal{A} = \mathbb{N}^*$, the set of finite lists of natural numbers) to the reals; it represents the full (lazy) parameter state of a probabilistic program.
(In step regression, for example, $\mathbf{q}$ stores the infinite stream of changepoints and the infinite collection of segment heights, indexed by a path in~$\mathcal{A}$.)

The challenge is that the iMCMC acceptance ratio formally involves a product over all of $\mathcal{A}$, which is infinite.
Framework~A shows how that product collapses to a finite one, provided the involution $I$ satisfies the three \cref{cond:visited-property,cond:non-visited-cancel-out,cond:involution-decomposition} below.
The key output is \cref{thm:alg1-acceptance-ratio}.
The three lazy HMC variants below are all instances of Framework~A: \cref{prop:HMCmod-properties} verifies \cref{cond:visited-property,cond:non-visited-cancel-out,cond:involution-decomposition}
for lazyHMC1, and \cref{sec:LazyHMC-Hilbert,sec:LazyHMC-with-original-integrator}
give the corresponding arguments for lazyHMC2 and lazyHMC3.
Readers mainly interested in the algorithmic construction may skip the measure-theoretic details and continue from \cref{sec:LazyHMC}.

\paragraph{Site measures.} Let $(X, \Sigma_{X}, \mu_{X})$ and $(Y, \Sigma_Y, \mu_Y)$ be two measure spaces  with $\mu_X$ and $\mu_Y$ probability measures. 
For simplicity we take $X$ and $Y$ to be measurable subsets of $\mathbb{R}$ with Borel $\sigma$-algebras $\Sigma_X = \{A \cap X \mid A \in \Sigma_{\mathbb{R}}\}$ and $\Sigma_Y= \{A \cap Y \mid A \in \Sigma_{\mathbb{R}}\}$, where $\Sigma_{\mathbb{R}}$ is the Borel $\sigma$-algebra of $\mathbb{R}$.
Assume $\mu_X$ and $\mu_Y$ are absolutely continuous with respect to the Lebesgue measure and let $\varphi_X$ and $\varphi_Y$ denote the corresponding derivatives. 
Let $(X^n, \Sigma_{X^n}, \mu_{X^n})$ be the product measure space with $\Sigma_{X^n} = \Sigma_X^{\otimes n} = \{A \cap X^n \mid A \in \Sigma_{\mathbb{R}^n}\}$ and measure $\mu_{X^n} = \mu_X^{\otimes n}$. We similarly define $(Y^n, \Sigma_{Y^n}, \mu_{Y^n})$.

\paragraph{Rose trees measures.}
Let $(\mathbb{X}, \Sigma_{\mathbb{X}})$ be the measurable space of the rose trees where each node contains a real number: $\mathbb{X} = X^{\mathcal{A}}$.
The $\sigma$-algebra $\Sigma_{\mathbb{X}}$ is the $\sigma$-algebra generated by cylindrical sets, i.e. sets of this form: $\{\mathbf{q} \in \mathbb{X} \mid (\mathbf{q}_{i_1}, \dots, \mathbf{q}_{i_n}) \in A\}$ with $i_1, \dots, i_n \in \mathcal{A}$ and $A \in \Sigma_{X^n}$.
The measure $\mu_{\mathbb{X}}$ on $\mathbb{X}$ is the countably-infinite product measure of $(X, \Sigma_X, \mu_X)$, given by the Kolmogorov extension theorem. 
This means that $\mu_{\mathbb{X}}(\{\mathbf{q} \in \mathbb{X} \mid (\mathbf{q}_{i_1}, \dots, \mathbf{q}_{i_n}) \in A \}) = \int_{X^n} [(\mathbf{q}_{i_1}, \dots, \mathbf{q}_{i_n}) \in A] \mu_X(\diff \mathbf{q}_{i_1}) \dots \mu_X(\diff \mathbf{q}_{i_n}) = \mu_{X^n} (A)$, where $\{\mathbf{q} \in \mathbb{X} \mid (\mathbf{q}_{i_1}, \dots, \mathbf{q}_{i_n}) \in A \}$ are the rose trees for which we restrict the values at nodes $i_1, \dots, i_n$ to be in the measurable set $A \in \Sigma_{X^n}$. 
$\mu_{\mathbb{X}}$ is a probability measure. 
Similarly for the momentum rose trees, we get the measure space $(\mathbb{Y}, \Sigma_{\mathbb{Y}}, \mu_{\mathbb{Y}})$ with $\mu_{\mathbb{Y}}$ as the countably-infinite product measure of $(Y, \Sigma_Y, \mu_Y)$. 

\paragraph{Target and auxiliary distributions.} Let $l: \mathbb{X} \rightarrow \mathbb{R}_{\geq 0}$ be the density of the unnormalized target distribution on $\mathbb{X}$ defined by the program.
We assume that $l$ is PACAP (see \cref{sec:AD}).
The target unnormalized measure is then $\mu_l(A) = \int [\mathbf{q} \in A] l(\mathbf{q}) \mu_{\mathbb{X}}(\diff \mathbf{q}),$ for any measurable $A \in \Sigma_{\mathbb{X}}$.
Let the normalized target measure be $\nu(\cdot) = \frac{\mu_l(\cdot)}{\mu_l(\mathbb{X})}.$
The auxiliary variable $\mathbf{p}$ has density $p_{\mathbf{q}}(\mathbf{p}) = 1$ w.r.t. $\mu_{\mathbb{Y}}$.

\paragraph{State space.} Let $\mathbb{X} \times \mathbb{Y}$ be the state space composed of position and momentum pairs, equipped with the $\sigma$-algebra $\Sigma_{\mathbb{X}} \otimes \Sigma_{\mathbb{Y}}$ and measure $\mu_{\mathbb{X}} \times \mu_{\mathbb{Y}}$.
We are now considering only the pairs for which $l(\mathbf{q})>0$, i.e. let $\mathbb{S} = \Supp(l) \times \mathbb{Y}$ with $\sigma$-algebra $\Sigma_{\mathbb{S}}=\{A \cap \mathbb{S} \mid A \in \Sigma_{\mathbb{X}} \otimes \Sigma_{\mathbb{Y}} \}$ and measure $\mu_{\mathbb{S}} = \mu_{\mathbb{X}} \times \mu_{\mathbb{Y}}$, where $\Supp(l) = \{\mathbf{q} \in \mathbb{X} \mid l(\mathbf{q})>0\}$.
Since $p_{\mathbf{q}}(\mathbf{p}) = 1$, the target density on the joint space simplifies to $w(\mathbf{q}, \mathbf{p}) = l(\mathbf{q})$.

Let $f: \mathbb{X} \rightarrow Z$ be the result function described by the given probabilistic program. 
The samples of $\mathbb{X}$ can be pushed-forward to $Z$.

\paragraph{Framework A} One iteration given position $\mathbf{q}, \mu_{\mathbb{Y}}, \mu_X, \mu_Y, l, I, v$ does:
\begin{enumerate}
  \item sample $\mathbf{p} \sim \mu_{\mathbb{Y}}$ lazily (since only finitely many components of the infinite-dimensional $\mathbf{p}$ are accessed, we sample each component on-demand)
  \item get proposed position and momentum $(\mathbf{q}', \mathbf{p}') = I(\mathbf{q}, \mathbf{p})$
  \item with probability $\min \{1, \alpha(\mathbf{q}, \mathbf{p})\}$ accept and return $\mathbf{q}'$, otherwise reject and return $\mathbf{q}$.
\end{enumerate} 

\paragraph{Involution conditions.}
For the acceptance ratio to exist and be computable, we assume the following sufficient properties. These conditions are abstract here, but \cref{prop:HMCmod-properties} shows that they hold for the lazy HMC construction below. 
Suppose there exist a measurable function $v: \mathbb{S} \rightarrow \mathcal{P}_{fin}(\mathcal{A})$ and measure preserving (w.r.t. the Lebesgue measure) involutions $I': \mathbb{R}^{2} \rightarrow \mathbb{R}^{2}$ and $I_T: \mathbb{R}^{2|T|} \rightharpoonup \mathbb{R}^{2|T|}$ for all $T \in \mathcal{P}_{fin}(\mathcal{A})$ which satisfy the following conditions: 
\begin{conds}
    \item \label{cond:visited-property} \textit{(Visited set is a cylinder.)} If $(\mathbf{q}, \mathbf{p}) \in v^{-1}(T)$ then $\{(\mathbf{x}, \mathbf{y}) \in \mathbb{S} \mid \mathbf{x}_T = \mathbf{q}_T, \mathbf{y}_T = \mathbf{p}_{T}\} \subset v^{-1}(T)$ (equivalently, $v^{-1}(T)$ is a cylinder set with base $C_T \in \Sigma_{\mathbb{R}^{2|T|}}$).
    \item \label{cond:involution-decomposition}\textit{(Involution decomposes along visited/unvisited split.)} For all $T \in \mathcal{P}_{fin}(\mathcal{A})$ and $(\mathbf{q}, \mathbf{p}) \in v^{-1}(T)$: $(\mathbf{q}'_T, \mathbf{p}'_T) = I_T(\mathbf{q}_T, \mathbf{p}_T)$ and $(\mathbf{q}'_i, \mathbf{p}'_i) = I'(\mathbf{q}_i, \mathbf{p}_i)$ for all $i \in \mathcal{A}\setminus T$,  where $(\mathbf{q}', \mathbf{p}') = I(\mathbf{q}, \mathbf{p})$. 
    \item \label{cond:non-visited-cancel-out} \textit{(Unvisited sites preserve density.)} $\varphi_{X}(q')\varphi_Y(p') = \varphi_{X}(q)\varphi_Y(p)$, where $(q', p') = I'(q, p).$
\end{conds}

\paragraph{Step regression example.}
In the step regression model of \cref{sec:stepReg}, the address set $\mathcal{A}$ contains addresses $a_j, b_j$ for $j = 1, 2, \ldots$, where $\mathbf{q}_{a_j}$ is the seed used to generate the $j$th segment height and $\mathbf{q}_{b_j}$ is the seed used to generate the $j$th exponential increment of the changepoint process.
For a given trace the visited set $T$ contains the addresses of those sites used by the finitely many segments that overlap the observed $x$-range (see \cref{fig:hmc-dim-change}).
\Cref{cond:visited-property} then says that any other rose tree agreeing with $\mathbf{q}$ on those visited sites visits the same set $T$: the values of the unvisited sites cannot change which segments overlap the data.
\Cref{cond:involution-decomposition} says that the proposed state is obtained by running a finite-dimensional involution on the visited sites, while each unvisited site $a_k$ or $b_k$ (for $k$ beyond the segments that cover the data) is transformed independently by $I'$.
Finally, \cref{cond:non-visited-cancel-out} says that this independent transformation $I'$ leaves the prior density on those unvisited sites unchanged, so their contributions cancel in the acceptance ratio.

\paragraph{Acceptance ratio.} We can reduce the Radon-Nikodym derivative needed in the acceptance ratio to a ratio of finite products (the proof can be found in \cref{prop:R-N-derivative-proof}):
\begin{proposition}\label{prop:R-N-derivative}
Let $(\mathbf{q}', \mathbf{p}') = I(\mathbf{q}, \mathbf{p})$, then $\mu_{\mathbb{S}}$ - almost everywhere we have:
\[\left( \frac{\diff (\mu_{\mathbb{S}} \circ I^{-1})}{\diff \mu_{\mathbb{S}}} \right)(\mathbf{q}, \mathbf{p}) = \frac{\prod_{i \in v(\mathbf{q}, \mathbf{p})} \varphi_X(\mathbf{q}'_i) \varphi_Y(\mathbf{p}'_i)}{\prod_{i \in v(\mathbf{q}, \mathbf{p})} \varphi_X(\mathbf{q}_i) \varphi_Y(\mathbf{p}_i)}.\]

\end{proposition}
\begin{theorem}
  \label{thm:alg1-acceptance-ratio}
Consider the following acceptance ratio in Framework A:
\[\alpha(\mathbf{q}, \mathbf{p}) = \frac{w(\mathbf{q}', \mathbf{p}')}{w(\mathbf{q}, \mathbf{p})} \cdot \frac{\prod_{i \in v(\mathbf{q}, \mathbf{p})} \varphi_X(\mathbf{q}'_i) \varphi_Y(\mathbf{p}'_i)}{\prod_{i \in v(\mathbf{q}, \mathbf{p})} \varphi_X(\mathbf{q}_i) \varphi_Y(\mathbf{p}_i)} = \frac{l(\mathbf{q}')}{l(\mathbf{q})} \cdot \frac{\prod_{i \in v(\mathbf{q}, \mathbf{p})} \varphi_X(\mathbf{q}'_i) \varphi_Y(\mathbf{p}'_i)}{\prod_{i \in v(\mathbf{q}, \mathbf{p})} \varphi_X(\mathbf{q}_i) \varphi_Y(\mathbf{p}_i)}.\]
Then the resulting kernel is stationary with respect to the target measure $\nu$.
\end{theorem}
The proof follows  by \cref{prop:R-N-derivative} and \cref{prop:iMCMC-stationary-kernel}.

\paragraph{Visited sites.} Since $l$ is PACAP, there exists a countable partition $\mathbb{X} = \uplus_{i=1}^{\infty} \mathbb{X}_i$ into analytic cylinders with given finite supports $T_i \subset \mathcal{A}$ with the following property. 
For each $i$, there exists an analytic function $l_i:V \rightarrow \mathbb{R}$, where $V \subseteq \mathbb{R}^{T_i}$ is open such that $\mathbb{X}_i \subseteq V \times \mathbb{R}^{\mathcal{A} \setminus T_i}$ and $l(\mathbf{q}) = l_{i}(\mathbf{q}_{T_i})$ when $\mathbf{q} \in \mathbb{X}_i$.
Let $v_l:\mathbb{X} \rightarrow \mathcal{P}_{fin}(\mathcal{A})$ be the function given by $v_l(\mathbf{q}) = T_i$ if $\mathbf{q} \in \mathbb{X}_i$. This is measurable because each analytic cylinder $\mathbb{X}_i$ is a cylinder set in $\mathbb{X}$, hence in $\Sigma_{\mathbb{X}}$.
We call $v_l$ the \emph{visited sites} function as, given a rose tree $\mathbf{q}$, it tells us it is enough to `visit' the sites with addresses $v_l(\mathbf{q})$ to compute $l(\mathbf{q})$, i.e. the values at the sites are needed for the computation. 
For step regression, $v_l(\mathbf{q})$ is the set of changepoint and height addresses used by the segments that overlap the data (see \cref{fig:hmc-dim-change}).

\paragraph{The gradient maps.}\looseness=-1
Moreover, we can define the measurable map $G: \mathbb{X} \rightarrow \mathbb{R}^{\mathcal{A}}$ such that for $\mathbf{q} \in \mathbb{X}_i \cap \Supp(l)$ we have $G(\mathbf{q}) = \mathbf{x}$ with $\mathbf{x}_{T_i} = -\nabla \log l_i(\mathbf{q}_{T_i})$ and $\mathbf{x}_j = 0$ for any $j \in \mathcal{A}\setminus T_i$.
If $\mathbf{q} \notin \Supp(l)$ we let $G(\mathbf{q})$ be the rose tree with all its nodes equal to $0$.
Let $T$ be a finite subset of $\mathcal{A}$ and consider the set $v_l^{-1}(\mathcal{P}(T))$ of positions for which the visited sites form a subset of $T$. 
Then $v_l^{-1}(\mathcal{P}(T))$ can be written as the union of the analytic cylinders $\mathbb{X}_{n_i}$ with supports $T_{n_i} \subseteq T$: $v_l^{-1}(\mathcal{P}(T))=\uplus_{i=1}^{\infty}\mathbb{X}_{n_i}$.
Hence, $v_l^{-1}(\mathcal{P}(T)) = \{\mathbf{q} \mid \mathbf{q}_T \in D_T\}$ for some measurable set $D_T \in \Sigma_{X^{|T|}}$.
Then there exists a measurable map $G_T: D_T \rightarrow \mathbb{R}^{\mathcal{A}}$ such that for any $\mathbf{q} \in v_l^{-1}(\mathcal{P}(T))$ we get $G(\mathbf{q}) = G_T(\mathbf{q}_T)$.

In the following sections we present different methods inspired by the HMC integrator.

\subsection{LazyHMC1: First Instantiation of Framework A}
\label{sec:LazyHMC}

Let $\mu_X$, $\mu_Y$ be the probability measures given by the standard normal distribution $\mathcal{N}(0, 1)$, so $\varphi_X$ and $\varphi_Y$ are the probability density function of $\mathcal{N}(0, 1)$.
Fix $\theta\in (0, \pi/2)$ and let $\epsilon_q = \sin (\theta)$ and $\epsilon_p = \tan (\theta/2)$.
Consider the following maps: 
  \begin{align*}
      \psi = \phi_{\epsilon_p}^P \circ \phi_{\epsilon_q}^Q \circ \phi_{\epsilon_p}^P, \qquad &F(\mathbf{q}, \mathbf{p})  = (\mathbf{q}, -\mathbf{p})\\
      \phi_{t}^P(\mathbf{q}, \mathbf{p})  = (\mathbf{q}, \mathbf{p} - t\mathbf{q} - t G(\mathbf{q})), \qquad &\phi_{t}^Q(\mathbf{q}, \mathbf{p})  = (\mathbf{q} + t\mathbf{p}, \mathbf{p}).
  \end{align*}
The added term $t \mathbf{q}$ is to account for the gradient of the log prior probability for each site, as each site is sampled from the normal distribution $\mathcal{N}(0, 1)$.
Let $\tau: \mathbb{R}^2 \rightarrow \mathbb{R}^2$ be the clockwise rotation by $\theta$: $\tau(q, p) = (\cos (\theta)q+\sin(\theta)p, -\sin(\theta)q +\cos (\theta)p)$.
The constants $\epsilon_q=\sin\theta$ and $\epsilon_p=\tan(\theta/2)$ are chosen so that, on any site with $G(\mathbf{q})_a=0$, the leapfrog step $\psi$ acts on $(q_a,p_a)$ as $\tau$ (\cref{proof:rotation-constants}).

\paragraph{Visited sites of trajectories.}
Let $v: \mathbb{S} \rightarrow \mathcal{P}_{fin}(\mathcal{A})$ be given by $v(\mathbf{q}, \mathbf{p}) = \cup_{i=0}^L v_l(\mathbf{q}^{(i)})$, where $(\mathbf{q}^{(i)}, \mathbf{p}^{(i)}) = \psi^i(\mathbf{q}, \mathbf{p})$ (for $i = 0, \dots, L$) are the states in the trajectory after $L$ leapfrog steps starting from $(\mathbf{q}, \mathbf{p})$.
The function $v$ tells us which sites need to be visited in both $\mathbf{q}$ and $\mathbf{p}$ in order for the weight $l(\mathbf{q}^{(L)})$ of the final position to be computed.
Notice that $v(\mathbf{q}, \mathbf{p}) = v(F \circ \psi^L(\mathbf{q}, \mathbf{p}))$ since
$L$ leapfrog steps from $F \circ \psi^L(\mathbf{q}, \mathbf{p})$ give the same positions $\mathbf{p}^{(i)}$ in the trajectory. 
For each position $\mathbf{q}^{(i)}$ we only need to visit sites in $v_l(\mathbf{q}^{(i)})$ to compute $l(\mathbf{q}^{(i)})$ and hence $G(\mathbf{q}^{(i)})$.

\paragraph{Constructing the involution on $\mathbb{S}$.}
It is easy to check that the map $F \circ \psi^L$ is an involution on $\mathbb{X} \times \mathbb{Y}$.
However, even if $(\mathbf{q}, \mathbf{p}) \in \mathbb{S}$, applying $ F \circ \psi^L$ on it might give us a state $(\mathbf{q}^{(L)}, \mathbf{p}^{(L)})$ in $\mathbb{X} \times \mathbb{Y} \setminus \mathbb{S}$, i.e. $l(\mathbf{q}^{(L)}) = 0$.
Since we want to have an involution on $\mathbb{S}$, i.e. on the states which have non-zero probability, we can define $I: \mathbb{S} \rightarrow \mathbb{S}$ the following way: 
\begin{equation*}
    I(\mathbf{q}, \mathbf{p}) := 
    \begin{cases}
        (F \circ \psi^L)(\mathbf{q}, \mathbf{p}) & \text{ if }  (F \circ \psi^L)(\mathbf{q}, \mathbf{p}) \in \mathbb{S}\\
        K(\mathbf{q}, \mathbf{p}) &  \text{ otherwise},
    \end{cases}
\end{equation*}
where $K(\mathbf{q}, \mathbf{p}) = (\mathbf{q}', \mathbf{p}')$ with $(\mathbf{q}', \mathbf{p}')_{v(\mathbf{q}, \mathbf{p})} = (\mathbf{q}, \mathbf{p})_{v(\mathbf{q}, \mathbf{p})}$ and $(\mathbf{q}', \mathbf{p}')_{\mathcal{A} \setminus v(\mathbf{q}, \mathbf{p})} = (F \circ \psi^L)(\mathbf{q}, \mathbf{p})_{\mathcal{A} \setminus v(\mathbf{q}, \mathbf{p})}$.
From the definition of $v$ and $v_l$ it follows that if $l(\mathbf{q})>0$ then also $l(\mathbf{q}')>0$, so $K(\mathbf{q}, \mathbf{p})$ is in $\mathbb{S}$.
Since $v(K(\mathbf{q}, \mathbf{p})) = v(\mathbf{q}, \mathbf{p})$ we have that $K$ is an involution on the measurable set $\mathbb{S} \setminus (F \circ \psi^L)^{-1}(\mathbb{X} \times \mathbb{Y}\setminus \mathbb{S})$.

\begin{proposition}
  \label{prop:HMCmod-properties}
  $I$ satisfies the \cref{cond:visited-property,cond:non-visited-cancel-out,cond:involution-decomposition}.
\end{proposition}
\begin{proof}
  We give the intuition for each condition, the full argument can be found in \cref{proof:HMC1}.

\emph{\Cref{cond:non-visited-cancel-out}.}
On any coordinate $j$ outside $v(\mathbf{q},\mathbf{p})$ the gradient $G(\mathbf{q}^{(i)})_j$ is $0$ for every $i$, so the leapfrog step $\psi$ acts on the $j$-th coordinate as the rotation $\tau$. 
After $L$ steps and the momentum flip, the induced map on each unvisited site is $I' = F'\circ\tau^L$, which preserves both Lebesgue measure and the normal product density $\varphi_X \times \varphi_Y$. Hence the unvisited-site factors cancel in the acceptance ratio.

\emph{\Cref{cond:visited-property}.}
Let $T$ be a finite subset of $\mathcal{A}$ and let $(\mathbf{q}, \mathbf{p}) \in \mathbb{S}$ with $v(\mathbf{q}, \mathbf{p}) = T$. 
By induction on $i$, the trajectory coordinates $(\mathbf{q}^{(i)}_T,\mathbf{p}^{(i)}_T)$ depend only on $(\mathbf{q}_T,\mathbf{p}_T)$, since the gradient at each step is supported in $T = v(\mathbf{q},\mathbf{p})$. 
So $v^{-1}(T)$ is a cylinder set determined by its $T$-coordinates.

\emph{\Cref{cond:involution-decomposition}.}
The previous two observations combine to give a clean decomposition: on the visited coordinates $I$ restricts to the finite-dimensional measure-preserving involution $I_T = F_T\circ\psi_T^L$ on $C_T$, and on each unvisited coordinate $I$ acts independently as $I'$.
\end{proof}

Let $u(\mathbf{q}, \mathbf{p}) = \{j \in \mathcal{A} \mid \exists i \in \{0, \dots, L\}, G(\mathbf{q}^{(i)})_j\neq 0\}$ be the function that gives the sites $j$ for which at least one position $\mathbf{q}'$ in the trajectory has the property that $G(\mathbf{q}')_j\neq 0$.
Notice that $u(\mathbf{q}, \mathbf{p}) \subseteq v(\mathbf{q}, \mathbf{p})$ and $u(\mathbf{q}, \mathbf{p}) = u(I(\mathbf{q}, \mathbf{p}))$. 
For example, we might have $j \in v(\mathbf{q}, \mathbf{p}) \setminus u(\mathbf{q}, \mathbf{p})$ if site $j$ corresponds to a variable which has constant $l$ for different branches, so $j$ would be visited. 

\paragraph{Acceptance ratio.}
By \cref{thm:alg1-acceptance-ratio} using the following acceptance ratio 
\[\alpha(\mathbf{q}, \mathbf{p}) = \frac{l(\mathbf{q}')}{l(\mathbf{q})} \cdot \frac{\prod_{j \in v(\mathbf{q}, \mathbf{p})} \varphi_X(\mathbf{q}'_j) \varphi_Y(\mathbf{p}'_j)}{\prod_{j \in v(\mathbf{q}, \mathbf{p})} \varphi_X(\mathbf{q}_j) \varphi_Y(\mathbf{p}_j)},\]
where $(\mathbf{q}', \mathbf{p}') = I(\mathbf{q}, \mathbf{p})$
results in a kernel that is stationary with respect to the target distribution. 

If $j \in  v(\mathbf{q}, \mathbf{p}) \setminus u(\mathbf{q}, \mathbf{p})$,  then $G(\mathbf{q}^{(i)})_j = 0$ for all $i\in \{0, \dots, L\}$ so $(\mathbf{q}'_j, \mathbf{p}'_j) = I'(\mathbf{q}_j, \mathbf{p}_j)$.
This means that $(\varphi_{X}\times \varphi_{Y})(I'(\mathbf{q}_j, \mathbf{p}_j)) = (\varphi_{X}\times \varphi_{Y})(\mathbf{q}_j, \mathbf{p}_j)$.
Hence, the acceptance ratio can be simplified (it is enough to know $u$ instead of $v$):
$\alpha(\mathbf{q}, \mathbf{p}) = \frac{l(\mathbf{q}')}{l(\mathbf{q})} \cdot \frac{\prod_{j \in u(\mathbf{q}, \mathbf{p})} \varphi_X(\mathbf{q}'_j) \varphi_Y(\mathbf{p}'_j)}{\prod_{j \in u(\mathbf{q}, \mathbf{p})} \varphi_X(\mathbf{q}_j) \varphi_Y(\mathbf{p}_j)}.$

\paragraph{Implementation details.}
As discussed in \cref{sec:Haskell-Nagata} we are able to get the sites on which the gradient of $\log l$ is not zero and the values of the gradient for those sites. 
Therefore, for any position $\mathbf{q}$ we can get the set $u_l(\mathbf{q})$ which contains the sites for which the gradient of $\log l(\mathbf{q})$ is not zero.
So we can construct $u(\mathbf{q}, \mathbf{p}) = \cup_{i \in \{0, \dots, L\}} u_l(\mathbf{q}^{(i)})$.
Therefore, there is no need to know $v$ for this method. 

\paragraph{Note on $K$.} If we had proposed a state that is not in $\mathbb{S}$, the acceptance ratio would have been $0$ so the proposed state would be rejected and the state $(\mathbf{q}, \mathbf{p})$ would be returned instead. 
This would be the same as proposing $(\mathbf{q}, \mathbf{p})$ from the start. 
Hence, intuitively, we would like $K$ to be the identity, which is what we do in the implementation.
This also makes $I$ an involution that satisfies the required properties, but we chose the different $K$ to make the presentation easier to follow.

\paragraph{Note on $I$.} An alternative way of defining $I$ is to let it be the identity for states $(\mathbf{q}, \mathbf{p})$ for which any of $F \circ \psi^i(\mathbf{q}, \mathbf{p})$ is not in $\mathbb{S}$ for $i = 1, \dots, L$.
This would correspond to rejecting if along the trajectory we encounter a position that is not in the support of $l$.

\paragraph{Note on the choice of $I'$.} One can notice that we could have chosen any $I'$ that is a $\text{Leb}_2$ preserving involution with $(\varphi_{X}\times \varphi_{Y})(q, p) = (\varphi_{X}\times \varphi_{Y})(I'(q,p))$ (see \cref{sec:LazyHMC-with-original-integrator}) to use on the sites which are not visited. 
However, $I'$ might then disagree with $F \circ \psi^L$ on the sites not visited: we may have $(F \circ \psi^L)(\mathbf{q}, \mathbf{p})_j \neq I'(\mathbf{q}_j, \mathbf{p}_j)$ for some $j \in \mathcal{A} \setminus v(\mathbf{q}, \mathbf{p})$.
This would mean that we need to know $v(\mathbf{q}, \mathbf{p})$ (not only $u(\mathbf{q}, \mathbf{p})$) to apply $I'$ on these sites. 
This is the case with the method in \cref{sec:LazyHMC-with-original-integrator}.

\subsection{LazyHMC2: Framework A with a Different Hamiltonian Splitting}
\label{sec:LazyHMC-Hilbert}

We are now considering a different $\psi$, inspired by the HMC integrator resulting from a different splitting of the Hamiltonian ~\cite{beskos2011HMConHS}. 
Fix $\theta>0$. Then $\psi = \phi_{\theta/2}^P \circ \phi_{\theta}^{Q, P} \circ \phi_{\theta/2}^P$ with:
\begin{align*}
    \phi_{t}^P(\mathbf{q}, \mathbf{p}) & = (\mathbf{q}, \mathbf{p} - t G(\mathbf{q})) \\
    \phi_{t}^{Q, P}(\mathbf{q}, \mathbf{p}) & = (\cos(t)\mathbf{q} + \sin (t)\mathbf{p}, - \sin(t) \mathbf{q}+\cos (t) \mathbf{p})
\end{align*}

The rest of the setup is the same as in \cref{sec:LazyHMC}. 
The key difference from lazyHMC1 is that the position–momentum
update step $\phi^{Q,P}$ is a \emph{joint rotation}
rather than a pure position drift.
Notice that if $j \notin v(\mathbf{q}, \mathbf{p})$ then $\psi(\mathbf{q}, \mathbf{p})_j = \tau(\mathbf{q}_j, \mathbf{p}_j)$, the same rotation $\tau$ as in \cref{sec:LazyHMC}.

\begin{proposition}
  \label{prop:HMC2-properties}
  $I$ (defined using the new $\psi$) satisfies \cref{cond:visited-property,cond:non-visited-cancel-out,cond:involution-decomposition}.
\end{proposition}
\begin{proof}
\looseness=-1
Both $\phi_{t}^P, \phi_{t}^{Q, P}$ act site-wise: each site's update uses only that site's own $(\mathbf{q}_j, \mathbf{p}_j)$, connected to other sites solely through the gradient $G$, which is supported on $T = v(\mathbf{q}, \mathbf{p})$.
This site-wise structure is all that \cref{cond:visited-property} needs, so the cylinder argument of \cref{prop:HMCmod-properties} applies the same here.
On an unvisited site $j$ we have $G(\mathbf{q}^{(i)})_j = 0$, so the momentum kick $\phi^P$ is the identity there and $\psi$ reduces to its middle factor $\phi_\theta^{Q,P}$, which is directly the rotation $\tau$. 
Since $\tau$ is a rotation it preserves $\varphi_X \times \varphi_Y$, giving \cref{cond:non-visited-cancel-out} for $I' = F' \circ \tau^L$.
The map $\psi_T$ is still volume preserving, so $I_T = F_T \circ \psi_T^L$ is a measure-preserving involution and $I$ decomposes as required for \cref{cond:involution-decomposition}.
The details are in \cref{proof:HMC2conditions}.
\end{proof}

\subsection{LazyHMC3: Framework A with the Usual Leapfrog Integrator}
\label{sec:LazyHMC-with-original-integrator}
The original HMC integrator (\S\ref{sec:HMC}) can also be used for the visited sites. 
But since the kernel outputs the full rose tree, the unvisited coordinates must be changed carefully: leaving them untouched or copying them arbitrarily would break the Markov chain.
We must also know which sites are visited to include them in the acceptance ratio, unlike the methods in \cref{sec:LazyHMC,sec:LazyHMC-Hilbert}, where the addresses of non-zero-gradient sites along the trajectory suffice.

Fix $\epsilon>0$ and let one step be $\psi = \phi_{\epsilon/2}^P \circ \phi_{\epsilon}^Q \circ \phi_{\epsilon/2}^P$, where $\phi_t^P, \phi_t^Q$ are defined as in \cref{sec:LazyHMC}.
Again, $X, Y, \mu_X, \mu_Y, v$ are defined the same as in \cref{sec:LazyHMC}.
Notice that if $j\in v(\mathbf{q}, \mathbf{p})$ and $G(\mathbf{q}^{(i)})_j = 0$ for all $i \in \{0, \dots, L\}$ we cannot guarantee that $(\mu_X \times \mu_Y)(I(\mathbf{q}, \mathbf{p})) = (\mu_X \times \mu_Y)(\mathbf{q}, \mathbf{p})$.
Therefore, we need to know $v(\mathbf{q}, \mathbf{p})$ not just $u(\mathbf{q}, \mathbf{p})$.
Moreover, we need to let $I(\mathbf{q}, \mathbf{p})_i = I'(\mathbf{q}_i, \mathbf{p}_i)$ for any $i \in \mathcal{A} \setminus v(\mathbf{q}, \mathbf{p})$, where $I': \mathbb{R}^2 \rightarrow \mathbb{R}^2$ is some $\text{Leb}_2$ preserving involution with $(\mu_X \times \mu_Y)(I'(\mathbf{q}, \mathbf{p})) = (\mu_X \times \mu_Y)(\mathbf{q}, \mathbf{p})$ (e.g. $I' = F' \circ \tau$, or the identity,  which is what we use in our implementation).
Because $I'$ preserves the site density $\varphi_X\times\varphi_Y$, every unvisited coordinate contributes a factor $1$ to the acceptance ratio, so the otherwise infinite product collapses to the finite product over the visited sites $v(\mathbf{q},\mathbf{p})$.

\begin{proposition}
  \label{prop:HMC3-properties}
  $I$ satisfies \cref{cond:visited-property,cond:non-visited-cancel-out,cond:involution-decomposition} (see \cref{proof:HMC3conditions} for the proof).
\end{proposition}

\paragraph{Implementation details.} To compute $v(\mathbf{q}, \mathbf{p})$ we need to first apply $F \circ \psi^L$ to get the states in the trajectory $(\mathbf{q}^{(i)}, \mathbf{p}^{(i)})$. 
Then for each $\mathbf{q}^{(i)}$ we can determine $v_l(\mathbf{q}^{(i)})$ by checking which thunks in the seed have been evaluated, as done for the single-site proposal kernel in \cite{DashKPS23LazyPPL}.
For the sites $j$ not in $\mathbf{v}(\mathbf{q}, \mathbf{p})$ we then need to change $\mathbf{q}^{(L)}_j, \mathbf{p}^{(L)}_j$ to $I'(\mathbf{q}_j, \mathbf{p}_j)$ before the accept/reject step.

\subsection{LazyNUTS}
HMC's performance depends on the number of leapfrog steps $L$, which is hard to tune. 
NUTS addresses this by choosing $L$ adaptively via a doubling procedure with a no-U-turn stopping rule. 
We develop a lazy version, lazyNUTS \citep{NUTS}, that keeps the doubling procedure and no-U-turn rule but still evaluates only finitely many sites per iteration, using the leapfrog integrator from lazyHMC1 (\cref{sec:LazyHMC}).
We give a correctness argument in \cref{sec:LazyNUTS}: the acceptance ratio is $1$ provided the stopping rule yields a start-independent set of proposed states ($\mathcal{C}_{\mathbf{s}}=\mathcal{C}_{\mathbf{s}'}$); as with the maximum tree depth in standard NUTS, our cap $M$ on the number of proposed states is a practical termination we do not show to preserve this. 
We also derive Framework~B (\cref{sec:Framework-B}), generalizing all lazy HMC and lazy NUTS versions, in the appendix.

\section{Example: Piecewise-Constant Regression with Poisson-Process Changepoints}
\label{sec:stepReg}
To make the preceding frameworks concrete, we now trace the behaviour of lazy HMC on a single model in detail: a piecewise-constant regression model with Poisson-process changepoints from \cite{DashKPS23LazyPPL}.
For simplicity, we refer to it as the step regression model.

Consider a set of synthetic two-dimensional datapoints as shown in \cref{fig:hmc-dim-change}.
The task is to fit a piecewise-constant function to this dataset. 
The model first draws a function \lstinline|f| from a prior, then observes the likelihood of the dataset being generated by \lstinline|f|, after introducing some Gaussian noise.
\begin{lstlisting}[caption = Step regression model, label = {lst:stepReg}]
stepReg :: [(Double, Double)] -> Meas (Double -> Double)
stepReg dataset = do
  f <- sample (splice (poissonPP 0 0.2) randConst)
  forM_ dataset (\(x, y) -> scoreLog ( normalLogPdf (f x) 0.1 y))
  return f
\end{lstlisting}
Now consider the prior. 
The \lstinline|splice| function builds a piecewise function from a point process generating changepoints and a prior \lstinline|randomFun| over the behaviour on each segment.
It samples an infinite stream of changepoints \lstinline|xs| and, for each, an independent function from \lstinline|randomFun|. 
An input is then evaluated by selecting the segment containing it.

\begin{lstlisting} 
splice :: Prob [Double] -> Prob (Double -> Double) -> Prob (Double -> Double)
splice pointProcess randomFun = do
  xs <- pointProcess
  fs <- mapM (const randomFun) xs
  default_f <- randomFun
  let h [] x = default_f x
      h ((a,f):rest) x
        | x <= a    = f x
        | otherwise = h rest x
  return (h (zip xs fs))
\end{lstlisting}
Since each segment is constant, we use a base prior sampling constant functions:
\begin{lstlisting}
randConst :: Prob (Double -> Double)
randConst = do
  b <- normal 0 3
  return (const b)
\end{lstlisting}
The point process is a Poisson process with rate $0.2$ and starting from $0$:
\begin{lstlisting}
poissonPP :: Double -> Double -> Prob [Double]
poissonPP lower rate = do
  step <- exponential rate
  let x = lower + step
  rest <- poissonPP x rate
  return (x : rest)
\end{lstlisting}

Note that a draw \lstinline|f| from \lstinline|splice| contains infinitely many segments.
Consequently, the function returned by \lstinline|stepReg| (\cref{lst:stepReg}) also has infinitely many segments.

This does not cause any practical issues, since evaluation is lazy.
Only those segments that are required for computing the likelihood or for producing the plots are ever evaluated.
For instance, in \cref{fig:step-posterior}, only segments corresponding to input values in the range $[-0.5,6.5]$ are actually used.

\begin{figure}[t]
  \centering
  \includegraphics[width=0.9\linewidth, trim=15 8 8 15,clip]{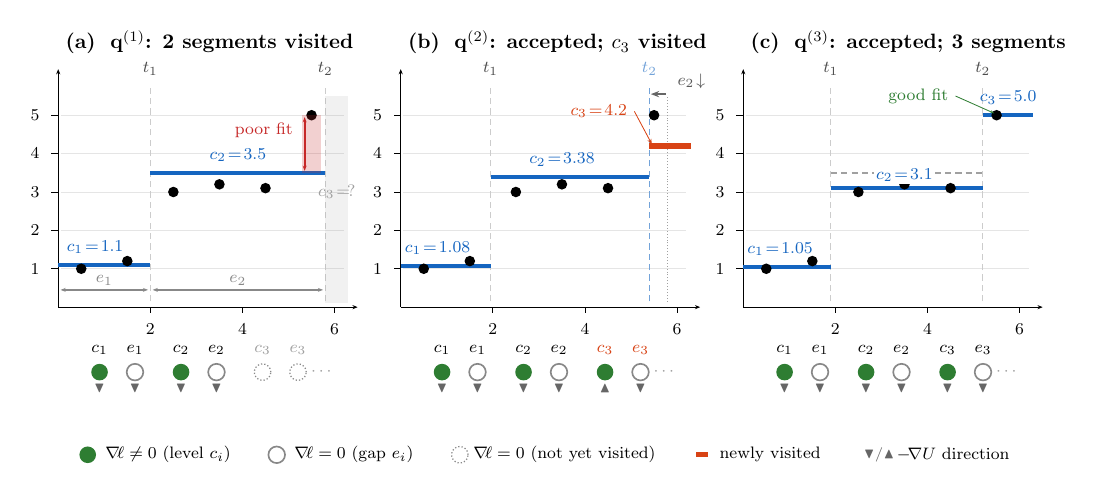}
  \caption{Dimension change during lazy HMC on the step regression model (\S\ref{sec:stepReg}).
  (a) Initial state $\mathbf{q}^{(1)}$, where the second changepoint lies to the right of all datapoints. Only two segments are evaluated; later segments are not visited due to lazy evaluation.
  (b) After five leapfrog steps (shown as a trajectory in $(e_2,c_2)$-space in \cref{fig:step-hmc-trajectory}), the second changepoint $t_2$ moves left of the last datapoint, activating a third segment. The previously unused sites $c_3$ and $e_3$ become visited, increasing the dimensionality of the evaluated state.
  (c) After a further five leapfrog steps, the segments adjust using the newly available gradient information, resulting in a better fit. Both transitions shown are accepted. See \cref{fig:step-posterior} for additional posterior samples.
  }
  \Description{Illustration of lazy HMC dimension change}
  \label{fig:hmc-dim-change}
\end{figure}

\begin{figure}[htbp]
    \centering
    \begin{subfigure}{0.45\textwidth}
        \centering
        \includegraphics[width=0.8\linewidth, trim=15pt 8pt 8pt 15pt,clip]{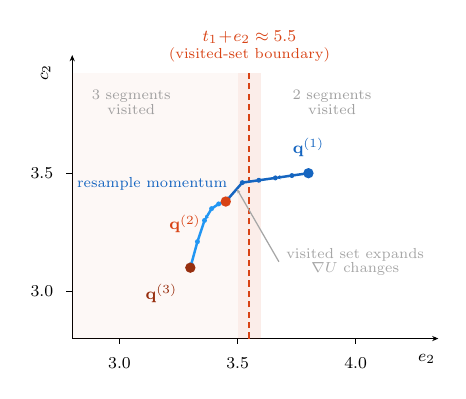}
        \caption{Illustrative example of lazy HMC trajectory on the step regression model.  The values in the $(e_2, c_2)$-space correspond to the positions $\mathbf{q}^{(1)}, \mathbf{q}^{(2)}, \mathbf{q}^{(3)}$ from \cref{fig:hmc-dim-change}, with $5$ leapfrog steps between successive positions.}
        \Description{Illustrative example of lazy HMC trajectory on the step regression model.}
        \label{fig:step-hmc-trajectory}
    \end{subfigure}
    \hfill
    \begin{subfigure}{0.45\textwidth}
        \includegraphics[width=\linewidth, trim=8pt 8pt 8pt 8pt,clip]{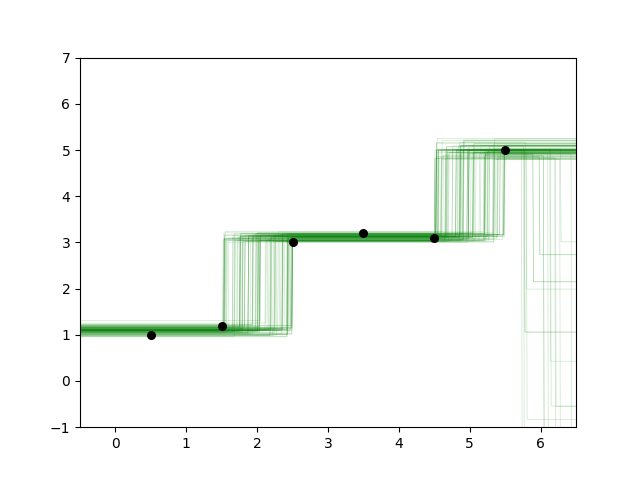}
        \caption{100 posterior samples from running lazyHMC1$(\epsilon=0.01,\ L=30)$ on the step regression model thinned every $10$ samples after $10^3$ burn-in samples.}
        \Description{lazyHMC1 posterior samples on step reg model.}
        \label{fig:step-posterior}
    \end{subfigure}
    \caption{Lazy HMC on the step regression model: (a) an illustrative trajectory, and (b) actual posterior samples.}
    \label{fig:step}
\end{figure}

\paragraph{Lazy HMC on the step regression model.} Consider the rose tree $\mathbf{q}$ corresponding to the infinite stream of seeds used in a run of the step regression model. 
We distinguish between the seeds $\mathbf{q}_{a_j}$ used to generate the heights of the constant functions and the seeds $\mathbf{q}_{b_j}$ used to generate the exponential increments in the Poisson process for $j = 1, \dots $.
Therefore the height of the $j$th constant function segment is $c_j(\mathbf{q}) = 3 \mathbf{q}_{a_j}$ and the $j$th exponential draw is $e_j(\mathbf{q}) = - \log \text{cdf}_{\mathcal{N}(0, 1)}(\mathbf{q}_{b_j})/0.2$ where $\text{cdf}_{\mathcal{N}(0, 1)}$ is the cumulative distribution function of $\mathcal{N}(0, 1)$.
Hence, the $j$th changepoint is given by $t_j(\mathbf{q}) = \sum_{i=1}^j e_i(\mathbf{q})$ and we can consider $t_0(\mathbf{q}) = 0$.

Suppose we start from a rose tree $\mathbf{q}^{(1)}$ for which the second changepoint $t_2$ is greater than the $x$-value of all datapoints, as shown in \cref{fig:hmc-dim-change}(a). 
In this case, the likelihood $l$ is given by:
\[l(\mathbf{q}^{(1)}) = \prod_{i=1}^2\varphi(y_i \mid c_1(\mathbf{q}^{(1)}), 0.1)  \prod_{i=3}^6\varphi(y_i \mid c_2(\mathbf{q}^{(1)}), 0.1) = \prod_{i=1}^2\varphi(y_i \mid 1.1, 0.1) \prod_{i=3}^6\varphi(y_i \mid 3.5, 0.1).\] where $(x_1, y_1) \dots (x_6, y_6)$ are the datapoints in increasing $x$-order and $\varphi(\cdot \mid \mu, s)$ denotes the density of $\mathcal{N}(\mu, s^2)$.
There exists an open set $A_1$ containing $\mathbf{q}^{(1)}$ such that \[l(\mathbf{q}) = \prod_{i=1}^2\varphi(y_i \mid c_1(\mathbf{q}) , 0.1) \prod_{i=3}^6\varphi(y_i \mid c_2(\mathbf{q}), 0.1)\]
for all $\mathbf{q} \in A_1$.
Thus, two segments are enough to cover the dataset (with respect to the $x$-axis), and the sites $a_i, b_i$ with $i \geq 3$ are not visited. 
Consequently, the gradient at those sites is zero.
Moreover, $l$ restricted to the set $A_1$ depends only on the sites $a_1$ and $a_2$. 
Hence, the gradient at the visited sites $b_1$ and $b_2$ is also zero.
The second segment in \cref{fig:hmc-dim-change}(a) is not a good fit for the $6$th datapoint.

\cref{fig:step-hmc-trajectory} shows the trajectory of the pair consisting of the second exponential increment and the second constant height after five leapfrog steps, starting from position $\mathbf{q}^{(1)}$ with momentum $\mathbf{p}^{(1)}$.

After the fourth leapfrog step, the second changepoint becomes smaller than $x_6$.
As a result, a third constant function is required in order to evaluate the likelihood contribution of the last datapoint $(x_6,y_6)$.
After the fifth leapfrog step is performed, the new position $\mathbf{q}^{(2)}$ (corresponding to \cref{fig:hmc-dim-change}(b)) is accepted.
The sites that were visited in the intermediary positions computed during the $5$ leapfrog steps are the ones corresponding to the first $3$ constant function segments, so $v(\mathbf{q}^{(1)}, \mathbf{p}^{(1)}) = \{a_1, a_2, a_3, b_1, b_2, b_3\}$. 
However, if we use lazyHMC1 from \cref{sec:LazyHMC} or lazyHMC2 from \cref{sec:LazyHMC-Hilbert} it is enough to only include the sites $\{a_1, a_2, a_3\}$ in the acceptance ratio as they are the only ones for which the gradient of $l$ is non-zero. 

The heights of the first two constant functions have decreased due to the gradient information at sites $a_1$ and $a_2$, resulting in a better fit.

After the changepoint moves left of $x_6$, only datapoints $3, 4, 5$ remain in segment $2$. Therefore $l$ is: \[l(\mathbf{q}) = \left(\prod_{i=1}^2  \varphi(y_i \mid c_1(\mathbf{q}), 0.1)\right) \left(\prod_{i=3}^5\varphi(y_i \mid c_2(\mathbf{q}), 0.1)\right)  \varphi(y_6 \mid c_3(\mathbf{q}), 0.1)\] for $\mathbf{q}$ in some open set containing $\mathbf{q}^{(2)}$.
In this region, the sites $a_3$ and $b_3$ become visited, and $a_3$ now has a non-zero gradient.
After resampling the momentum and performing five additional leapfrog steps, we reach position $\mathbf{q}^{(3)}$ corresponding to \cref{fig:hmc-dim-change}(c). 
The segments now provide a better fit, as gradient information from $a_1$, $a_2$, and $a_3$ has been incorporated by the leapfrog integrator. 
\cref{fig:step-posterior} shows 100 posterior samples obtained by lazyHMC1 on the step regression model. 

\paragraph{$l$ is PACAP}
Let $T = \{(n_1, \dots, n_6) \in \mathbb{Z}^6_{>0} \mid n_1 \leq n_2 \leq \dots \leq n_6\}$ be the set of weakly increasing 6-tuples of positive integers. 
Consider the following countable partition into analytic cylinders of the domain $\mathbb{X} = \mathbb{R}^{\mathcal{A}}$ of $l$: $\mathbb{X} = \cup_{t\in T} U_t$ where 
\[U_{(n_1, n_2, \dots, n_6)} = \{\mathbf{q} \in \mathbb{X}\mid t_{n_{i}-1} (\mathbf{q})< x_i \leq t_{n_{i}}(\mathbf{q}) \text{ for } 1 \leq i \leq 6 \}\]
for any $(n_1, n_2, \dots, n_6) \in T$.
The set $U_{(n_1, n_2, \dots, n_6)}$ corresponds to all the rose trees for which the $i$th datapoint $(x_i, y_i)$ is going to be on the $n_i$ segment for all $1 \leq i \leq 6$. For each $\mathbf{q}$, each datapoint lies in exactly one segment.
Therefore the union $\mathbb{X} = \cup_{t\in T} U_t$ is disjoint. 
Note that $t_{n_i}(\mathbf{q})$ can be computed using only the sites $b_1, b_2, \dots b_{n_i}$. We let the support of $U_{(n_1, n_2, \dots, n_6)}$ be $B = \{b_1, b_2, \dots, b_{n_6}\} \cup \{a_1, a_2, \dots, a_{n_6}\}$. 
To see why $U_{(n_1, n_2, \dots, n_6)}$ is an analytic cylinder (see \cref{def:PACAP}), consider the analytic functions $h_{n_i-1}: \mathbb{R}^B \rightarrow \mathbb{R}, g_{n_i}: \mathbb{R}^B \rightarrow \mathbb{R}$ given by $h_{n_i-1}(\mathbf{q}_B) = t_{n_i-1}(\mathbf{q})-x_i$ and $g_{n_i}(\mathbf{q}_B) = x_i-t_{n_i}(\mathbf{q})$ for $1 \leq i \leq 6$.
Now let the set $V = \cap_{i=1}^6 h_{n_i-1}^{-1}((-\infty, 0))$ be the open set encoding the strict inequalities: $V = \{\mathbf{q}_B \in \mathbb{R}^B \mid h_{n_i-1}(\mathbf{q}_B) < 0 \text{ for } 1 \leq i \leq 6\}$. 
Hence, we can write $U_{(n_1, n_2, \dots, n_6)}$ as: 
\[U_{(n_1, n_2, \dots, n_6)} = \{\mathbf{q} \in \mathbb{X}\mid \mathbf{q}_B \in V \ \& \ g_{n_{i}} (\mathbf{q}_B)\leq 0 \text{ for } 1 \leq i \leq 6 \}.\]
Moreover, on $U_{(n_1, n_2, \dots, n_6)}$ we have: 
$l(\mathbf{q}) = \prod_{i=1}^6 \varphi(y_i \mid c_{n_i}(\mathbf{q}), 0.1).$
Considering that $c_j$ only needs information from the site $a_j$ there exists an analytic function $f:V \rightarrow \mathbb{R}$ given by $f(\mathbf{q}_B) = \prod_{i=1}^6 \varphi(y_i \mid c_{n_i}(\mathbf{q}_B), 0.1)$.
Hence, $l(\mathbf{q}) = f(\mathbf{q}_B)$ when $\mathbf{q} \in U_{(n_1, n_2, \dots, n_6)}$.

\section{Experiments}
\label{sec:experiments}
We implemented the lazy HMC methods (\S\ref{sec:LazyHMC}, \S\ref{sec:LazyHMC-Hilbert} and \S\ref{sec:LazyHMC-with-original-integrator}) and lazyNUTS methods (\S\ref{sec:LazyNUTS}) as an extension to the LazyPPL library, by also using the automatic differentiation method (\S\ref{sec:Haskell-Nagata}).
Our implementation is available at \url{https://github.com/lazyppl-team/lazyhmc}.

We include the following illustrative experiments that emphasize the non-parametric side of lazy HMC, testing the aspects that standard HMC methods do not typically address. To be clear, we are not evaluating against established benchmarks, since there aren't really any in the non-parametric setting. Although the experiments here are simple, they nonetheless provide empirical evidence that lazy HMC performs as intended for non-parametric and infinite-dimensional models.

In what follows, for the specified step size $\epsilon$, we have used a fresh value sampled uniformly at each iteration from $[\epsilon/2, 3\cdot\epsilon/2]$.
The maximum number of samples generated in a single iteration of lazyNUTS is denoted by $M$.
Imposing a bound $M$ on the number of proposed states may seem to go against the adaptive design of NUTS, but in practice NUTS implementations (e.g. Stan) always cap the trajectory via a maximum tree depth. 
We set $M$ mainly for comparability: large enough not to artificially cap the trajectory lengths lazyNUTS explores, while keeping them of a similar order of magnitude to those in the lazy HMC variants. 
For comparison, we also ran the lazy lightweight Metropolis-Hastings (lazyLMH) method from \cite{DashKPS23LazyPPL};
this simple method proposes a new rose tree by resampling the value at each site with probability $b$, and does not use gradients at all. 

We have also compared lazy HMC performance with NP-HMC (the variant presented in \cite{NPHMCMakZO21} which does not include truncation or discontinuous HMC). 
Posteriors were comparable to our methods. Lazy HMC variants were generally faster, although we observed slower performance on the GMM model, where the bottleneck appears to be the current lack of optimization in LazyPPL's automatic differentiation implementation.
These timings are not perfectly comparable: NP-HMC is implemented in Python, while our methods are implemented in Haskell and LazyPPL's use of laziness means that models with infinite structure cannot always be expressed identically in both systems. 
Compare e.g. \lstinline|geometric| (natural for lazy HMC) and \lstinline|geometricStrict| (suited to NP-HMC).

\paragraph{Experiment: Geometric distribution.} 
We first sample from the geometric distribution described in \cref{lst:geometric} with $p = 0.2$.
This example shows the methods can handle dimension changes. 
The methods appear to sample correctly from the distribution, as shown in \cref{fig:geometricProb}, which compares the probability mass estimated from the samples with the ground truth distribution.
The total variation distance (TVD) provides a scalar measure of the discrepancy from the ground truth, with values closer to zero indicating better agreement. 
The TVDs are $0.0146$ (lazyLMH), $0.034$ (NP-HMC), $0.0198$ (lazyNUTS), $0.0115$ (lazyHMC1), $0.0117$ (lazyHMC2), and $0.0132$ (lazyHMC3).
The total times were $16.6$ (NP-HMC), $0.6$ (lazyHMC1), $0.6$ (lazyHMC2), $0.7$ (lazyHMC3) and $0.8$ (lazyNUTS) minutes; NP-HMC is at least ten times slower than the lazy variants.

\begin{figure}[htbp]
    \centering
    \begin{subfigure}{0.45\textwidth}
        \centering
        \includegraphics[width=0.8\linewidth, trim=8pt 8pt 8pt 12pt,clip]{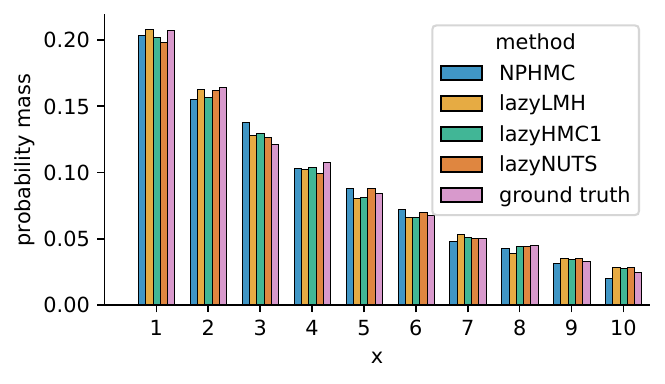}
        \caption{Estimated posterior mass for the Geometric distribution, averaged over $10$ runs. Sample sizes per run: $1,300$ for \texttt{lazyHMC1}, \texttt{NP-HMC} $(\epsilon=0.1,\ L=15)$, and \texttt{lazyNUTS} $(\epsilon = 0.1, M=2^5)$; $10^4$ for \texttt{lazyLMH} $(b=0.5)$.}
        \Description{A bar chart of posterior mass by method.}
        \label{fig:geometricProb}
    \end{subfigure}
    \hfill
    \begin{subfigure}{0.45\textwidth}
        \includegraphics[width=\linewidth, trim=8pt 8pt 8pt 8pt,clip]{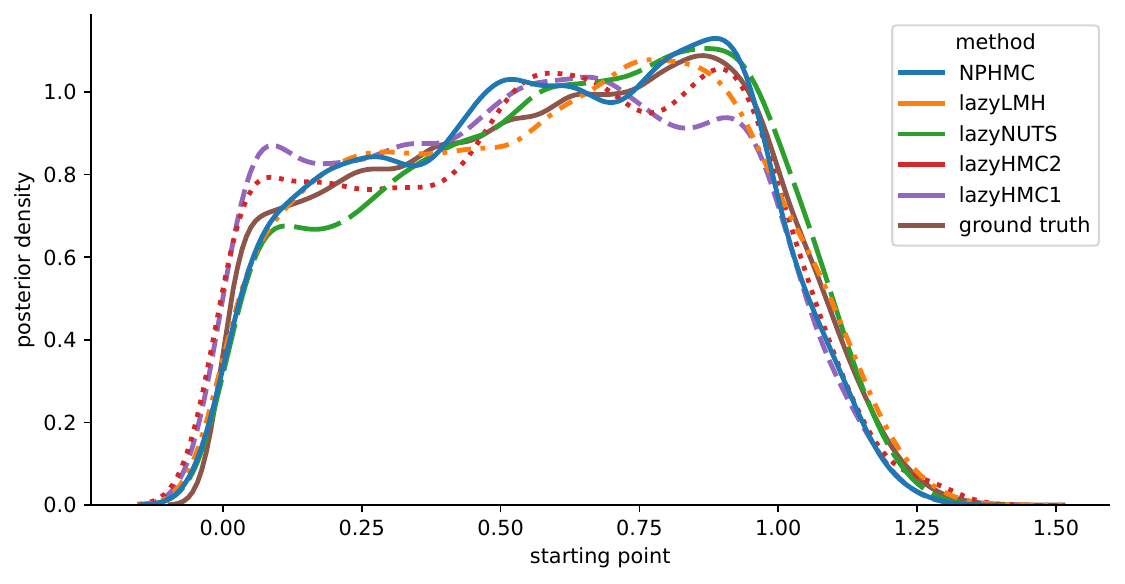}
        \caption{Kernel density estimate for the random walk model, averaged over $10$ runs. Sample sizes per run: $1,300$ for \texttt{lazyHMC1}, \texttt{lazyHMC2}, \texttt{NP-HMC} $(\epsilon=0.1,\ L=5)$, and \texttt{lazyNUTS} $(\epsilon = 0.1, M=2^5)$; $2\times10^4$ for \texttt{lazyLMH} $(b=0.5)$.}
        \Description{Kernel density estimate for the Walk model.}
        \label{fig:walk}
    \end{subfigure}
    \caption{Sampler comparison on the Geometric (left) and random walk (right) models.}
    \label{fig:main}
\end{figure}

\paragraph{Experiment: Random Walk.}
We consider the random walk model from \cite{NPHMCMakZO21}.
The model is presented in \cref{lst:walk}. 
The posterior cannot be computed exactly, so we will use $10^6$ importance samples instead of the ground truth.
For the lazy HMC, lazyNUTS, and NP-HMC methods, we used a heuristic initialisation procedure to find a starting state closer to the typical set. Specifically, we first ran 100 iterations with the momentum excluded from the acceptance ratio. These iterations were used only for initialisation and were discarded. We then ran the correct kernel for 1300 iterations and discarded the first $20\%$ samples as burn-in. We then ran the correct kernel for the number of iterations specified in \cref{fig:walk}, discarding the first $20\%$ of samples as burn-in for each method as well as a thinning for lazyLMH. The resulting kernel density estimates are shown in \cref{fig:walk}.

The upper bound on the number of samples considered per iteration is much larger for lazyNUTS than for the other methods, which makes it slower: lazyNUTS takes $1.6$ minutes (for all chains), compared with $0.2$ minutes for lazyHMC1 and lazyHMC2 and $0.3$ minutes for lazyHMC3. 
However, lazyNUTS also achieves a much larger effective sample size (ESS), with an ESS of $732.2$ compared with $189.6$ for lazyHMC1, $175.6$ for lazyHMC2 and $231.9$ for lazyHMC3. 
When computational cost is taken into account, the lazy HMC variants are more efficient, achieving $20.09$, $18.61$ and $15.09$ ESS/s respectively, compared with $7.73$ ESS/s for lazyNUTS. In comparison, NP-HMC is slower than the lazy HMC variants, at $19$ total minutes, and has a similar ESS of $213.2$, leading to a much lower efficiency of $0.18$ ESS/s.

We emphasise that this model demonstrates some of the core ideas of compositional modelling that are enabled by this lazy approach. 
Recall that line~3 of the model (\cref{lst:walk}) produces random infinite sequences: the walk continues forever.
This infinite stochastic process determines the infinite-dimensional state space that lazy HMC operates over.
Our only observation (line~5) is about the first point the stopping criterion is met.
Lazy HMC produces results, despite operating over an infinite dimensional space,
because it focuses lazily on the dimensions that are actually needed.

\paragraph{Experiment: Gaussian mixture model (GMM) with unbounded number of components.}
We consider the Gaussian Mixture model inspired by \cite{ZhouDCC} in which the number of components is unbounded:
\begin{align*}
  K \sim \text{Poisson}(9) + 1, \qquad &\mu_k \sim \text{Uniform}\left(\textstyle\frac{20(k-1)}{K}, \frac{20k}{K} \right), \qquad y_n \sim \textstyle \frac{1}{K} \sum_{k=1}^K \mathcal{N}(\mu_k, 1.5^2).
\end{align*}
for $k = 1, \dots, K$.
\begin{figure}[t]
  \centering
  \includegraphics[width=0.7\linewidth, trim=8 8 8 12,clip]{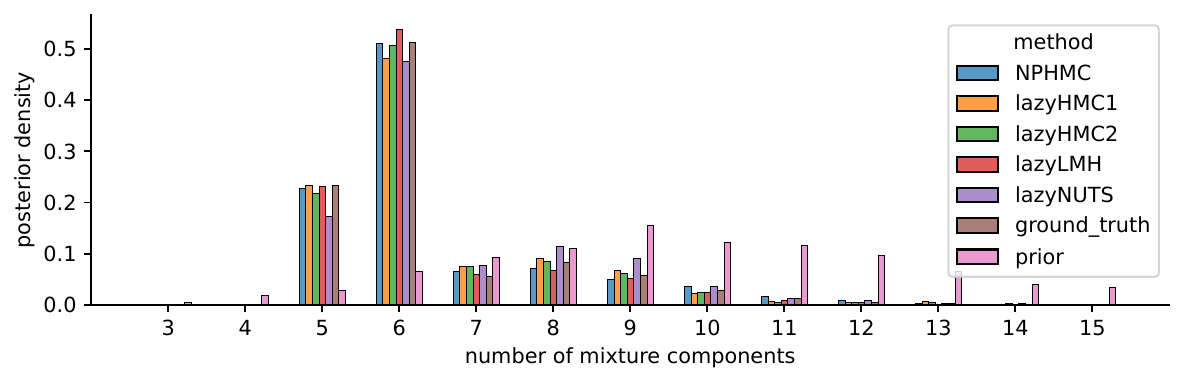}
  \caption{Histogram of the number of components for the GMM, averaged over $10$ runs. Sample sizes per run: $2000$ for lazyHMC1, lazyHMC2, NP-HMC $(\epsilon=0.05,\ L=20)$, and lazyNUTS $(\epsilon = 0.1, M=2^5)$; $4\times 10^4$ for lazyLMH $(b=0.5)$. We discard the first $20\%$ of the samples for each chain. }
  \Description{Histogram for GMM.}
  \label{fig:gmm}
\end{figure}
The model is presented in \cref{sec:clustering}.

We fix the true number of components to $K^* = 5$ and sample the true components means as $\mu_k^* \sim \text{Uniform}\left(\textstyle\frac{20(k-1)}{5}, \frac{20k}{5} \right)$ for $k = 1 \dots 5$.
Each point in our synthetic dataset is drawn from the Gaussian mixture with $K^* = 5$ components with means $\mu_1^*, \dots \mu_5^*$. 
Concretely, we first sample one component index $k$ uniformly from $\{1, \dots, 5\}$, and then draw the datapoint from $\mathcal{N}(\mu_k^*, 1.5^2)$.
We split the dataset into $200$ points for training and $50$ points for testing.

The prior ($\text{Poisson}(9)+1$) on the number of components and the distribution inferred from the posterior samples can be seen in \cref{fig:gmm}.
As the posterior cannot be computed exactly, we again use $10^6$ importance samples instead of the ground truth. 
We also approximate the log pointwise predictive density (LPPD, \cite{vehtari2014bayesian}) for the test dataset $\{y_1, \dots y_n\}$ by $\sum_{i=1}^n \log{\frac{1}{N}}\sum_{j=1}^N p(y_i \mid \theta_j)$, where $\theta_j$ are the samples from the target posterior.
The true LPPD is $-143.73$ and the mean and one standard deviation LPPD over the $10$ chains for each method is: $-147.74 \pm 0.28$ (lazyHMC1), $-147.73 \pm 0.26$ (lazyHMC2), $-147.84 \pm 0.10$ (lazyNUTS), $-147.62 \pm 0.14$ (lazyLMH), $-147.75 \pm 0.09$ (NP-HMC).
The total times in minutes across the chains are: $72$ (lazyHMC1, lazyHMC2), $129$ (lazyNUTS), $41$ (lazyLMH), $27$ (NP-HMC).

\begin{figure}[htbp]
    \centering
    \begin{subfigure}{0.45\textwidth}
        \centering
        \includegraphics[width=1\linewidth, trim=8 8 8 12,clip]{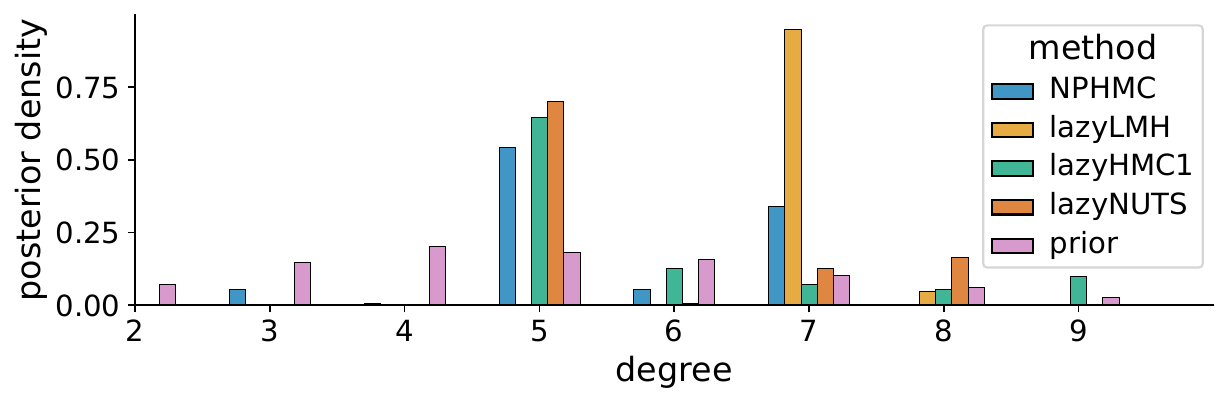}
        \caption{Histogram of the degree in the polynomial model, averaged over $10$ runs. Sample sizes per run: $3000$ for lazyHMC1 $(\epsilon=0.005,\ L=64)$, lazyNUTS $(\epsilon = 0.005, M=2^6)$; $10^6$ for lazyLMH $(b=0.5)$. For lazyLMH we thin the chain to obtain $\sim3000$ samples. In all cases, only the second half of the samples is used (the first half is discarded as burn-in).}
        \Description{Histogram for Poly.}
        \label{fig:poly-degree}
    \end{subfigure}
    \hfill
    \begin{subfigure}{0.45\textwidth}
        \includegraphics[width=\linewidth, trim=8pt 8pt 8pt 8pt,clip]{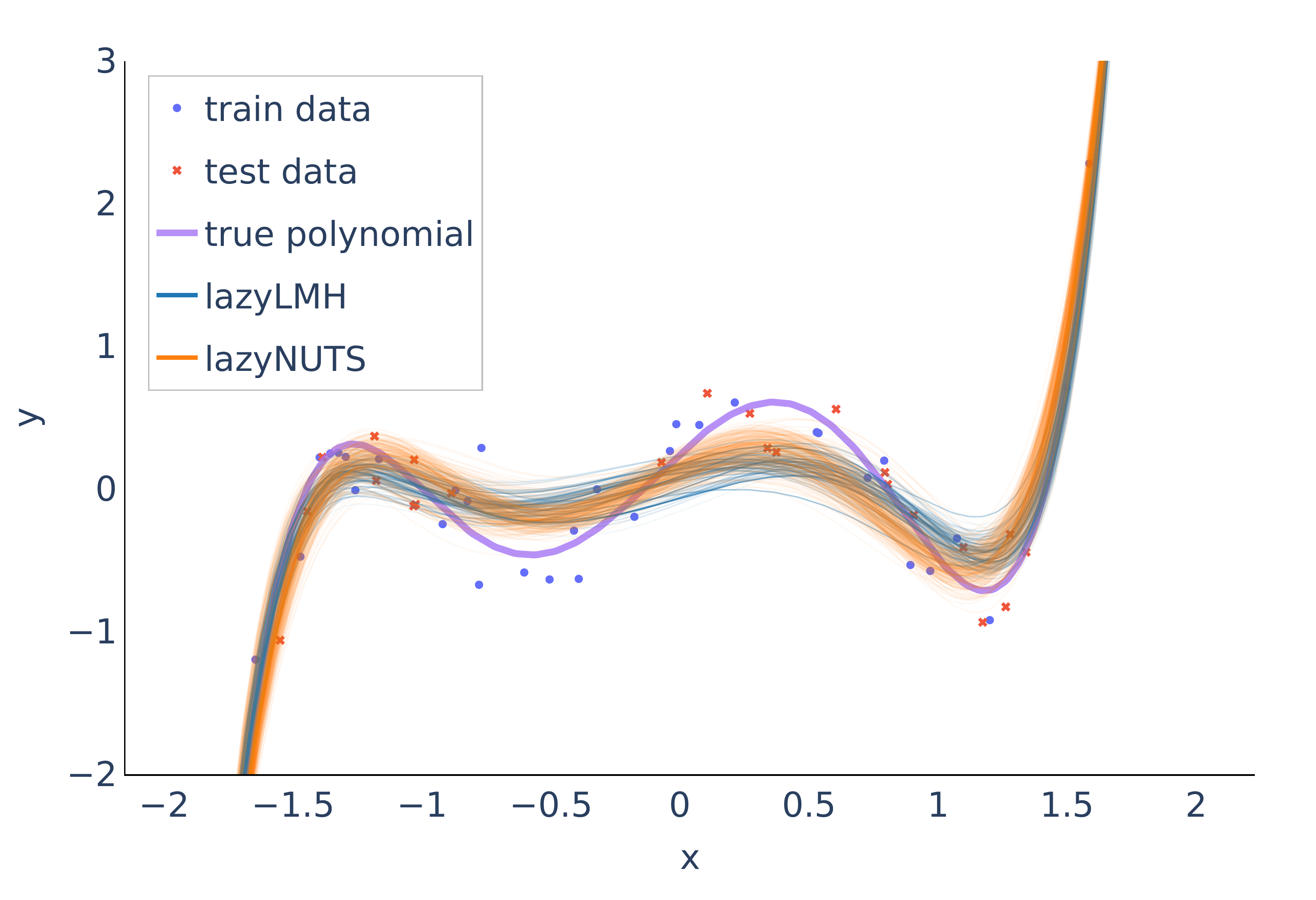}
        \caption{ Posterior polynomial fits based on the same post-burn-in samples as in \cref{fig:poly-degree}. For this figure, we further thin all chains by keeping every second sample.}
        \Description{Poly posterior.}
        \label{fig:poly-posterior}
    \end{subfigure}
    \caption{Histogram of sampled polynomial degrees and posterior fits for the polynomial regression model.}
    \label{fig:poly}
\end{figure}

\paragraph{Experiment: Polynomial regression model with unbounded degree.}
In this example, we generate a synthetic regression dataset from a degree $5$ polynomial.
We sample $40$ training inputs and 
$25$ test inputs independently as $x_i \sim \text{Uniform}(-2, 2)$.
The outputs are generated from a fixed ground-truth polynomial $P^*$ of degree $5$, corrupted by Gaussian noise: $y_i \sim \mathcal{N}(P^*(x_i), 0.25^2)$. 
The polynomial $P^*$ as well as the training and test points can be seen in \cref{fig:poly-posterior}.
The polynomial prior for our model is: 
\[D \sim \text{Poisson}(4) + 1, \qquad (a_k)_{k=0}^D \stackrel{\text{i.i.d.}}{\sim} \mathcal{N}(0, 0.5^2), \qquad Q(x) = \sum_{k=0}^D a_k x^k.\]
Since we do not want to constrain the degree $D$ of the polynomial we sample it from the Poisson distribution. 
The coefficients of the polynomial $Q$ are then sampled from the normal distribution. 
For each train point $(x_i, y_i)$, we score the likelihood of it being generated using the polynomial $Q$ together with the Gaussian noise:
\begin{align*}
  y_i & \sim \mathcal{N}(Q(x_i), 0.25^2).
\end{align*}

\begin{table}[t]
\centering
\caption{For the polynomial regression experiment, we report test dataset LPPD and runtime in minutes summed over all chains. LPPD is computed using the same post-burn-in samples as in \cref{fig:poly-degree}. The true test dataset LPPD is $3.82$. Reported runtimes correspond to the full chain execution, including burn-in.}
\label{tab:poly-LPPD}
\begin{tabular}{lcccc}
\toprule
& lazyNUTS
& lazyHMC1
& lazyLMH
& NP-HMC \\
\midrule
LPPD mean $\pm$ std
& $-5.94 \pm 2.53$
& $-4.69 \pm 2.16$
& $-13.11 \pm 9.72$
& $-7.35 \pm 6.16$ \\
Time (minutes)
& $10.9$
& $23.2$
& $17.2$
& $233.7$\\
\bottomrule
\end{tabular}
\end{table}

In \cref{fig:poly-degree} we show the distribution of the degree of the inferred polynomial. 
While the samples from lazyHMC1 and lazyNUTS concentrate on polynomials of degree $5$, 
those from lazyLMH are stuck mostly with polynomials of degree $7$. 
The performance difference is also visible in \cref{fig:poly-posterior}, where the lazyNUTS samples fit the dataset better than the lazyLMH samples. 
For visual clarity, \cref{fig:poly-posterior} shows only lazyNUTS and lazyLMH, since plotting all methods together causes substantial overlap.
NP-HMC also has a smaller posterior mass on degree $5$ polynomials. 
Despite similar or longer runtime, lazyLMH and NP-HMC achieve worse LPPD than our methods (\cref{tab:poly-LPPD}).

\medskip
\paragraph{Conclusion.}
We have given a Hamiltonian Monte Carlo simulation that works over probabilistic programs involving infinite dimensions, by evaluating at the dimensions lazily. The key idea is that the relevant gradient is only non-zero at finitely many dimensions (because the likelihood functions are PACAP, \S\ref{sec:AD}) and the acceptance ratio is designed so that the factors corresponding to unused dimensions cancel (\S\ref{sec:LazyHMC-big-section}). To improve on the issues with leapfrog and step size, we also provide a lazyNUTS procedure. 

Our initial experimentation shows that the performance is good even where the changing dimensions play a big role in the model (\S\ref{sec:experiments}). However, in the broader context of real-world systems and large non-parametric models, to be clear, we don't expect a purely declarative approach to be sufficient: the modeller will likely need to guide or interact with the inference engine (e.g.~Gen~\cite{DBLP:conf/pldi/Cusumano-Towner19}).
In our setting this interaction is mostly the familiar one of choosing HMC hyperparameters (step size, trajectory length $L$, or the bound $M$ in lazyNUTS) and tuning them to the model at hand; and a natural direction for future work is to incorporate automatic step-size adaptation (such as in \cite{NUTS}), reducing the need for manual tuning.

\begin{acks}
We would like to thank Alex Lew, Matthijs V\'{a}k\'{a}r and Fabian Zaiser for valuable discussions and advice.
This work was partially supported by the National Research Foundation, Singapore, under its RSS Scheme (NRF-RSS2022-009), ERC Grant BLAST, AFOSR under award number FA9550-21-1-0038, and grants from ARIA Safeguarded AI.
\end{acks}

\bibliography{biblio}

\newpage
\appendix

\section{Proof of the Fundamental Lemma~\eqref{eqn:logrel-fundlemma}, from \cref{thm:ad}}\label{appendix:fundamental-lemma}

The proof is by induction on typing derivations, and requires preliminary lemmas on support enlargement (\cref{lem:support}), closure under finite intersection (\cref{lem:intersection}), restriction (\cref{lem:restriction}), gluing (\cref{lem:gluing}), and composition (\cref{lem:composition}).

We note that $\mathbb X = \RR^{\mathcal A}$ is itself an analytic cylinder (with empty support), so c-analytic; and $\varnothing$ is c-analytic (the empty union), with $\rel{\varnothing}{\tau}$ a singleton for every~$\tau$.

\begin{lemma}[Support enlargement]\label{lem:support}
  If $U$ is an analytic cylinder with support~$B$ and $B' \supseteq B$ is finite, then $U$ is also an analytic cylinder with support~$B'$.
\end{lemma}
\begin{proof}
  Let $V \subseteq \RR^B$ and $g_k : V \to \RR$ be the open set and constraints for~$U$.
  Let $\pi : \RR^{B'} \to \RR^B$ be the coordinate projection, which is continuous and linear (hence analytic).
  Take $V' = \pi^{-1}(V)$ (open since $\pi$ is continuous) and constraints $g_k \circ \pi$.
\end{proof}

\begin{lemma}[Closure under finite intersection]\label{lem:intersection}
  The intersection of two analytic cylinders is an analytic cylinder.
  Consequently, c-analytic sets are closed under finite intersection.
\end{lemma}
\begin{proof}
  Given analytic cylinders $U_1, U_2$ with supports $B_1, B_2$, enlarge both to the common support $B = B_1 \cup B_2$ (\cref{lem:support}), then intersect the open sets and concatenate the constraint lists (cf.~\cite{DBLP:conf/lics/HuotLMS23}, Cor.~B.10 for the finite-dimensional case).
  For c-analytic closure: if $W = \biguplus_i A_i$ and $W' = \biguplus_j C_j$
  are c-analytic, then $W \cap W' = \biguplus_{i,j} (A_i \cap C_j)$,
  which is disjoint since the $A_i$ and $C_j$ are separately disjoint,
  and is a countable union of analytic cylinders.
\end{proof}

\begin{lemma}[Restriction]\label{lem:restriction}
  If $f \in \rel U \tau$ and $U' \subseteq U$ is c-analytic, then $f|_{U'} \in \rel{U'}\tau$.
\end{lemma}
\begin{proof}
  By induction on~$\tau$.
  For $\tau = \tyreal$: if $f$ is PACAP on $U$ with partition $\{A_k\}_k$ into analytic cylinders, and $U' = \biguplus_m C_m$ is c-analytic, then $U' = \biguplus_{k,m} (A_k \cap C_m)$.
  Each $A_k \cap C_m$ is an analytic cylinder (\cref{lem:intersection}) with support $B_k \cup D_m$.
  On this piece, $f(x) = f_k(x|_{B_k})$, where $f_k$ is analytic on some open $V_k \subseteq \RR^{B_k}$.
  By support enlargement (\cref{lem:support}) from $B_k$ to $B_k \cup D_m$, the function $f_k$ lifts to $f_k \circ \pi$ on the enlarged open set, witnessing PACAP for each piece. Hence $f|_{U'}$ is PACAP.
  The unit and product cases are immediate.
  For $\tau = \tysum{i}{I}\tau_i$: if $f \in \rel U {\tysum{i}{I}\tau_i}$ via partition $U = \biguplus_i U_i$, then $U' = \biguplus_i (U_i \cap U')$ where each $U_i \cap U'$ is c-analytic (\cref{lem:intersection}), and $f_i|_{U_i \cap U'} \in \rel{U_i \cap U'}{\tau_i}$ by the induction hypothesis.
  For $\tau = \tyarrow{\tau_1}{\tau_2}$: any c-analytic $U'' \subseteq U'$ is also a c-analytic subset of~$U$, so the defining clause for~$\rel U {\tyarrow{\tau_1}{\tau_2}}$ directly gives the result.
\end{proof}

\begin{lemma}[Gluing]\label{lem:gluing}
  If $U = \biguplus_{i \in I} U_i$ is a countable partition into c-analytic sets and $f_i \in \rel{U_i}\tau$ for each~$i$, then the function $f$ defined by $f(u) = f_i(u)$ for $u \in U_i$ satisfies $f \in \rel U \tau$.
\end{lemma}
\begin{proof}
  By induction on $\tau$.
  For $\tau = \tyreal$: each $f_i$ is PACAP on $U_i$; concatenating their analytic-cylinder partitions gives a PACAP partition of $U$.
  The unit and product cases are immediate.
  For $\tau = \tysum{l}{I}\tau_l$: each $f_k \in \rel{U_k}{\tysum{l}{I}\tau_l}$ gives a partition $U_k = \biguplus_l U_{kl}$ with $g_{kl} \in \rel{U_{kl}}{\tau_l}$.
  Let $W_l = \biguplus_k U_{kl}$ for each $l \in I$; each $W_l$ is c-analytic (a countable disjoint union of c-analytic sets).
  Since the $U_k$ are pairwise disjoint, so are the $U_{kl}$ across different~$k$, giving $\biguplus_l W_l = \biguplus_k U_k = U$.
  By the induction hypothesis at $\tau_l$ (gluing the $g_{kl}$ over $W_l = \biguplus_k U_{kl}$), the resulting function on $W_l$ is in $\rel{W_l}{\tau_l}$.
  For $\tau = \tyarrow{\tau_1}{\tau_2}$: given c-analytic $U' \subseteq U$ and $g \in \rel{U'}{\tau_1}$, let $U'_i = U' \cap U_i$, which is c-analytic by \cref{lem:intersection}.
  By restriction (\cref{lem:restriction}), $g|_{U'_i} \in \rel{U'_i}{\tau_1}$.
  Since $U'_i \subseteq U_i$ is c-analytic and $f_i \in \rel{U_i}{\tyarrow{\tau_1}{\tau_2}}$,
  we get $\lambda u.\,f_i(u)(g(u)) \in \rel{U'_i}{\tau_2}$.
  By the induction hypothesis at $\tau_2$ (gluing over $U' = \biguplus_i U'_i$), $\lambda u.\,f(u)(g(u)) \in \rel{U'}{\tau_2}$.
\end{proof}

\begin{lemma}[Composition]\label{lem:composition}
  If $c : \RR^n \to \RR$ is PAP and $g_1,\dots,g_n$ are PACAP on a c-analytic set~$U$, then $c \circ (g_1,\dots,g_n)$ is PACAP on~$U$.
\end{lemma}
\begin{proof}
  Since $c$ is PAP, there is a finite partition $\RR^n = \biguplus_j P_j$ where each $P_j$ is an analytic set with open set~$V_j$ and constraints $h_{j,i} : V_j \to \RR$, together with analytic $c_j : V_j \to \RR$ agreeing with $c$ on~$P_j$.
  Let $\{A_k\}_k$ be a common refinement of the PACAP partitions of $g_1,\dots,g_n$, obtained by iterated intersection (\cref{lem:intersection}); by support enlargement (\cref{lem:support}), on each~$A_k$ the tuple $(g_1,\dots,g_n)$ is analytic on a common open set $W_k \subseteq \RR^{B_k}$.
  Refine further:
  $A_k \cap (g_1,\dots,g_n)^{-1}(P_j)$ is an analytic cylinder with support~$B_k$,
  open set $W_k \cap (g_1,\dots,g_n)|_{B_k}^{-1}(V_j)$ (open since $(g_1,\dots,g_n)|_{B_k}$ is continuous),
  and constraints consisting of those of~$A_k$ together with $h_{j,i} \circ (g_1,\dots,g_n)|_{B_k}$ (analytic, as a composition of analytic functions on open sets).
  On each such piece, $c_j \circ (g_1,\dots,g_n)|_{B_k}$ is analytic on the same open set.
\end{proof}

\noindent\textbf{Fundamental lemma}~\eqref{eqn:logrel-fundlemma}\textbf{, from \cref{thm:ad}.}\enspace
\textit{If\/ $\judg{x_1:\tau_1,\dots,x_n:\tau_n}{e}{\tau}$ and given c-analytic $U$ and $f_1\in \rel U {\tau_1}$, \dots, $f_n\in \rel U {\tau_n}$, then $\lambda u.\,\sem{e}(f_1(u),\dots,f_n(u))\in \rel U \tau$.}

\begin{proof}
  By induction on the typing derivation of $\judg{x_1:\tau_1,\dots,x_n:\tau_n}{e}{\tau}$.
  Throughout, write $\gamma(u) = (f_1(u), \dots, f_n(u))$ for the substituted valuation.

  The real literal ($\lambda u.\,r$ is constant, hence PACAP), variable ($\lambda u.\,\sem{x_j}(\gamma(u)) = f_j$), and unit cases are immediate.

  \paragraph{Lambda.}
  If $e = \tmabs{x}{\tau_1}{e'}$ with $\judg{\Gamma, x:\tau_1}{e'}{\tau_2}$.
  Take any c-analytic $U' \subseteq U$ and $g \in \rel{U'}{\tau_1}$.
  By restriction, each $f_j|_{U'} \in \rel{U'}{\tau_j}$.
  By the induction hypothesis for $e'$ at $U'$:
  $\lambda u.\,\sem{e'}(f_1(u), \dots, f_n(u), g(u)) \in \rel{U'}{\tau_2}$.
  Since $\sem{\tmabs{x}{\tau_1}{e'}}(\gamma(u))(g(u)) = \sem{e'}(\gamma(u), g(u))$, this shows $\lambda u.\,\sem{e}(\gamma(u)) \in \rel U {\tyarrow{\tau_1}{\tau_2}}$.

  \paragraph{Application.}
  If $e = \tmapp{e_1}{e_2}$ with $\judg{\Gamma}{e_1}{\tyarrow{\tau_1}{\tau_2}}$ and $\judg{\Gamma}{e_2}{\tau_1}$.
  By the induction hypothesis, $h_1 := \lambda u.\,\sem{e_1}(\gamma(u)) \in \rel U {\tyarrow{\tau_1}{\tau_2}}$
  and $h_2 := \lambda u.\,\sem{e_2}(\gamma(u)) \in \rel U {\tau_1}$.
  Since $U$ is itself c-analytic, instantiating with $U' = U$ and $g = h_2$:
  $\lambda u.\,h_1(u)(h_2(u)) \in \rel U {\tau_2}$.

  The pair and projection cases follow directly from the product-type clause.

  \paragraph{Injection.}
  If $e = \tminj{i}{e'}$: by the induction hypothesis $g := \lambda u.\,\sem{e'}(\gamma(u)) \in \rel U {\tau_i}$.
  The partition with $U_i = U$ and $U_j = \varnothing$ for $j \neq i$ witnesses $\lambda u.\,\sem{e}(\gamma(u)) \in \rel U {\tysum{i}{I}\tau_i}$, since $\varnothing$ is c-analytic and $\rel{\varnothing}{\tau_j}$ is inhabited.

  \paragraph{Case.}
  If $e = \tmcase{e'}{x}{e_i}{I}$ with $\judg{\Gamma}{e'}{\tysum{i}{I}\tau_i}$ and $\judg{\Gamma, x:\tau_i}{e_i}{\tau}$ for each $i \in I$.
  By the induction hypothesis on $e'$:
  $\lambda u.\,\sem{e'}(\gamma(u)) \in \rel U {\tysum{i}{I}\tau_i}$,
  giving $U = \biguplus_{i \in I} U_i$ with each $U_i$ c-analytic and $h_i \in \rel{U_i}{\tau_i}$.
  For each $i$: by restriction, $f_j|_{U_i} \in \rel{U_i}{\tau_j}$ for all~$j$, and $h_i \in \rel{U_i}{\tau_i}$.
  By the induction hypothesis for $e_i$ at $U_i$ with the extended valuation:
  $\lambda u.\,\sem{e_i}(f_1(u), \dots, f_n(u), h_i(u)) \in \rel{U_i}\tau$.
  Since $\sem{e}(\gamma(u)) = \sem{e_i}(\gamma(u), h_i(u))$ for $u \in U_i$,
  the result follows from gluing (\cref{lem:gluing}).

  \paragraph{Constants.}
  A constant $c : \tyarrow{\typroduct{\tyreal}{\cdots}}{\tyreal}$ denotes a total PAP function $\RR^n\to\RR$ (Example~\ref{ex:PAP}).
  Take c-analytic $U' \subseteq U$ and $(g_1,\dots,g_n) \in \rel{U'}{\typroduct{\tyreal}{\cdots}}$, so each $g_i$ is PACAP on~$U'$.
  By the composition lemma (\cref{lem:composition}), $\lambda u.\,c(g_1(u),\dots,g_n(u)) \in \rel{U'}{\tyreal}$.
\end{proof}

\section{Encoding the Examples and Verifying PACAP}
\label{appendix:recursion-encoding}

The core calculus in \cref{sec:core-calculus} has countable sum types and case analysis but no explicit fixpoint operator.
We show here that this suffices to encode primitive recursion and stream corecursion, as used in \cref{sec:geometric}. For the geometric distribution, which involves unbounded search, we verify the PACAP property of the resulting likelihood directly.
These encodings are standard in type theory and programming language semantics, but we include them for completeness, in case the relationship between the Haskell examples and the core calculus may not be immediately apparent.

The key observation is that countable case expressions make the syntax \emph{infinitary}: a single well-typed term may be an infinite (but countably branching) syntax tree.
Haskell's general recursion is used in practice to describe these infinite terms finitely, but the denotational semantics is defined on the infinite terms themselves.

\paragraph{Primitive recursion over $\NN$.}
Since $\NN = \tysum{n}{\NN}\tyunit$ is a type in the calculus, countable case analysis directly provides primitive recursion.
Given a base case $z : \tau$ and a step function $s : \NN \to \tau \to \tau$ (both terms in the calculus), the primitive recursor $\mathit{rec}(z, s) : \NN \to \tau$ is:
\[
  \mathit{rec}(z, s) \;=\; \lambda n.\;\tmcase{n}{\_}{e_i}{\NN}
\]
where each branch $e_i$ is the $i$-fold syntactic unrolling of $s$: $e_0 = z$, $e_1 = s\;\tminj{0}{()}\;z$, $e_2 = s\;\tminj{1}{()}\;(s\;\tminj{0}{()}\;z)$, and so on.
Each $e_i$ is a finite well-typed term of type $\tau$, built by literally inlining $s$ a total of $i$ times.
No fixpoint combinator is needed; the countable case expression directly enumerates all branches.

\paragraph{Stream corecursion.}
A stream of type $\NN \to \tyreal$ is a function from $\NN$ to $\tyreal$.
Given a seed $a : \sigma$, a head function $h : \sigma \to \tyreal$, and a tail function $t : \sigma \to \sigma$, the corecursively defined stream $\mathit{unfold}(a, h, t) : \NN \to \tyreal$ is:
\[
  \mathit{unfold}(a, h, t) \;=\; \lambda n.\;\tmcase{n}{\_}{h(\underbrace{t \circ \cdots \circ t}_{i\text{ times}}(a))}{\NN}
\]
That is, the $i$-th branch is the finite term $h(t^i(a))$, where $t^i$ denotes $i$-fold syntactic composition of $t$.

\paragraph{Models covered by the core calculus.}
Using primitive recursion and stream corecursion, the following models from \cref{sec:lazyppl-overview,sec:experiments} can be encoded in the core calculus, and are therefore PACAP by \cref{thm:ad}:
\begin{itemize}
\item \textbf{Gaussian mixture clustering} (\cref{sec:clustering}).
The number of clusters is drawn from a Poisson distribution (a natural number, hence a countable sum type).
The function \lstinline|mapM| then recurses over the finite list \lstinline|[1..k]|, which is primitive recursion over~$\NN$.
\item \textbf{Step regression} (\cref{lst:stepReg}).
The Poisson point process \lstinline|poissonPP| produces an infinite stream of changepoints by stream corecursion (unfold with seed = current position).
The function \lstinline|splice| constructs a piecewise function by consuming this stream, also corecursively.
\item \textbf{Polynomial regression} (\cref{sec:experiments}).
The degree is drawn from a Poisson distribution, and the coefficients are sampled by primitive recursion over the resulting finite list.
\end{itemize}

\paragraph{Models requiring direct verification.}
Recall from \cref{lst:geometric} that the geometric distribution samples an IID Bernoulli sequence and returns the index of the first success.
For simplicity, we take $\mathcal{A} = \NN$ (the monadic seed-splitting into $\NN^*$ described in \cref{sec:monads-in-lambda-calculus} does not affect the argument).

The stream \lstinline|iid (bernoulli p)| is encoded by a $\lambda$-term that reads off coordinates: $\lambda\omega.\lambda n.\,(\omega\;n < p)$, of type $(\NN\to\tyreal)\to(\NN\to\{\tmtrue,\tmfalse\})$, using the PAP comparison constant.
(Recall that booleans $\{\tmtrue,\tmfalse\} = \tysum{b}{\{0,1\}}\tyunit$ are a special case of countable sums.)

The functions \lstinline|findIndex| and \lstinline|find| perform unbounded searches and so cannot be expressed in the core calculus, which for simplicity has no general fixpoint operator.
Consequently, \cref{thm:ad} does not directly apply to programs built from these functions, including the geometric distribution (\cref{lst:geometric}) and the random walk model (\cref{lst:walk}).

Nonetheless, the PACAP property of these likelihoods can be verified directly; we illustrate with the geometric distribution.
The geometric distribution partitions the seed space $\RR^\NN$ (up to a measure-zero set) into countably many analytic cylinders: for each $i \in \NN$, the set $\{\omega \mid \omega_0 \geq p,\;\dots,\;\omega_{i-1}\geq p,\;\omega_i < p\}$ is an analytic cylinder with support $\{0,\dots,i\}$, since the non-strict constraints $p - \omega_j \leq 0$ are analytic and the strict constraint $\omega_i < p$ is absorbed into the open set $V$ (\cref{def:PACAP}).
The uncovered set $\{\omega \mid \forall n.\;\omega_n \geq p\}$ has measure zero under any product measure with $p\in(0,1)$.
On each cylinder, the geometric distribution returns a fixed natural number, so any likelihood obtained by composing with a scoring function $\NN \to \tyreal$ is PACAP by definition.
The same cylinder reasoning applies to any program whose likelihood factors through countably many deterministic branches.

More generally, programs involving unbounded search go beyond our core calculus.
Formal verification of almost-sure termination is an active research area (e.g.~\cite{Barthe_Katoen_Silva_2020}).
Combining such formal guarantees with PACAP-ness would require an `omega-PACAP' notion, analogous to the $\omega$-PAP spaces of~\cite{DBLP:conf/lics/HuotLMS23}, with recursion via chain-complete partial orders that are compatible with the PACAP structure.
For the models considered in this paper, the geometric distribution and random walk involve unbounded search and require the direct cylinder partition verification illustrated above, while the clustering, step regression, and polynomial regression models use only bounded recursion or corecursion and are covered directly by \cref{thm:ad}.

\section{Additional Proofs and iMCMC Kernel}
\subsection{iMCMC Kernel}
\label{app:iMCMCkernel}
One iteration of the iMCMC algorithm (\cref{sec:iMCMC}) yields the kernel:
\begin{equation}
  \label{eq:iMCMCkernel}
\kappa(\mathbf{q}, A) = \int_{\mathbb{Y}}\Bigl([I(\mathbf{q}, \mathbf{p}) \in (A \times \mathbb{Y}) \cap \mathbb{S}]\tilde{\alpha}(\mathbf{q}, \mathbf{p}) + [(\mathbf{q}, \mathbf{p}) \in (A \times \mathbb{Y}) \cap \mathbb{S}](1-\tilde{\alpha}(\mathbf{q}, \mathbf{p}))\Bigr)p_\mathbf{q}(\mathbf{p})\mu_{\mathbb{Y}}(\diff \mathbf{p}).
\end{equation}
Informally, $\kappa(\mathbf{q}, A)$ tells us the probability of starting with a position $\mathbf{q}$ and returning a position in $A$. 
The first term of the sum corresponds to accepting the proposed state $I(\mathbf{q}, \mathbf{p})$, and the second term corresponds to rejecting it and returning the same position $\mathbf{q}$.

\subsection{Proof of \cref{prop:HMCacc_ratio}}
\label{proof:accHMC}
\begin{proof} 
  By \cref{prop:iMCMC-stationary-kernel}, it is enough to let 
  \[\alpha(\mathbf{q}, \mathbf{p}) = \frac{\zeta(\mathbf{q}', \mathbf{p}')}{\zeta(\mathbf{q}, \mathbf{p})} \cdot \left(\frac{\diff (\text{Leb}_{2n}\circ(\Psi^{(L)})^{-1})}{\diff \text{Leb}_{2n}}\right)(\mathbf{q}, \mathbf{p}).\]

  Using that the leapfrog integrator $\Psi$ is volume preserving (\cref{prop:leapfrog-reversible-volume-preserving}), i.e. $\frac{\diff (\text{Leb}_{2n}\circ(\Psi^{(L)})^{-1})}{\diff \text{Leb}_{2n}}(\mathbf{q}, \mathbf{p}) = 1$ we can simplify $\alpha$:
  \[\alpha(\mathbf{q}, \mathbf{p}) = \frac{\zeta(\mathbf{q}', \mathbf{p}')}{\zeta(\mathbf{q}, \mathbf{p})} = \frac{\exp(-H(\mathbf{q}', \mathbf{p}'))}{\exp(-H(\mathbf{q}, \mathbf{p}))} = \frac{l(\mathbf{q}')\varphi_n(\mathbf{p}')}{l(\mathbf{q})\varphi_n(\mathbf{p})}.\]
\end{proof}

\subsection{Reduction of $\psi$ to $\tau$ on Sites with $0$ Gradient}
\begin{proposition}\label{proof:rotation-constants}
For $\epsilon_q=\sin\theta$ and $\epsilon_p=\tan(\theta/2)$, the map $\psi$ restricted to any site $a$ with $G(\mathbf{q})_a=0$ equals the rotation $\tau$.
\end{proposition}
\begin{proof}
At such a site $\phi^P_t$ reduces to $(q,p)\mapsto(q,\,p-tq)$, so composing $\psi=\phi^P_{\epsilon_p}\circ\phi^Q_{\epsilon_q}\circ\phi^P_{\epsilon_p}$ on $(q_a,p_a)$ gives
\[
\begin{pmatrix} q'\\ p'\end{pmatrix}=
\begin{pmatrix} 1-\epsilon_p\epsilon_q & \epsilon_q\\[2pt] -\epsilon_p(2-\epsilon_p\epsilon_q) & 1-\epsilon_p\epsilon_q\end{pmatrix}
\begin{pmatrix} q\\ p\end{pmatrix}.
\]
Since $\epsilon_p\epsilon_q=\tan(\theta/2)\sin\theta=2\sin^2(\theta/2)=1-\cos\theta$, the diagonal entries equal $\cos\theta$ and the top-right entry is $\sin\theta$.
Moreover $\epsilon_p(2-\epsilon_p\epsilon_q)=\tan(\theta/2)(1+\cos\theta)=\sin\theta$, so the bottom-left entry is $-\sin\theta$. The matrix is therefore exactly $\tau$.
\end{proof}

\subsection{Proof of \cref{prop:HMCmod-properties}}
\label{proof:HMC1}
\begin{proof}
  \leavevmode
\begin{itemize}
  \item \cref{cond:non-visited-cancel-out}: Let $I' = F' \circ \tau^L$, where $F': \mathbb{R}^2 \rightarrow \mathbb{R}^2$ is the momentum flip: $F'(q, p)= (q, -p)$. 
  We get that $I'(q, p) = (\cos (L\theta)q+\sin(L\theta)p, \sin(L\theta)q -\cos (L\theta)p).$
  $I'$ is a measure preserving involution (with respect to $\text{Leb}_{2}$) on $\mathbb{R}^2$. 
  Moreover, $(\varphi_{X}\times \varphi_{Y})(q, p) = (\varphi_{X}\times \varphi_{Y})(I'(q,p))$.

  \item \cref{cond:visited-property}: Let $T$ be a finite subset of $\mathcal{A}$ and let $(\mathbf{q}, \mathbf{p}) \in \mathbb{S}$ with $v(\mathbf{q}, \mathbf{p}) = T$.
  This means that $\cup_{i=0}^{L} v_l(\mathbf{q}^{(i)}) = T$, where $(\mathbf{q}^{(i)}, \mathbf{p}^{(i)}) = \psi^{i}(\mathbf{q}, \mathbf{p})$.
  This implies that in order to compute $G(\mathbf{q}^{(i)})_T$ one only needs to know $\mathbf{q}^{(i)}_T$.
  By induction we can prove that is enough to know $\mathbf{q}_T, \mathbf{p}_T$ to compute $\mathbf{q}^{(i)}_T$ and $\mathbf{p}^{(i)}_T$. 
  So changing the values at any sites $\mathcal{A}\setminus T$ in $\mathbf{q}, \mathbf{p}$ would not change $(\mathbf{q}^{(i)}_T,\mathbf{p}^{(i)}_T)$. 
  Therefore if $v(\mathbf{q}, \mathbf{p}) = T$ then $\{(\mathbf{x}, \mathbf{y})\in \mathbb{X} \times \mathbb{Y} \mid (\mathbf{x}_T, \mathbf{y}_T) = (\mathbf{q}_T, \mathbf{p}_T)\} \subseteq v^{-1}(T)$.
  We can deduce that there exists a measurable set $C_T \in \Sigma_{X^{|T|}} \otimes \Sigma_{Y^{|T|}}$ such that $v^{-1}(T) = \{(\mathbf{q}, \mathbf{p}) \mid (\mathbf{q}_T, \mathbf{p}_T) \in C_T\}$.

  \item \cref{cond:involution-decomposition}: Consider the following maps with domain and codomain included in $\mathbb{R}^{T}$:
  \begin{align*}
    \psi_T  = \phi_T^P \circ \phi_T^Q \circ \phi_T^P, &\qquad F_T(\mathbf{q}, \mathbf{p})  = (\mathbf{q}, -\mathbf{p})\\
    \phi_T^P(\mathbf{q}, \mathbf{p}) = (\mathbf{q}, \mathbf{p} - \epsilon_p\mathbf{q} - \epsilon_p G_T(\mathbf{q})), &\qquad \phi_T^Q(\mathbf{q}, \mathbf{p}) = (\mathbf{q} + \epsilon_q\mathbf{p}, \mathbf{p})
  \end{align*} 
  Since $v(\mathbf{q}, \mathbf{p}) = T$ then it is enough to know the value of $\mathbf{q}^{(L)}$ at sites $T$ to check if $(F \circ \psi^L)(\mathbf{q}, \mathbf{p})$ is in $\mathbb{S}$ or not. 
  Hence, the map $I_T: C_T \rightarrow C_T$ is well defined: 
  \begin{equation*}
    I_T(\mathbf{q}_T, \mathbf{p}_T) := 
    \begin{cases}
        (F_T \circ \psi^L_T)(\mathbf{q}_T, \mathbf{p}_T) & \text{ if }  (F \circ \psi^L)(\mathbf{q}, \mathbf{p}) \in \mathbb{S}\\
        (\mathbf{q}_T, \mathbf{p}_T) &  \text{ otherwise}.
    \end{cases}
  \end{equation*}
  It is easy to check that $I_T$ is a measure preserving involution on $C_T$.
  
  Since $G(\mathbf{q}^{(i)})_j = 0$ for any $j \in \mathcal{A}\setminus T$ we get $(\mathbf{q}^{(L)}_j, \mathbf{p}^{(L)}_j) = I'(\mathbf{q}_j, \mathbf{p}_j)$.
  Let $(\mathbf{q}', \mathbf{p}') = I(\mathbf{q}, \mathbf{p})$, then by definition of $I$, $(\mathbf{q}'_j, \mathbf{p}'_j) = I'(\mathbf{q}_j, \mathbf{p}_j)$ for any $j \in \mathcal{A}\setminus T$.
\end{itemize}
\end{proof}

\subsection{Proof of \cref{prop:HMC2-properties}}
\label{proof:HMC2conditions}
\begin{proof}
\leavevmode
\begin{itemize}
  \item \cref{cond:non-visited-cancel-out}: On a site $j$ with $G(\mathbf{q}^{(i)})_j = 0$ for all $i$, the kick $\phi_t^P(\mathbf{q}, \mathbf{p}) = (\mathbf{q}, \mathbf{p} - tG(\mathbf{q}))$ is the identity on the $j$-th coordinate, so $\psi = \phi_{\theta/2}^P \circ \phi_\theta^{Q,P} \circ \phi_{\theta/2}^P$ acts there as its middle factor $\phi_\theta^{Q,P}$, i.e.\ the clockwise rotation $\tau$ by $\theta$. Let $I' = F' \circ \tau^L$, where $F'(q, p) = (q, -p)$ is the momentum flip. Then $I'(q, p) = (\cos(L\theta)q + \sin(L\theta)p,\ \sin(L\theta)q - \cos(L\theta)p)$ is a measure preserving involution (with respect to $\text{Leb}_2$) on $\mathbb{R}^2$. Since $\tau$ is a rotation we get $(\varphi_X \times \varphi_Y)(q, p) = (\varphi_X \times \varphi_Y)(I'(q,p))$.

  \item \cref{cond:visited-property}: As in \cref{proof:HMC1}, using that the integrator acts site-wise.
  Let $T$ be a finite subset of $\mathcal{A}$ and let $(\mathbf{q}, \mathbf{p}) \in \mathbb{S}$ with $v(\mathbf{q}, \mathbf{p}) = T$. 
  Both $\phi^P$ and $\phi^{Q,P}$ act site-wise, coupled across sites only through the gradient $G$, which is supported in $T$. 
  By induction it is enough to know $\mathbf{q}_T, \mathbf{p}_T$ to compute $\mathbf{q}^{(i)}_T$ and $\mathbf{p}^{(i)}_T$, so changing the values at sites $\mathcal{A}\setminus T$ does not change $(\mathbf{q}^{(i)}_T, \mathbf{p}^{(i)}_T)$. 
  Therefore there exists a measurable set $C_T \in \Sigma_{X^{|T|}} \otimes \Sigma_{Y^{|T|}}$ such that $v^{-1}(T) = \{(\mathbf{q}, \mathbf{p}) \mid (\mathbf{q}_T, \mathbf{p}_T) \in C_T\}$.

  \item \cref{cond:involution-decomposition}: Let $\psi_T, F_T, \phi_T^P, \phi_T^{Q,P}: \mathbb{R}^{2|T|} \to \mathbb{R}^{2|T|}$ be the maps
  \begin{align*}
    \psi_T = \phi_T^P \circ \phi_T^{Q,P} \circ \phi_T^P, &\qquad F_T(\mathbf{q}, \mathbf{p}) = (\mathbf{q}, -\mathbf{p})\\
    \phi_T^P(\mathbf{q}, \mathbf{p}) = (\mathbf{q}, \mathbf{p} - \tfrac{\theta}{2} G_T(\mathbf{q})), &\qquad \phi_T^{Q,P}(\mathbf{q}, \mathbf{p}) = (\cos\theta\, \mathbf{q} + \sin\theta\, \mathbf{p},\ -\sin\theta\, \mathbf{q} + \cos\theta\, \mathbf{p}).
  \end{align*}
  Both $\phi_T^P$ and $\phi_T^{Q,P}$ are volume preserving so $\psi_T$ is also volume preserving. 
  Since $v(\mathbf{q}, \mathbf{p}) = T$, it is enough to know $\mathbf{q}^{(L)}$ at sites $T$ to check whether $(F \circ \psi^L)(\mathbf{q}, \mathbf{p}) \in \mathbb{S}$, so the map $I_T: C_T \to C_T$ is well defined:
  \begin{equation*}
    I_T(\mathbf{q}_T, \mathbf{p}_T) :=
    \begin{cases}
        (F_T \circ \psi^L_T)(\mathbf{q}_T, \mathbf{p}_T) & \text{ if } (F \circ \psi^L)(\mathbf{q}, \mathbf{p}) \in \mathbb{S}\\
        (\mathbf{q}_T, \mathbf{p}_T) & \text{ otherwise},
    \end{cases}
  \end{equation*}
  and it is a measure preserving involution on $C_T$. 
  Since $G(\mathbf{q}^{(i)})_j = 0$ for any $j \in \mathcal{A}\setminus T$, we get $(\mathbf{q}^{(L)}_j, \mathbf{p}^{(L)}_j) = I'(\mathbf{q}_j, \mathbf{p}_j)$. 
  Letting $(\mathbf{q}', \mathbf{p}') = I(\mathbf{q}, \mathbf{p})$, by definition of $I$ we have $(\mathbf{q}'_j, \mathbf{p}'_j) = I'(\mathbf{q}_j, \mathbf{p}_j)$ for any $j \in \mathcal{A}\setminus T$.
\end{itemize}
\end{proof}

\subsection{Proof of \cref{prop:HMC3-properties}}
\label{proof:HMC3conditions}
\begin{proof}
\Cref{cond:non-visited-cancel-out} holds by construction: $I'$ is chosen as a $\text{Leb}_2$-preserving involution with $(\varphi_X \times \varphi_Y) \circ I' = \varphi_X \times \varphi_Y$.
\Cref{cond:visited-property} is proved by the same induction on the trajectory index as in \cref{prop:HMCmod-properties}: the gradient at each step is supported in $T = v(\mathbf{q},\mathbf{p})$, so $(\mathbf{q}^{(i)}_T, \mathbf{p}^{(i)}_T)$ depends only on $(\mathbf{q}_T, \mathbf{p}_T)$ and $v^{-1}(T)$ is a cylinder set.
\Cref{cond:involution-decomposition} is built into the definition of $I$: on the visited coordinates $I$ restricts to the finite-dimensional leapfrog involution $I_T = F_T \circ \psi_T^L$ (measure-preserving by \cref{prop:leapfrog-reversible-volume-preserving}), and on each unvisited coordinate $I$ acts independently as $I'$.
\end{proof}

\section{Framework B: Generalizing Framework A from Lazy HMC towards Lazy NUTS}
\label{sec:Framework-B}
In this section we present a slightly more general framework than the one presented in \cref{sec:Framework-A}.
This framework can still be seen as an instance of the iMCMC framework (\cref{sec:iMCMC}).
The idea is to have multiple involutions and to sample an extra variable that tells us which involution to use.  
One could see this as the GIST framework \citep{bou2024gist} which uses Gibbs sampling to tune the HMC hyperparameters to get locally adaptive HMC.
NUTS \citep{NUTS} is also part of GIST. 
We also show in \cref{sec:LazyNUTS} how we can construct lazyNUTS on rose trees as part of Framework B. 

\paragraph{Tuning parameters random variable.} We augment the state space $\mathbb{S}$ with the random variable $\beta$ which selects the involution to be applied on the state $(\mathbf{q}, \mathbf{p})$.
Let $(\mathbb{B}, \Sigma_{\mathbb{B}}, \mu_{\mathbb{B}})$ be a $\sigma$-finite measure space and for any $\mathbf{s} \in \mathbb{S}$ let $p(\cdot \mid \mathbf{s})$ be the probability density with respect to $\mu_{\mathbb{B}}$ from which $\beta$ is sampled given $\mathbf{s}$.
Let $W(\mathbf{s}, \beta) = w(\mathbf{s})p(\beta \mid \mathbf{s})$ and $\mathbb{Z} = \{(\mathbf{s}, \beta) \in \mathbb{S}\times \mathbb{B} \mid w(\mathbf{s})p(\beta \mid \mathbf{s}) > 0\}$ equipped with $\sigma$-algebra $\Sigma_{\mathbb{Z}} = \{A \cap \mathbb{Z} \mid A \in \Sigma_{\mathbb{S}} \otimes \Sigma_{\mathbb{B}}\}$ and measure $\mu_{\mathbb{Z}} = \mu_{\mathbb{S}} \times \mu_{\mathbb{B}}$.
We will usually denote with $\mathbf{s}$ a position-momentum pair $(\mathbf{q}, \mathbf{p})$.

\paragraph{Involutions.} Our main involution $I$ is now defined on $\mathbb{Z}$. We assume $I: \mathbb{Z} \rightarrow \mathbb{Z}$ is measurable and we can write it as $I(\mathbf{s}, \beta) = (g(\mathbf{s}, \beta), h(\mathbf{s}, \beta))$ with measurable $h: \mathbb{Z} \rightarrow \mathbb{B}$ and measurable bijection $g: \mathbb{Z} \rightarrow \mathbb{S}$ with measurable inverse. 
Then, under certain assumptions on measurability and absolute continuity (see \cref{prop:prod-of-RN-deriv2} in the Appendix), we have : 
\begin{equation}
  \label{eqn:lazyGIST-derivative}
\left(\frac{\diff \mu_{\mathbb{Z}}\circ I^{-1}}{\diff\mu_{\mathbb{Z}}}\right)(\mathbf{s},\beta) 
  = \left(\frac{\diff\left(\mu_{\mathbb{S}}\circ g_{t_{\mathbf{s}}^{-1}(\beta)}^{-1}\right)}{\diff \mu_{\mathbb{S}}}\right)(\mathbf{s}) \left(\frac{\diff(\mu_{\mathbb{B}}\circ t_{\mathbf{s}}^{-1})}{\diff\mu_{\mathbb{B}}}\right)(\beta),
\end{equation} 
where $g_{\beta}(\mathbf{s}) = g(\mathbf{s}, \beta)$ and $t_{\mathbf{s}} (\beta) = h(g_{\beta}(\mathbf{s})^{-1}, \beta).$

One iteration of a method from Framework B given position $\mathbf{q}, p(\cdot \mid \mathbf{q}, \mathbf{p}), \mu_X, \mu_Y, l, I, \alpha$ does:
\begin{enumerate}
  \item sample $\mathbf{p} \sim \mu_{\mathbb{Y}}$ lazily, sampling each component on-demand
  \item sample $\beta \sim p(\cdot \mid \mathbf{q}, \mathbf{p})$
  \item let $(\mathbf{q}', \mathbf{p}', \beta') = I(\mathbf{q}, \mathbf{p}, \beta)$
  \item with probability $\min \{1, \alpha(\mathbf{q}, \mathbf{p}, \beta)\}$ accept and return $\mathbf{q}'$, otherwise reject and return $\mathbf{q}$.
\end{enumerate} 
Let $\alpha(\mathbf{s},\beta) = \frac{W(I(\mathbf{s}, \beta))}{W(\mathbf{s}, \beta)}\left(\frac{\diff \mu_{\mathbb{Z}}\circ I^{-1}}{\diff\mu_{\mathbb{Z}}}\right)(\mathbf{s},\beta)$ 
, then by \cref{prop:iMCMC-stationary-kernel} we get the following. 
\begin{proposition}
The kernel resulting from Framework B is stationary with respect to the target $\nu$.
\end{proposition}

We will now focus on the particular case in which $h$ depends only on $\beta$ for any $(\mathbf{s}, \beta) \in \mathbb{Z}$, so $h:\mathbb{B} \rightarrow \mathbb{B}$ and $I(\mathbf{s}, \beta) = (g(\mathbf{s}, \beta), h(\beta)).$

Moreover, assume that for each $\beta$, there exists $v_{\beta}: \mathbb{S} \rightarrow \mathcal{P}_{fin} (\mathcal{A})$ such that $g(\cdot, \beta)$ is a measure preserving involution that satisfies the conditions stated in \cref{cond:involution-decomposition,cond:visited-property,cond:non-visited-cancel-out} from \cref{sec:Framework-A}.

\begin{proposition}
\label{prop:framework-B-particular}
Under the extra assumptions for $h$ and $g$ and using the following acceptance ratio in Framework B: 
\[\alpha(\mathbf{s}, \beta) = \frac{W(I(\mathbf{s}, \beta))}{W(\mathbf{s}, \beta)}\frac{ (\varphi_X \times \varphi_Y)(g(\mathbf{s}, h^{-1}(\beta))_{v_{h^{-1}(\beta)}(\mathbf{s})})}{(\varphi_X \times \varphi_Y)(\mathbf{s}_{v_{h^{-1}(\beta)}(\mathbf{s})})} \left(\frac{\diff(\mu_{\mathbb{B}}\circ h^{-1})}{\diff\mu_{\mathbb{B}}}\right)(\beta)\]
we get a stationary kernel with respect to the target measure $\nu$.
\end{proposition}
\begin{proof}

From \cref{eqn:lazyGIST-derivative} and the fact that $I(\mathbf{s}, \beta) = (g(\mathbf{s}, \beta), h(\beta))$ we get:
\[\left(\frac{\diff \mu_{\mathbb{Z}}\circ I^{-1}}{\diff\mu_{\mathbb{Z}}}\right)(\mathbf{s},\beta) 
  = \left(\frac{\diff\left(\mu_{\mathbb{S}}\circ g_{h^{-1}(\beta)}^{-1}\right)}{\diff \mu_{\mathbb{S}}}\right)(\mathbf{s}) \left(\frac{\diff(\mu_{\mathbb{B}}\circ h^{-1})}{\diff\mu_{\mathbb{B}}}\right)(\beta). \]

  Since for each $\beta$, there exists $v_{\beta}: \mathbb{S} \rightarrow \mathcal{P}_{fin} (\mathcal{A})$ such that $g(\cdot, \beta)$ is a measure preserving involution that satisfies the conditions stated in \cref{cond:involution-decomposition,cond:visited-property,cond:non-visited-cancel-out} from \cref{sec:Framework-A}, we can use \cref{prop:R-N-derivative}: 
\begin{equation}
\label{eqn:framework-B-RN}
\left(\frac{\diff \mu_{\mathbb{Z}}\circ I^{-1}}{\diff\mu_{\mathbb{Z}}}\right)(\mathbf{s},\beta) 
  = \frac{ (\varphi_X \times \varphi_Y)(g(\mathbf{s}, h^{-1}(\beta))_{v_{h^{-1}(\beta)}(\mathbf{s})})}{(\varphi_X \times \varphi_Y)(\mathbf{s}_{v_{h^{-1}(\beta)}(\mathbf{s})})} \left(\frac{\diff(\mu_{\mathbb{B}}\circ h^{-1})}{\diff \mu_{\mathbb{B}}}\right)(\beta).
\end{equation} 

Then by \cref{prop:iMCMC-stationary-kernel} and \cref{eqn:framework-B-RN} we get the following: 
\[\alpha(\mathbf{s}, \beta) = \frac{W(I(\mathbf{s}, \beta))}{W(\mathbf{s}, \beta)}\frac{ (\varphi_X \times \varphi_Y)(g(\mathbf{s}, h^{-1}(\beta))_{v_{h^{-1}(\beta)}(\mathbf{s})})}{(\varphi_X \times \varphi_Y)(\mathbf{s}_{v_{h^{-1}(\beta)}(\mathbf{s})})} \left(\frac{\diff(\mu_{\mathbb{B}}\circ h^{-1})}{\diff\mu_{\mathbb{B}}}\right)(\beta).\]

\end{proof}

\subsection{LazyNUTS: A No-U-Turn Sampler over Infinite Dimensional State Spaces} 
\label{sec:LazyNUTS} Tuning the hyperparameter $L$ of number of steps in the HMC algorithm is not trivial. If $L$ is too small, then we might get stuck in a local minimum and not explore the whole state space. However, if $L$ is too large, we might explore the same area multiple times in one iteration, which is not efficient. NUTS \citep{NUTS} aims to solve this by adaptively setting the path length $L$ using the doubling procedure with the no-U-turn criterion as the stopping condition. 

\paragraph{The doubling procedure.} Given the initial position-momentum state $\mathbf{s}$, the idea is to construct a set of proposed states $\mathcal{C}_{\mathbf{s}}$ from which we sample the next state according to the weights of the states in $\mathcal{C}_{\mathbf{s}}$. In order to preserve detailed balance, $\mathcal{C}_{\mathbf{s}}$ is constructed by recursively doubling the trajectory either forwards in time (with probability $1/2$) or backwards in time (with probability $1/2$) until the stopping conditions are met. 

\paragraph{Property of $\mathcal{C}_{\mathbf{s}}$.} The stopping condition must be chosen such that for any $(\mathbf{s}') \in \mathcal{C}_{\mathbf{s}}$, the sets $\mathcal{C}_{\mathbf{s}}$ and $\mathcal{C}_{\mathbf{s}'}$ are the same. 
That means the proposed states set resulting from the doubling procedure would be the same if we were to start with any other state in the set and do the doubling procedure.

\paragraph{The no-U-turn criterion.} NUTS stops the doubling procedure if the trajectory constructed so far makes a U-turn. Let $(\mathbf{x}, \mathbf{y}) \in \mathbb{R}^{2n}$ and $(\mathbf{x}', \mathbf{y}')\in \mathbb{R}^{2n}$ be two states in the trajectory so that the direction is from $(\mathbf{x}, \mathbf{y})$ to $(\mathbf{x}', \mathbf{y}')$, i.e. $(\mathbf{x}', \mathbf{y}')$ can be obtained by performing some $i$ leapfrog steps from $(\mathbf{x}, \mathbf{y})$. The two states satisfy the U-turn criterion if: \[(\mathbf{x}' - \mathbf{x}) \cdot \mathbf{y} < 0 \text{ or } (\mathbf{x}' - \mathbf{x}) \cdot \mathbf{y}' < 0.\] This tells that if we were to move a very tiny step forwards from $\mathbf{x}'$ with momentum $\mathbf{y}'$ or a very tiny step backwards from $\mathbf{x}$ with momentum $\mathbf{y}$, we would end up reducing the distance between the positions $\mathbf{x}$ and $\mathbf{x'}$. Therefore, a U-turn is made, so the doubling procedure should be stopped. This is checked for certain pairs of states, in such a way that detailed balance holds. For rigorous details and proof of correctness see \cite{NUTS}. 

\paragraph{LazyNUTS setup.} Let $(\mathbb{B}, \Sigma_{\mathbb{B}}, \mu_{\mathbb{B}})$ be the measure space of the integers together with the counting measure. The hyperparameter will be a number of leapfrog steps, so in this case $\beta$ will be $L$. Writing NUTS as part of Framework B, $L$ represents how many leapfrog steps and in which direction we should go from the initial state in order to get to the proposed state. In practice this corresponds to performing the doubling procedure and choosing the next state according to the correct weights. Hence, $L$ would be sampled at the end of the doubling procedure from the categorical distribution defined by the weights of the states that are part of the set of proposed states. 

\paragraph{Constructing the involution.} Let $I^{(L)} = F \circ \psi^L$ be an involution with $\psi$ as defined in \cref{sec:LazyHMC} together with the corresponding visited sites function $v_L(\mathbf{q}, \mathbf{p}) = \cup_{i=0}^L v_l(\mathbf{q}^{(i)})$, where $(\mathbf{q}^{(i)}, \mathbf{p}^{(i)}) = \psi^i(\mathbf{q}, \mathbf{p})$. The conditions in \cref{cond:visited-property,cond:non-visited-cancel-out,cond:involution-decomposition} hold for $I^{(L)}$. 
Notice that the definition of $I^{(L)}$ also makes sense for $L<0$: $I^{(L)} = F \circ \psi^L = \psi^{-L} \circ F$. 
Therefore, the involution $I^{(L)}$ represents $L$ leapfrog steps followed by negation of the momentum in the direction of the momentum if $L>0$, otherwise, we first negate the momentum and move $L$ leapfrog steps. 
Since $(I^{(L)})^{-1} = I^{(-L)}$ we get the involution $I: \mathbb{S} \times \mathbb{B} \rightarrow \mathbb{S} \times \mathbb{B}$ with $I(\mathbf{q}, \mathbf{p}, L) = (I^{(L)}(\mathbf{q}, \mathbf{p}), -L)$. 

\paragraph{Distribution of $L$.} Suppose we start with an initial state $(\mathbf{q}, \mathbf{p})$ and we perform the doubling procedure. 
Consider the set $\mathcal{C}_{(\mathbf{q}, \mathbf{p})}$ resulting from the doubling procedure with a suitable stopping condition. 
Let $V(\mathcal{C}_{(\mathbf{q}, \mathbf{p})}) = \cup_{(\mathbf{q}', \mathbf{p}')\in \mathcal{C}_{(\mathbf{q}, \mathbf{p})}} v_l(\mathbf{q}')$ be the set of visited sites in at least one position in $\mathcal{C}_{(\mathbf{q}, \mathbf{p})}$. 
We can define the distribution of L as: 
\begin{equation*} p(L \mid \mathbf{q}, \mathbf{p}) = \frac{l(\mathbf{q}^{(L)})\prod_{j \in V(\mathcal{C}_{(\mathbf{q}, \mathbf{p})}) } \varphi_X(\mathbf{q}^{(L)}_j) \varphi_Y(\mathbf{p}^{(L)}_j)}{\sum_{(\mathbf{q}', \mathbf{p}') \in \mathcal{C}_{(\mathbf{q}, \mathbf{p})}} l(\mathbf{q}')\prod_{j \in V(\mathcal{C}_{(\mathbf{q}, \mathbf{p})}) } \varphi_X(\mathbf{q}'_j) \varphi_Y(\mathbf{p}'_j)} \end{equation*} 
if $(\mathbf{q}^{(L)}, \mathbf{p}^{(L)}) \in \mathcal{C}_{(\mathbf{q}, \mathbf{p})}$ and $p(L \mid \mathbf{q}, \mathbf{p}) = 0$ otherwise, where $(\mathbf{q}^{(L)}, \mathbf{p}^{(L)}) = I^{(L)}(\mathbf{q}, \mathbf{p})$.

\begin{proposition}
  \label{prop:NUTS-alpha-is-1}
  Suppose the stopping condition is start-independent, i.e. $\mathcal{C}_{\mathbf{s}'}=\mathcal{C}_{\mathbf{s}}$ (hence $V(\mathcal{C}_{\mathbf{s}'})=V(\mathcal{C}_{\mathbf{s}})$) for every $\mathbf{s}'\in\mathcal{C}_{\mathbf{s}}$. Then the acceptance ratio from \cref{prop:framework-B-particular} is $1$.
\end{proposition}

\begin{proof}

Let $h(L) = -L$, then $h$ is measure preserving with respect to $\mu_{\mathbb{B}}$. 
Let $L$ be such that $I^{(L)}((\mathbf{q}, \mathbf{p})) \in \mathcal{C}_{(\mathbf{q}, \mathbf{p})}$. 
The acceptance ratio from \cref{prop:framework-B-particular} can be written as: 
\begin{align*} \alpha(\mathbf{s}, L) &= \frac{W(I(\mathbf{s}, L))}{W(\mathbf{s}, L)}\frac{ (\varphi_X \times \varphi_Y)(I^{(L)}(\mathbf{s}, h^{-1}(L))_{v_{h^{-1}(L)}(\mathbf{s})})}{(\varphi_X \times \varphi_Y)(\mathbf{s}_{v_{h^{-1}(L)}(\mathbf{s})})} \\ & = \frac{l(\mathbf{q}^{(L)}) l(\mathbf{q})\prod_{j \in V(\mathcal{C}_{(\mathbf{q}^{(L)}, \mathbf{p}^{(L)})}) } \varphi_X(\mathbf{q}_j) \varphi_Y(\mathbf{p}_j)}{l(\mathbf{q})l(\mathbf{q}^{(L)})\prod_{j \in V(\mathcal{C}_{(\mathbf{q}, \mathbf{p})}) } \varphi_X(\mathbf{q}^{(L)}_j) \varphi_Y(\mathbf{p}^{(L)}_j)} \frac{\prod_{j \in v_L(\mathbf{q}, \mathbf{p}) } \varphi_X(\mathbf{q}^{(L)}_j) \varphi_Y(\mathbf{p}^{(L)}_j)}{\prod_{j \in v_L(\mathbf{q}, \mathbf{p}) } \varphi_X(\mathbf{q}_j) \varphi_Y(\mathbf{p}_j)}\\ &=1 \end{align*} 
The last step followed from the fact that $v_L(\mathbf{q}, \mathbf{p}) \subset V(\mathcal{C}_{(\mathbf{q}, \mathbf{p})})$ and for any $j \in V(\mathcal{C}_{(\mathbf{q}, \mathbf{p})}) \setminus v_L(\mathbf{q}, \mathbf{p})$ we have $\varphi_X(\mathbf{q}^{(L)}_j) \varphi_Y(\mathbf{p}^{(L)}_j) = \varphi_X(\mathbf{q}_j) \varphi_Y(\mathbf{p}_j)$. 
Here we also use $V(\mathcal{C}_{(\mathbf{q}^{(L)}, \mathbf{p}^{(L)})}) = V(\mathcal{C}_{(\mathbf{q}, \mathbf{p})})$, which is the start-independence assumption.
\end{proof}

This corresponds to the fact that there is no more accept/reject state after the doubling procedure (since we sample one of the states in the constructed set and then return it). 
One iteration of lazyNUTS does the following:
\begin{itemize}
  \item sample $\mathbf{p} \sim \mu_{\mathbb{Y}}$ lazily, sampling each component on-demand
  \item construct $\mathcal{C}_{(\mathbf{q}, \mathbf{p})}$ using the doubling procedure with suitable stopping conditions. 
  \item sample the next state $(\tilde{\mathbf{q}}, \tilde{\mathbf{p}})$ from the distribution that assigns for each state $(\mathbf{q}', \mathbf{p}') \in \mathcal{C}_{(\mathbf{q}, \mathbf{p})}$ the (unnormalized) weight $l(\mathbf{q}')\prod_{j \in V(\mathcal{C}_{(\mathbf{q}, \mathbf{p})}) } \varphi_X(\mathbf{q}'_j) \varphi_Y(\mathbf{p}'_j)$ 
  \item return $\tilde{\mathbf{q}}$. 
  \end{itemize} 

\paragraph{No-U-turn condition in LazyNUTS} To get a stopping condition of the doubling procedure similar to the no-U-turn condition we must take into consideration only finitely many dimensions on which to check if a U-turn was made. In the case in which we know the visited sites for a state, like in \cref{sec:LazyHMC-with-original-integrator}, the no-U-turn condition takes into consideration all the dimensions corresponding to the visited sites. However, if we do not have access to all the visited sites, we could still check if a U-turn was made using the sites for which the two positions have non-zero gradient. This is the stopping condition that we used in our implementation together with an upper bound $M$ on the number of proposed states, usually $M = 2^6$.

\Cref{prop:NUTS-alpha-is-1} assumes a start-independent set of proposed states. Our implementation checks the no-U-turn criterion only on sites with non-zero gradient and additionally caps the number of proposed states at $M$. The gradient-restricted U-turn check is intended to preserve this invariance; the cap $M$, however, is a practical termination (analogous to the maximum tree depth in standard NUTS implementations) and, like the maximum tree depth, is not guaranteed to preserve it. When $M$ does not bind, the construction coincides with the uncapped doubling for which the invariance holds.

\section{Technical Propositions}

\begin{proposition} \label{prop:R-N-derivative-proof}

\[\left( \frac{\diff (\mu_{\mathbb{S}} \circ I^{-1})}{\diff \mu_{\mathbb{S}}} \right)(\mathbf{q}, \mathbf{p}) = \frac{\prod_{i \in v(\mathbf{q}, \mathbf{p})} \varphi_X(\mathbf{q}'_i) \varphi_Y(\mathbf{p}'_i)}{\prod_{i \in v(\mathbf{q}, \mathbf{p})} \varphi_X(\mathbf{q}_i) \varphi_Y(\mathbf{p}_i)},\]

($\mu_{\mathbb{S}}$ - almost everywhere), where $(\mathbf{q}', \mathbf{p}') = I(\mathbf{q}, \mathbf{p})$.
\end{proposition}

\begin{proof}
Let $r(\mathbf{q}, \mathbf{p}) = \frac{\prod_{i \in v(\mathbf{q}, \mathbf{p})} \varphi_X(\mathbf{q}'_i) \varphi_Y(\mathbf{p}'_i)}{\prod_{i \in v(\mathbf{q}, \mathbf{p})} \varphi_X(\mathbf{q}_i) \varphi_Y(\mathbf{p}_i)}$, where $(\mathbf{q}', \mathbf{p}') = I(\mathbf{q}, \mathbf{p})$.

Let $\mathcal{I}$ be a finite subset of $\mathcal{A}$, $B \in \Sigma_{\mathbb{R}^{2|\mathcal{I}|}} $ and let $A = \{(\mathbf{q}, \mathbf{p}) \in \mathbb{S} \mid (\mathbf{q}_{\mathcal{I}}, \mathbf{p}_{\mathcal{I}}) \in B\}$, where $\mathbf{q}_{\mathcal{I}}$ are the values of the tree $\mathbf{q}$ at nodes $\mathcal{I}$.
Our aim is to prove that for all such finite cylinder sets $A$ we have:
\begin{equation}
    \label{eqn:r-is-radon-nikodym}
    (\mu_{\mathbb{S}} \circ I^{-1})(A) = \int_A r(\mathbf{s}) \mu_{\mathbb{S}}(\diff \mathbf{s}).
\end{equation}

We can partition the state space $\mathbb{S} = \cup_{T \in \mathcal{P}_{fin}(\mathcal{A})} \mathbb{S}_T$, where $\mathbb{S}_T = \{ (\mathbf{q}, \mathbf{p}) \in \mathbb{S} \mid v(\mathbf{q}, \mathbf{p}) = T\}$.

We will first fix $T \in \mathcal{P}_{fin}(\mathcal{A})$ and show that \cref{eqn:r-is-radon-nikodym} holds for $A \cap \mathbb{S}_T$.
Using property \ref{cond:visited-property} we can write $\mathbb{S}_T = \{(\mathbf{q}, \mathbf{p}) \in \mathbb{S} \mid (\mathbf{q}_T, 
\mathbf{p}_T) \in C\}$ for some set $C \in \Sigma_{\mathbb{R}^{2|T|}}$.

Let $\mathcal{J} = {\mathcal{I}} \cup T$ and $n = |\mathcal{J}|$. 
There exists $D \in \Sigma_{\mathbb{R}^{2n}}$ such that $A \cap \mathbb{S}_T = \{(\mathbf{q}, \mathbf{p}) \in \mathbb{S} \mid (\mathbf{q}_{\mathcal{J}}, 
\mathbf{p}_{\mathcal{J}}) \in D\}$.
From property \ref{cond:involution-decomposition} we have that $I^{-1}(A \cap \mathbb{S}_T) = \{(\mathbf{q}, \mathbf{p}) \in \mathbb{S} \mid (\mathbf{q}_{\mathcal{J}}, 
\mathbf{p}_{\mathcal{J}}) \in f^{-1}(D)\}$, where $f:\mathbb{R}^{2n} \to \mathbb{R}^{2n}$ is given by $f(\mathbf{q}, \mathbf{p})_T = I_T(\mathbf{q}, \mathbf{p})$ and $f(\mathbf{q}, \mathbf{p})_j = I'(\mathbf{q}_j, \mathbf{p}_j)$ for any $j \in \mathcal{J}\setminus T$.
Therefore: 
\[(\mu_{\mathbb{S}} \circ I^{-1})(A \cap \mathbb{S}_T) = (\mu_{X^n} \times \mu_{Y^n})(f^{-1}(D))\]

Notice that for $\mathbf{q}, \mathbf{p} \in A \cap \mathbb{S}_T$, $r(\mathbf{q}, \mathbf{p})$ only depends on the values $(\mathbf{q}_T, \mathbf{p}_T)$.
Using Fubini/Tonelli we get:
\[\int_{A \cap \mathbb{S}_T} r(\mathbf{s}) \mu_{\mathbb{S}}(\diff \mathbf{s}) = \int_{\mathbb{R}^{2n}} [(\mathbf{q}_{\mathcal{J}}, 
\mathbf{p}_{\mathcal{J}}) \in D] \frac{\prod_{i \in T} \varphi_X(\mathbf{q}'_i) \varphi_Y(\mathbf{p}'_i)}{\prod_{i \in T} \varphi_X(\mathbf{q}_i) \varphi_Y(\mathbf{p}_i)} \mu_{X^n}(\diff \mathbf{q}_{\mathcal{J}})\mu_{Y^n}(
\diff \mathbf{p}_{\mathcal{J}})\]

From properties \ref{cond:involution-decomposition}, \ref{cond:non-visited-cancel-out} if $v(\mathbf{q}, \mathbf{p}) = T$ we have that $\varphi_X(\mathbf{q}'_i) \varphi_Y(\mathbf{p}'_i) = \varphi_X(\mathbf{q}_i) \varphi_Y(\mathbf{p}_i)$ for any $ i \in \mathcal{A} \setminus T$.
Hence, 
\begin{align*}
\int_{A \cap \mathbb{S}_T} r(\mathbf{s}) \mu_{\mathbb{S}}(\diff \mathbf{s})  
& = \int_{\mathbb{R}^{2n}} [(\mathbf{q}_{\mathcal{J}}, 
\mathbf{p}_{\mathcal{J}}) \in D] \frac{\prod_{i \in T} \varphi_X(\mathbf{q}'_i) \varphi_Y(\mathbf{p}'_i)}{\prod_{i \in T} \varphi_X(\mathbf{q}_i) \varphi_Y(\mathbf{p}_i)} \mu_{X^n}(\diff \mathbf{q}_{\mathcal{J}})\mu_{Y^n}(
\diff \mathbf{p}_{\mathcal{J}}) \\
& = \int_{\mathbb{R}^{2n}} [(\mathbf{q}_{\mathcal{J}}, 
\mathbf{p}_{\mathcal{J}}) \in D] \frac{\prod_{i \in \mathcal{J}} \varphi_X(\mathbf{q}'_i) \varphi_Y(\mathbf{p}'_i)}{\prod_{i \in \mathcal{J}} \varphi_X(\mathbf{q}_i) \varphi_Y(\mathbf{p}_i)} \mu_{X^n}(\diff \mathbf{q}_{\mathcal{J}})\mu_{Y^n}(
\diff \mathbf{p}_{\mathcal{J}}) \\
& = \int_{\mathbb{R}^{2n}} [(\mathbf{q}_{\mathcal{J}}, 
\mathbf{p}_{\mathcal{J}}) \in D] \frac{ (\varphi_{X^n} \times \varphi_{Y^n})(f(\mathbf{q}_{\mathcal{J}}, \mathbf{p}_{\mathcal{J}}))}{\varphi_{X^n}(\mathbf{q}_{\mathcal{J}})\varphi_{Y^n}(\mathbf{p}_{\mathcal{J}})} \mu_{X^n}(\diff \mathbf{q}_{\mathcal{J}})\mu_{Y^n}(
\diff \mathbf{p}_{\mathcal{J}})\\
& = \int_{\mathbb{R}^{2n}} [(\mathbf{q}_{\mathcal{J}}, 
\mathbf{p}_{\mathcal{J}}) \in D]  (\varphi_{X^n} \times \varphi_{Y^n})(f(\mathbf{q}_{\mathcal{J}}, \mathbf{p}_{\mathcal{J}})) \text{Leb}_{2n}(\diff \mathbf{q}_{\mathcal{J}}, 
\diff \mathbf{p}_{\mathcal{J}})\\
& = \int_{\mathbb{R}^{2n}} [(\mathbf{q}_{\mathcal{J}}, 
\mathbf{p}_{\mathcal{J}}) \in f^{-1}(D)]  (\varphi_{X^n} \times \varphi_{Y^n})(\mathbf{q}_{\mathcal{J}}, \mathbf{p}_{\mathcal{J}}) \text{Leb}_{2n}(\diff \mathbf{q}_{\mathcal{J}}, 
\diff \mathbf{p}_{\mathcal{J}})\\
&= (\mu_{X^n} \times \mu_{Y^n})(f^{-1}(D))
\end{align*}

The second to last step is a change of variable using the measure preserving bijection $f^{-1}$ (as both $I_T$ and $I'$ are measure preserving involutions) with respect to the Lebesgue measure on $\mathbb{R}^{2n}$.

Finally,  
\begin{align*}
    (\mu_{\mathbb{S}} \circ I^{-1})(A) &= (\mu_{\mathbb{S}} \circ I^{-1})(\cup_{T \in \mathcal{P}_{fin}(\mathcal{A})} A \cap \mathbb{S}_T) \\
    &= \mu_{\mathbb{S}}(\uplus_{{T \in \mathcal{P}_{fin}(\mathcal{A})}}I^{-1}(A \cap \mathbb{S}_T))\\
    &= \sum_{T \in \mathcal{P}_{fin}(\mathcal{A})}\mu_{\mathbb{S}}(I^{-1}(A \cap \mathbb{S}_T))\\
    &= \sum_{T \in \mathcal{P}_{fin}(\mathcal{A})}\int_{A \cap \mathbb{S}_T} r(\mathbf{s}) \mu_{\mathbb{S}}(\diff \mathbf{s})\\
    &= \int_A r(\mathbf{s}) \mu_{\mathbb{S}}(\diff \mathbf{s}).
\end{align*}

The sets for which this equality holds is a $\lambda$-system which contains the finite cylinder sets; the finite cylinder sets form a $\pi$-system generating $\Sigma_{\mathbb{S}}$, so by the $\pi$-$\lambda$ theorem the equality holds for all $A \in \Sigma_{\mathbb{S}}$.

\end{proof}

\begin{proposition}
\label{prop:prod-of-RN-deriv2}
Let $(X, \Sigma_X, \mu_X)$ and $(Y, \Sigma_Y, \mu_Y)$ be two $\sigma$-finite measure spaces and let $(Z, \Sigma_Z, \mu_Z)$ be the product measure space with $Z = X \times Y$. 
Let $g: X \times Y \rightarrow X $  and $h: X \times Y \rightarrow Y$ and involution $f: X\times Y \rightarrow X \times Y$ with $f(x, y) = (g(x, y), h(x, y))$ be measurable maps. 
Let $g_y: X \rightarrow X$ with $g_y(x) = g(x, y)$ be a measurable bijection with measurable inverse for each $y \in Y$ and assume that the Radon-Nikodym derivative $\frac{\diff (\mu_X\circ g_y^{-1})}{\diff \mu_X}(x)$ exists (for all $y \in Y$) and is measurable (as a map from $Z$ to $\mathbb{R}_{\geq 0}$).
Then we have the following: 
\begin{enumerate}
  \item The function $t_x(y) = h(g_y^{-1}(x), y)$ is a bijection on $Y$ for each $x \in X$

  \item Assume the maps $(x, y) \mapsto t_x(y)$ and $(x, y) \mapsto t_x^{-1}(y)$ are $\Sigma_Z$-measurable and the Radon-Nikodym derivative $\left(\frac{\diff(\mu_Y\circ t_x^{-1})}{\diff\mu_Y}\right)(y)$ exists (for all $x \in X$) and is measurable (as a map from $Z$ to $\mathbb{R}_{\geq 0}$). 
  Then the following Radon-Nikodym derivative exists:  
  \[ \left(\frac{\diff ((\mu_X\times \mu_Y)\circ f^{-1})}{\diff\mu_X\times \mu_Y}\right)(x,y) 
  = \left(\frac{\diff(\mu_X\circ g_{t_x^{-1}(y)}^{-1})}{\diff\mu_X}\right)(x) \left(\frac{\diff(\mu_Y\circ t_x^{-1})}{\diff\mu_Y}\right)(y). \]

  \item If $h(x, y) = h(y)$ for all $(x, y) \in Z$ then the Radon-Nikodym derivative simplifies to:
  \[ \left(\frac{\diff ((\mu_X\times \mu_Y)\circ f^{-1})}{\diff\mu_X\times \mu_Y}\right)(x,y) 
  = \left(\frac{\diff(\mu_X\circ g_{h^{-1}(y)}^{-1})}{\diff\mu_X}\right)(x) \left(\frac{\diff(\mu_Y\circ h^{-1})}{\diff\mu_Y}\right)(y). \]
\end{enumerate}

\end{proposition}
\begin{proof}
\begin{enumerate}
\item
Notice that since $f$ is an involution it follows that it is bijective. 
Fix $x \in X$. 

To prove that $t_x$ is injective let $t_x(y_1) = t_x(y_2)$.
Since $g(g_{y_1}^{-1}(x), y_1) = x$ then $f(g_{y_1}^{-1}(x), y_1) = (x, h(g_{y_1}^{-1}(x), y_1)) = (x, t_x(y_1))$.
Similarly, $f(g_{y_2}^{-1}(x), y_2) = (x, t_x(y_2))$. 
Since $(x, t_x(y_1)) = (x, t_x(y_2))$ we get $f(g_{y_1}^{-1}(x), y_1) = f(g_{y_2}^{-1}(x), y_2)$. By the injectivity of $f$ we get that $y_1 = y_2$, so $t_x$ is injective. 

Let $y \in Y$. 
Since $f$ is surjective, there exists $(u, z) \in X \times Y$ such that $f(u, z) = (x, y)$. 
This implies $g(u, z) = x$.
Then $t_x(z) = h(g_z^{-1}(x), z) = h(u, z) = y$.
Therefore, $t_x$ is surjective. 
\item 
Let $A \in \Sigma_X$ and $B \in \Sigma_Y$. 
It is enough to show that for any such $\mu_Z$-measurable set $A \times B$ we have:
\[((\mu_X\times \mu_Y)\circ f^{-1})(A \times B) = \int_{A \times B} \left(\frac{\diff(\mu_X\circ g_{t_x^{-1}(y)}^{-1})}{\diff\mu_X}\right)(x) \left(\frac{\diff(\mu_Y\circ t_x^{-1})}{\diff\mu_Y}\right)(y)\ (\mu_X \times \mu_Y)(\diff x, \diff y)\]
for all such sets $A \times B \in \Sigma_Z$.
\begin{align*}
    ((\mu_X\times \mu_Y)\circ f^{-1})(A \times B) 
    &= \int_Y \int_X [g(x, y) \in A] [h(x, y) \in B] \mu_X (\diff x) \mu_Y (\diff y)\\
    &= \int_Y \left(\int_X [g_y(x) \in A] [h(x, y) \in B] \mu_X (\diff x)\right)\mu_Y (\diff y)\\
\end{align*}
Performing a change of variable $u = g_y(x)$ and using the existence of $\left(\frac{\diff(\mu_X\circ g_{y}^{-1})}{\diff\mu_X}\right)$: 
\begin{align*}
    ((\mu_X\times \mu_Y)\circ f^{-1})(A \times B) 
    &= \int_Y \left(\int_X [u \in A] [h(g_y^{-1}(u), y) \in B] (\mu_X \circ g_y^{-1}) (\diff u)\right)\mu_Y (\diff y)\\
    &= \int_Y \left(\int_X [u \in A] [h(g_y^{-1}(u), y) \in B] \left(\frac{\diff(\mu_X\circ g_{y}^{-1})}{\diff\mu_X}\right)(u) \mu_X (\diff u)\right)\mu_Y (\diff y)\\
    &= \int_Y \left(\int_A [t_u(y) \in B] \left(\frac{\diff(\mu_X\circ g_{y}^{-1})}{\diff\mu_X}\right)(u) \mu_X (\diff u)\right)\mu_Y (\diff y)
\end{align*}
Swapping the integrals (using Tonelli), then performing a change of variables $z = t_u(y)$ and using the existence of $\left(\frac{\diff(\mu_Y\circ t_u^{-1})}{\diff\mu_Y}\right)$ we get: 

\begin{align*}
    ((\mu_X\times \mu_Y)\circ f^{-1})(A \times B) 
    &= \int_A \left(\int_Y [z \in B] \left(\frac{\diff(\mu_X\circ g_{t_u^{-1}(z)}^{-1})}{\diff\mu_X}\right)(u) (\mu_Y \circ t_u^{-1}) (\diff z)\right) \mu_X(\diff u)\\
    &= \int_A \left(\int_Y [z \in B] \left(\frac{\diff(\mu_X\circ g_{t_u^{-1}(z)}^{-1})}{\diff\mu_X}\right)(u) \left(\frac{\diff(\mu_Y\circ t_u^{-1})}{\diff\mu_Y}\right)(z) \mu_Y (\diff z)\right) \mu_X(\diff u)\\
    &= \int_A \int_B \left(\frac{\diff(\mu_X\circ g_{t_u^{-1}(z)}^{-1})}{\diff\mu_X}\right)(u) \left(\frac{\diff(\mu_Y\circ t_u^{-1})}{\diff\mu_Y}\right)(z) \mu_Y (\diff z) \mu_X(\diff u)\\
    &= \int_A \int_B \left(\frac{\diff(\mu_X\circ g_{t_x^{-1}(y)}^{-1})}{\diff\mu_X}\right)(x) \left(\frac{\diff(\mu_Y\circ t_x^{-1})}{\diff\mu_Y}\right)(y) \mu_Y (\diff y) \mu_X(\diff x)\\
\end{align*}
\item If $h(x, y) = h(y)$ then $t_x = h$ so the result follows directly from the previous part.
\end{enumerate}
\end{proof}

\bibliographystyle{ACM-Reference-Format}
%\bibliography{biblio}

\end{document}